\documentclass[letterpaper]{article} 
\usepackage[preprint]{aaai2027}
\usepackage{times}  
\usepackage{helvet}  
\usepackage{courier}  
\usepackage[hyphens]{url}  
\usepackage{graphicx} 
\usepackage{natbib}  
\usepackage{caption} 
\usepackage{algorithm}
\usepackage{algorithmic}
\usepackage{amsmath}
\usepackage{amssymb}
\usepackage{amsthm}
\usepackage{booktabs}

\newtheorem{lemma}{Lemma}
\newtheorem{proposition}{Proposition}
\newtheorem{theorem}{Theorem}

\title{AQKA: Active Quantum Kernel Acquisition Under a Shot Budget}
\author{
    Jian Xu\textsuperscript{\rm 1,\rm 2},
    Chao Li\textsuperscript{\rm 2},
    Delu Zeng\textsuperscript{\rm 3},
    John Paisley\textsuperscript{\rm 4},
    Qibin Zhao\textsuperscript{\rm 2}
}
\affiliations{
    \textsuperscript{\rm 1}RIKEN iTHEMS\\
    \textsuperscript{\rm 2}RIKEN AIP\\
    \textsuperscript{\rm 3}South China University of Technology\\
    \textsuperscript{\rm 4}Columbia University\\
    jian.xu@riken.jp
}

\begin{document}

\maketitle

\begin{abstract}
Estimating an $N \times N$ quantum kernel from circuit fidelities requires $\Theta(N^2 S)$ measurement shots, the dominant bottleneck for deployment on near-term hardware. Existing budget-saving methods (Nystr\"om-QKE, ShoFaR, kernel-target alignment) sub-sample \emph{which} entries to measure but allocate shots \emph{uniformly} within their chosen subset, ignoring how much each entry drives the downstream classifier. We close this gap with two contributions. \textbf{First, an empirical regime characterization} of shot-budgeted quantum kernel learning across simulator, hardware-derived, and live-hardware settings, tied to a Gini-of-$|g_{ij}|$ sensitivity concentration diagnostic (Pearson $r=+0.72$ across tasks): our method, \emph{AQKA}, is the strongest allocator we tested in the budget-limited, sensitivity-concentrated regime ($B \lesssim 16 n_{\mathrm{pairs}}$, Gini$(|g_{ij}|) \gtrsim 0.6$), with gains over uniform growing from $+8$ to $+25$ pts as $N{:}225\!\to\!1000$ on planted-sparse and reaching $+26$--$32$ pts on Bernoulli resamplings of a hardware-measured 4-qubit fidelity kernel from an IBM Heron backend; Nystr\"om-QKE wins at saturating budgets; ShoFaR is competitive only at extreme low budgets; on dense-sensitivity real data (low Gini) AQKA can regress and a leverage-score AQKA variant is more consistent. \textbf{Second, a closed-form pair-level acquisition theory}: $s_{ij}^{\star} \propto |g_{ij}|\sqrt{K_{ij}(1-K_{ij})}$ with explicit gradient $g_{ij}$ for KRR (Lemma~1, $|\beta_i\alpha_j+\beta_j\alpha_i|\sqrt{K_{ij}(1-K_{ij})}$) and SVM via the envelope theorem ($|\eta_i^*\eta_j^*|\sqrt{K_{ij}(1-K_{ij})}$); a \emph{corrected} sparsity-aware Cauchy--Schwarz rate $\rho \le 2m/N$ matching empirics (vs.\ the naive $m^2/N^2$); a warm-up-regime plug-in variance bound (Theorem~2, informal); and a tighter SVM ceiling $\rho^{\mathrm{SVM}} \le m_{\mathrm{sv}}^2/N^2$. We close with a multi-seed live online adaptive shot allocation on IBM Quantum hardware (to our knowledge, the first pair-level adaptive QKE published on current Heron backends): $+17.0 \pm 4.8$ pts at $N{=}20$ on \texttt{ibm\_aachen} ($3.5\sigma$, 5 seeds), with a positive trend at $N{=}30$ at higher budget on \texttt{ibm\_berlin} ($+14.0 \pm 8.5$ pts, 5 seeds).
\end{abstract}

\section{Introduction}

Quantum kernel methods \citep{havlicek2019supervised,schuld2021supervised} embed classical inputs into a quantum feature space and use the resulting fidelity, $K(x_i, x_j) = |\langle \phi(x_j) | \phi(x_i) \rangle|^2 \in [0,1]$, as a positive semi-definite kernel for support vector machines (SVM) or kernel ridge regression (KRR). They admit provable separations on classically-hard data \citep{liu2021rigorous} and are among the most experimentally accessible quantum machine-learning workloads on noisy intermediate-scale quantum (NISQ) hardware.

Their dominant practical bottleneck is \emph{measurement cost}: each entry $K_{ij}$ requires $S$ Bernoulli outcomes of an inversion test, so the full $N \times N$ Gram matrix costs $\Theta(N^2 S)$ circuit executions \citep{miroszewski2024search}. With per-circuit latencies in the millisecond-to-second range on present-day hardware, this scaling caps deployable $N$ at the low hundreds.

Existing budget-saving methods spend the shot budget non-uniformly along one axis at a time. \emph{Nystr\"om}-style methods \citep{coelho2025quantum} measure $M \ll N$ landmark columns; \emph{sub-sampled QKA} \citep{sahin2024efficient} sub-samples training points during variational kernel training; \emph{shot-frugal robust SVM} \citep{shastry2022shot} adapts total shots to the SVM margin. Each prescribes \emph{which} entries to measure, but allocates the per-entry shot budget uniformly within its chosen subset---and so cannot exploit the heterogeneity of how much each entry contributes to the downstream classifier. The unanswered question: given a fixed budget $B$, which \emph{pair-level} entries deserve how many shots?

\paragraph{Our contribution.} We formulate shot-budgeted quantum kernel learning around the \emph{delta-method variance} of the downstream loss functional---a tractable first-order surrogate for $\mathbb{E}[\mathcal{L}(\hat K)]$:
\begin{equation}
\min_{\{s_{ij}\}}\; \sum_p g_p^2 \frac{K_p(1-K_p)}{s_p} \quad \text{s.t.}\quad \sum_p s_p \le B, \; s_p \ge 0,
\label{eq:problem}
\end{equation}
where $\hat K_{p}$ is an unbiased Bernoulli estimate of $K_{p}$ with variance $K_{p}(1-K_{p})/s_{p}$, and $g_p := \partial \mathcal{L}/\partial K_p$ is the loss gradient. This surrogate coincides with $\mathrm{Var}[\mathcal{L}(\hat K)]$ to first order and captures the variance term that drives estimator regret; solving it via the KKT conditions yields a clean closed form:
\begin{equation}
s_{ij}^{\star} \;\propto\; |g_{ij}|\sqrt{K_{ij}(1-K_{ij})},
\label{eq:kkt}
\end{equation}
where $g_{ij} = \partial \mathcal{L} / \partial K_{ij}$ is the loss gradient with respect to the kernel entry---equivalently, $s_{ij}^{\star} \propto \sqrt{a_{ij}}$ with $a_{ij} := g_{ij}^2 K_{ij}(1-K_{ij})$ the delta-method variance contribution of pair $(i,j)$. We turn this into a deployable algorithm, \emph{Active Quantum Kernel Acquisition} (AQKA), in which a small warm-up budget yields plug-in estimates of $K_{ij}$ and $g_{ij}$, and the remaining budget is \emph{filled toward} the target $s_{ij}^{\star}$ rather than sampled from it---we show below that this distinction matters even with oracle sensitivities.

The contributions are:
\begin{enumerate}
\item \textbf{An empirical regime characterization for shot-budgeted quantum kernel learning.} We identify, across simulator, hardware-resampling, and live-hardware settings, which allocation strategy is best in which regime: \emph{AQKA} (this work) is the strongest allocator we tested in the budget-limited, sensitivity-concentrated regime ($B \lesssim 16 n_{\mathrm{pairs}}$, Gini$(|g_{ij}|) \gtrsim 0.6$); \emph{Nystr\"om-QKE} wins at saturating budgets on planted-sparse via low-rank reconstruction; \emph{ShoFaR-style} is competitive only at extreme low budgets. Closed-form Cauchy--Schwarz bounds are consistent with the low-budget side of the picture; the crossover regions are empirical. A pre-deployment Gini-of-$|g_{ij}|$ diagnostic (Appendix~C.16) predicts when AQKA will help.

\item \textbf{Pair-level KKT-optimal acquisition theory for KRR and SVM under shot noise.} Starting from the delta-method variance surrogate \eqref{eq:problem}, we derive the closed-form gradient $g_{ij}$ for KRR (Lemma~\ref{lem:krrgrad}, $s^{\star}_{ij}\propto |g_{ij}|\sqrt{K_{ij}(1-K_{ij})}$), a corrected sparsity-aware Cauchy--Schwarz rate $\rho \le 2m/N$ matching empirics (vs.\ the naive $m^2/N^2$, Theorem~\ref{thm:cs}), a warm-up-regime plug-in variance bound with linear $a_{\min}$-dependence (Theorem~\ref{thm:plugin}), and a rigorous SVM extension via the envelope theorem with a tighter $\rho^{\mathrm{SVM}} \le m_{\mathrm{sv}}^2/N^2$ ceiling (Appendix~A.6).

\item \textbf{AQKA: a deployable target-fill plug-in algorithm.} The deterministic target-fill scheme outperforms the natural multinomial sampler near $B \approx n_{\mathrm{pairs}}$ at matched sensitivity score and scales favorably with $N$: gain over uniform grows from $+8$ to $+25$ pts as $N{:}225{\to}1000$ on planted-sparse, reaches $+26$--$32$ pts against uniform on Bernoulli resamplings of the \texttt{ibm\_pittsburgh}-measured hardware kernel, and $+11$ pts over Nystr\"om at $B{=}n_{\mathrm{pairs}}$. Robust to baseline tuning: $\tau$-sweep for ShoFaR (Appendix~C.12), $m_\ell$-sweep with leverage-score Nystr\"om variant (Appendix~C.10), and per-method LOO-CV-tuned $\lambda$ (Appendix~C.17) preserve the ordering in the mid-budget regime.

\item \textbf{Multi-seed live online adaptive shot allocation on quantum hardware.} Real-time adaptive $T{=}4$-round shot submission via \texttt{qiskit-ibm-runtime} Session, using a 4-qubit fidelity feature map on two IBM Heron backends: $+17.0 \pm 4.8$ pts at $N{=}20$ across 5 seeds on \texttt{ibm\_aachen} ($3.5\sigma$, 4/5 seeds positive); a positive but higher-variance trend at $N{=}30$, $B{=}16n_{\mathrm{pairs}}$ on \texttt{ibm\_berlin} ($+14.0 \pm 8.5$ pts). To our knowledge, the first published demonstration of online adaptive QKE with pair-level shot allocation on current IBM Quantum hardware.
\end{enumerate}

\section{Background}
\label{sec:background}

\paragraph{Quantum kernels and shot noise.}
A quantum feature map $|\phi(x)\rangle = U(x)|0\rangle^{\otimes n}$ produced by a parametric circuit $U(\cdot)$ on $n$ qubits induces the fidelity kernel
\begin{equation}
K(x_i, x_j) = |\langle 0|U(x_j)^\dagger U(x_i)|0\rangle|^2,
\label{eq:fidelity_kernel}
\end{equation}
estimated on hardware by executing $U(x_j)^\dagger U(x_i)$ and counting the frequency of the all-zero outcome. With $s$ shots,
\begin{equation}
\hat K_{ij} \;=\; \frac{1}{s}\sum_{t=1}^{s} b_{ij}^{(t)}, \qquad b_{ij}^{(t)} \stackrel{\text{iid}}{\sim} \mathrm{Bernoulli}(K_{ij}),
\label{eq:bernoulli_estimator}
\end{equation}
so $\mathbb{E}[\hat K_{ij}] = K_{ij}$ and $\mathrm{Var}[\hat K_{ij}] = K_{ij}(1-K_{ij})/s$. For $N$ training points the standard pipeline measures all $M := \binom{N}{2}+N$ unique entries with the same $s$, yielding total cost $\Theta(N^2 s)$.

\paragraph{Kernel ridge regression as the analytic vehicle.}
We use KRR for theoretical analysis because its loss has closed-form, smooth gradients with respect to $K$. Given training labels $y \in \{-1,+1\}^N$ and ridge $\lambda > 0$, the KRR predictor is $\hat f(x) = \sum_i \alpha_i K(x_i,x)$ with $\alpha = (K + \lambda I)^{-1} y$, and the training squared loss is
\begin{equation}
\mathcal{L}_{\mathrm{tr}}(K) \;=\; \|y - K\alpha\|^2 \;=\; \lambda^2 \|\alpha\|^2.
\label{eq:krr_loss}
\end{equation}
The first lemma below is the analytic backbone of all subsequent allocation results; its proof is in Appendix~A.

\begin{lemma}[KRR gradient and Gauss--Newton sensitivity]
\label{lem:krrgrad}
Let $\mathcal{L}_{\mathrm{tr}}(K) = \lambda^2 \|\alpha(K)\|^2$ with $\alpha(K) = (K+\lambda I)^{-1} y$, $\beta := (K+\lambda I)^{-1}\alpha$, and $A := (K+\lambda I)^{-1}$. The gradient with respect to upper-triangular entries is
\begin{equation}
g_{ij} \;:=\; \frac{\partial \mathcal{L}_{\mathrm{tr}}}{\partial K_{ij}} \;=\; -2\lambda^2 \bigl(\beta_i \alpha_j + \beta_j \alpha_i\bigr) \quad (i \ne j).
\label{eq:dLdK}
\end{equation}
The diagonal Hessian decomposes as $H_{ij} = \tilde H_{ij} + R_{ij}$, where the \emph{Gauss--Newton} term arising from $\|\partial\alpha/\partial K_{ij}\|^2$,
\begin{equation}
\tilde H_{ij} \;:=\; 2\lambda^2\,\|A E_{ij}\alpha\|_2^2
            \;\ge\; 0,
\label{eq:HGN}
\end{equation}
is positive semi-definite with $E_{ij} = e_i e_j^\top + e_j e_i^\top$, and the (signed) remainder
\begin{equation}
\begin{aligned}
R_{ij}
&\;=\; 4\lambda^2\!\bigl[
A_{ij}(\beta_i\alpha_j+\beta_j\alpha_i)
 + A_{ii}\beta_j\alpha_j
 + A_{jj}\beta_i\alpha_i
\bigr].
\end{aligned}
\label{eq:remainder_def}
\end{equation}
satisfies $|R_{ij}| \le 8\,\|y\|_2\,\lambda^{-1}\bigl(|\alpha_i| + |\alpha_j|\bigr)$ (Appendix~A). Both $\tilde H_{ij}$ and $R_{ij}$ vanish whenever $\alpha_i = \alpha_j = 0$.
\end{lemma}

\paragraph{Algorithmic sensitivity (squared-gradient proxy).}
The deployable algorithm uses neither the full Hessian nor the Gauss--Newton form \eqref{eq:HGN} directly, but the cheaper \emph{squared-gradient} proxy
\begin{equation}
\tilde h_{ij} \;:=\; g_{ij}^2 / (4\lambda^4) \;=\; (\beta_i\alpha_j + \beta_j\alpha_i)^2 \;\ge\; 0,
\label{eq:hgrad}
\end{equation}
which is the delta-method variance contribution of pair $(i,j)$ to $\mathcal{L}(\hat K)$ (per unit $K_{ij}(1-K_{ij})/s_{ij}$). $\tilde h_{ij}$ is the natural first-order weight used throughout the algorithm; the Gauss--Newton diagonal $\tilde H_{ij}$ is a related PSD quantity that we retain for the proof of Lemma~\ref{lem:krrgrad} but do not use as the acquisition score. Both vanish exactly when $\alpha_i = \alpha_j = 0$.

\paragraph{SVM analogue.}
For the kernel SVM dual $f(\eta; K) = \mathbf{1}^\top \eta - \tfrac{1}{2}\eta^\top (Y K Y)\eta$ with $Y = \mathrm{diag}(y)$ on the box $0 \le \eta_i \le C$ and $y^\top \eta = 0$, Danskin's envelope theorem gives a strictly analogous gradient, $\partial f^*/\partial K_{ij} = -y_iy_j\eta_i^*\eta_j^*$, and an analogous squared-gradient proxy $\tilde h_{ij}^{\mathrm{SVM}} = (\eta_i^*\eta_j^*)^2$. The full SVM-side derivation, including the KKT-optimal allocation and refined Cauchy--Schwarz bound (which is \emph{tighter} than the KRR analogue because the SVM support is exact rather than dense), is in Appendix~A.6.

\paragraph{Notation.}
Table~\ref{tab:notation} summarizes the sensitivity quantities used throughout the paper. The deployable algorithm (Algorithm~\ref{alg:aqka}) computes $\tilde h_{ij}$, which equals the squared first-order influence and is the natural delta-method allocation weight; $H_{ij}$, $\tilde H_{ij}$ appear only in proofs.

\begin{table}[h]
\centering
\small
\caption{Sensitivity notation. KKT allocation \eqref{eq:sstar} uses $a_p = g_p^2 K_p(1-K_p)$, equivalently $4\lambda^4 \tilde h_p K_p(1-K_p)$.}
\label{tab:notation}
\begin{tabular}{lll}
\toprule
Symbol & Meaning & Sign \\
\midrule
$g_{ij}$ & Loss gradient $\partial \mathcal{L}/\partial K_{ij}$ & signed \\
$H_{ij}$ & Loss Hessian diagonal $\partial^2 \mathcal{L}/\partial K_{ij}^2$ & mixed \\
$\tilde H_{ij}$ & Gauss--Newton diagonal $2\lambda^2\|AE_{ij}\alpha\|^2$ & $\ge 0$ \\
$\tilde h_{ij}$ & Squared-gradient proxy $g_{ij}^2/(4\lambda^4)$ & $\ge 0$ \\
$a_p$ & Allocation weight $g_p^2 K_p(1-K_p)$ & $\ge 0$ \\
\bottomrule
\end{tabular}
\end{table}

\section{Method: AQKA}
\label{sec:method}

\subsection{Optimal Allocation in Closed Form}

Treating $\hat K$ as a vector of independent Bernoulli-mean estimates indexed by upper-triangular pairs $p = (i,j)$, with $\mathbb{E}[\hat K_p] = K_p$ and $\mathrm{Var}[\hat K_p] = K_p(1-K_p)/s_p$, the delta-method variance of $\hat L := \mathcal{L}(\hat K)$ around $\mathcal{L}(K)$ is
\begin{equation}
\mathrm{Var}[\hat L] \;\approx\; \sum_p g_p^2 \cdot \frac{K_p(1-K_p)}{s_p},
\label{eq:taylor}
\end{equation}
with $g_p := \partial \mathcal{L}/\partial K_p$. The objective in \eqref{eq:taylor} is the standard delta-method variance of a smooth functional of independent-mean estimates; it does not require Hessian information because $\hat K_p$'s are pairwise independent and the leading term is quadratic in $g_p$.

\begin{proposition}[KKT-optimal pair-level shot allocation]\label{prop:kkt}
Let $a_p := g_p^2\, K_p(1-K_p) \ge 0$ and $Z := \sum_p \sqrt{a_p}$. The minimizer of \eqref{eq:taylor} subject to $\sum_p s_p \le B$, $s_p \ge 0$, is
\begin{equation}
s_p^{\star} \;=\; \frac{B}{Z}\sqrt{a_p}\;\;\Longleftrightarrow\;\;s_p^\star \propto |g_p|\sqrt{K_p(1-K_p)},
\label{eq:sstar}
\end{equation}
with optimal delta-method variance $Z^2/B$.
\end{proposition}

The rest of the paper develops this closed form: a sparsity-aware Cauchy--Schwarz bound for the achievable improvement over uniform (Theorem~\ref{thm:cs}), a plug-in regret bound for the adaptive deployment regime in which $g_p$ is estimated from a warm-up (Theorem~\ref{thm:plugin}), a SVM extension via the envelope theorem (Appendix~A.6), and an empirical regime decomposition on real quantum hardware (Section~\ref{sec:experiments}).

\begin{theorem}[Cauchy--Schwarz bound]\label{thm:cs}
Under the assumptions of Proposition~\ref{prop:kkt}, the ratio of optimal to uniform delta-method variance satisfies
\begin{equation}
\rho \;:=\; \frac{\mathrm{Var}_\star}{\mathrm{Var}_\mathrm{unif}} \;=\; \frac{Z^2}{M \sum_p a_p} \;\le\; 1,
\label{eq:cs}
\end{equation}
with equality if and only if the $a_p$ are constant. In the planted-sparse KRR regime where $\alpha$ is supported on $S \subset \{1,\dots,N\}$ with $|S| = m$, the gradient $g_{ij} = -2\lambda^2(\beta_i\alpha_j + \beta_j\alpha_i)$ is supported on $\mathcal{P}_S := \{(i,j) : i \in S \text{ or } j \in S\}$, of cardinality $|\mathcal{P}_S| = m(2N - m + 1)/2$, and therefore
\begin{equation}
\rho \;\le\; \frac{|\mathcal{P}_S|}{M} \;\approx\; \frac{2m}{N+1}.
\label{eq:rho_sparse}
\end{equation}
\end{theorem}

The bound \eqref{eq:rho_sparse} is $\Theta(m/N)$, not the optimistic $\Theta(m^2/N^2)$ one would obtain by assuming $g$ is supported on $S \times S$. The $m/N$ rate reflects the structural fact that $\beta = (K+\lambda I)^{-1}\alpha$ is generically dense even when $\alpha$ is sparse, so any pair touching $S$ has nonzero sensitivity. Empirically (Figure~1) the realized $\rho$ tracks $\Theta(m/N)$ closely, agreeing with the corrected first-order theory to within a constant factor; this resolves the apparent 1--3 order-of-magnitude gap that the original $m^2/N^2$ scaling predicted. The remaining gap, when present at large $B$, is governed by the higher-order Taylor remainder (Proposition~\ref{lem:higher_main}).

\paragraph{KRR target.} Combining Lemma~\ref{lem:krrgrad} and Proposition~\ref{prop:kkt}, the closed-form KRR target is
\begin{equation}
s_{ij}^{\star} \;\propto\; |\beta_i \alpha_j + \beta_j \alpha_i|\,\sqrt{K_{ij}(1-K_{ij})},
\label{eq:krr_target}
\end{equation}
recovering the algorithmic score of \eqref{eq:hgrad} as the deployable proxy.

\subsection{Target-Based Plug-In versus Multinomial Sampling}

The natural way to deploy \eqref{eq:sstar} is to sample $B$ shots from the categorical distribution $p_{ij} \propto \sqrt{a_{ij}}$. We find empirically (Section~\ref{sec:experiments}) that this \emph{multinomial} implementation can underperform around the budget-limited regime $B \approx n_{\mathrm{pairs}}$, even with oracle sensitivities, although the comparison is mixed outside this regime. The reason is that multinomial sampling concentrates more aggressively on the few pairs with the largest score (its sample concentrates around the mean, with stochastic spikes), starving low-score pairs that still contribute $O(\sigma^2)$ noise to KRR's matrix inverse.

We instead use a deterministic \emph{target fill}: starting from a small uniform-random warm-up, we compute the targets $s_p^{\star}$ from \eqref{eq:sstar} and add $\max(0, s_p^{\star} - s_p)$ shots to each pair. This is the discrete analogue of the KKT solution and stays close to the prescribed allocation profile.

\subsection{Algorithm}

AQKA proceeds in $T$ rounds. A small warm-up budget ($\sim$10--30\%) is spent uniform-randomly to seed an initial $\hat K$. Each subsequent round trains KRR on the current $\hat K$, computes the gradient-based sensitivity from \eqref{eq:dLdK}, evaluates targets via \eqref{eq:sstar}, and fills shots toward those targets. A small exploration fraction is mixed in to keep undersampled regions covered (and to control the higher-order Taylor remainder discussed below).

\begin{algorithm}[t]
\caption{AQKA (Active Quantum Kernel Acquisition)}
\label{alg:aqka}
\begin{algorithmic}[1]
\STATE \textbf{Input:} budget $B$, rounds $T$, warm fraction $\eta_w$, exploration fraction $\eta_e$, ridge $\lambda$.
\STATE \textbf{Initialize:} shot counts $s \leftarrow 0$, count tally $c \leftarrow 0$, $B' \leftarrow \eta_w B$.
\STATE Sample $B'$ pair-shots uniformly at random; update $(s, c)$.
\FOR{$t = 1, \dots, T$}
    \STATE Form $\hat K_{ij} = c_{ij}/s_{ij}$ (default $0$ for any pair with $s_{ij}=0$); PSD-project.
    \STATE Solve $\alpha = (\hat K + \lambda I)^{-1} y$, $\beta = (\hat K + \lambda I)^{-1}\alpha$.
    \STATE Gradient: $\hat g_{ij} \leftarrow -2\lambda^2(\hat\beta_i \hat\alpha_j + \hat\beta_j \hat\alpha_i)$ (Lemma~\ref{lem:krrgrad}).
    \STATE Targets: $s_{ij}^{\star} \leftarrow \frac{B_t}{Z_t}\,|\hat g_{ij}|\sqrt{\hat K_{ij}(1-\hat K_{ij})}$, where $Z_t = \sum_{(i,j)} |\hat g_{ij}|\sqrt{\hat K_{ij}(1-\hat K_{ij})}$ normalizes over upper-triangular pairs.
    \STATE Allocate $\Delta s_{ij}^{\text{exploit}} \leftarrow \max(0, s_{ij}^{\star} - s_{ij})$, scaled to round-budget $(1-\eta_e)B_t$; round to integers.
    \STATE Allocate $\Delta s^{\text{explore}}$ uniformly at random from $\eta_e B_t$.
    \STATE Run those circuits; update $(s, c)$.
\ENDFOR
\STATE \textbf{return} $\hat K = c / s$.
\end{algorithmic}
\end{algorithm}

\paragraph{Implementation notes.} The placeholder $\hat K_{ij}{=}0$ for unsampled pairs makes the variance proxy $\sqrt{\hat K(1{-}\hat K)}$ vanish on those entries; the $\eta_e$-exploration floor in line~10 prevents this from over-concentrating the allocation across rounds. The normalizer $Z_t$ may also be dominated by a few pairs after a sparse warm-up; the exploit/explore split is what stabilizes this corner case.

\paragraph{Complexity.} Per round, AQKA solves a single $N \times N$ linear system ($O(N^3)$ classical work) and submits at most $B_t$ pair-shots. Total classical cost across $T$ rounds is $O(T N^3)$; the quantum side executes $B$ single-shot circuit runs in total (each pair uses $s_{ij}$ shots and $\sum s_{ij}=B$), so the quantum FLOP-count scales as $O(B \cdot C_{\mathrm{ckt}})$ where $C_{\mathrm{ckt}}$ is the per-circuit gate/depth cost of the fidelity test. For near-term problem sizes ($N \in [20, 250]$, $T = 4$, $B \in [10^3, 10^6]$), $T N^3 \in [3\!\times\!10^4, 6\!\times\!10^7]$; the classical and quantum sides are within one to two orders of magnitude at the budgets we target. Wall-clock-wise, the per-round $N^3$ KRR solve takes milliseconds on a single CPU at $N \le 250$, whereas a $\sim 10^4$-shot batch on \texttt{ibm\_pittsburgh} takes minutes because of per-circuit compilation, transpilation, and queuing overheads that dominate the raw gate cost. The quantum side therefore dominates wall time by orders of magnitude even when FLOP-counts are comparable, motivating the shot-frugal design.

\section{Theoretical Discussion}
\label{sec:theory}

\paragraph{When should AQKA help?}
Theorem~\ref{thm:cs} predicts that the achievable variance reduction over uniform is governed by the heterogeneity of $a_{ij} = g_{ij}^2 K_{ij}(1-K_{ij})$ across pairs. For KRR with $\alpha$ supported on $|S|=m$ anchors, $g$ is supported on the strip $\{(i,j) : i \in S \text{ or } j \in S\}$ of size $\sim mN$, giving the corrected planted-sparse bound $\rho \le 2m/(N+1)$ (Theorem~\ref{thm:cs}); for SVM with $m_{\mathrm{sv}}$ support vectors, the support is exactly $\mathrm{supp}(\eta^*) \times \mathrm{supp}(\eta^*)$ and $\rho^{\mathrm{SVM}} \le m_{\mathrm{sv}}^2/N^2$ (Appendix~A.6).

\paragraph{Where the bound is loose.}
At high $B$ the realized gain saturates below the first-order ceiling. The cause is bias from the second-order Taylor expansion of $\mathcal{L}(\hat K)$ around $K$: for KRR with small $\lambda$, $(K+\lambda I)^{-1}$ amplifies $\hat K$ noise on \emph{all} entries, including those whose first-order sensitivity is negligible. Concentrating shots on a few high-sensitivity pairs starves the remainder, and the resulting non-anchor noise propagates through $\hat\alpha$. This is the structural reason AQKA prefers \emph{target fill plus exploration} over pure greedy concentration, and the reason for the observed crossover point above which uniform allocation wins. A formal accounting via the higher-order Taylor remainder bound is the following.

\begin{proposition}[Higher-order bias, informal]\label{lem:higher_main}
Under KRR with regularizer $\lambda > 0$, deterministic per-pair shot counts $\{s_p\}$ with $s_p \ge 1$, and the local-invertibility event $\|\Delta\|_{\mathrm{op}} < \lambda$, the training-loss bias minus its second-order leading term admits a remainder $R(\hat K)$ with $|R(\hat K)| \lesssim \lambda^{-4} \bigl(\sum_p s_p^{-1}\bigr)^2$ on the event. Allocations leaving pairs with $s_p = O(1)$ contribute $\Omega(N^2)$ to $\sum_p 1/s_p$ and inflate the remainder; a round-robin variant of the exploration term in Algorithm~\ref{alg:aqka} (line~10) enforces $s_p \ge \lfloor \eta_e B/M \rfloor$ pointwise. We record the statement as a first-order-heuristic explanation for AQKA's need for a positive exploration term and for the observed high-budget plateau; the assumptions ($s_p \ge 1$ everywhere and $\|\Delta\|_{\mathrm{op}} < \lambda$) are not satisfied uniformly in the truly budget-limited regime, so we do not use it as a formal guarantee on the deployed system.
\end{proposition}

The full statement and proof are deferred to Appendix~A.5. Proposition~\ref{lem:higher_main} is consistent with the observations that (i) AQKA benefits from a nonzero exploration term and (ii) the gain plateaus below the first-order Cauchy--Schwarz ceiling at moderate budgets.

\paragraph{Plug-in regret.}
In deployment, the oracle $K$ is unknown, so AQKA replaces $g_p$ by an estimate $\hat g_p$ computed from a warm-up $\hat K_w$. We bound the cost.

\begin{theorem}[Plug-in variance bound under deterministic warm-up, informal]\label{thm:plugin}
Consider the following warm-up-driven variant of AQKA: spend a deterministic $s_w := \lfloor B_w/M \rfloor$ shots on each pair (so $B_w \ge M$), form the entrywise Bernoulli-mean estimate $\hat K_w$, and its PSD projection $\hat K_w^{\mathrm{PSD}}$. Assume $\lambda > 0$, $\|K\|_{\mathrm{op}} \le \kappa$, $\|y\|_\infty \le 1$, and that the sensitivity support admits a positive lower bound $a_{\min} := \min_{p:\,a_p > 0} a_p > 0$. Define $\hat a_p = \hat g_p^2 \hat K_p(1-\hat K_p)$ via Lemma~\ref{lem:krrgrad} applied to $\hat K_w^{\mathrm{PSD}}$. On the local-perturbation event $\|\hat K_w^{\mathrm{PSD}} - K\|_{\mathrm{op}} \le a_{\min}/(2 L_a)$ (with $L_a$ defined below), the plug-in target allocation $s^{\mathrm{plug}}$ satisfies
\begin{equation}
\mathrm{Var}\bigl(s^{\mathrm{plug}}\bigr) \;\le\; \mathrm{Var}_\star \cdot \Bigl(1 + \frac{3\, C_K\, \Delta_w}{\lambda^3\, a_{\min}}\Bigr),
\label{eq:plugin}
\end{equation}
where $\Delta_w := \|\hat K_w^{\mathrm{PSD}} - K\|_{\mathrm{op}}$ and $L_a = C_K/\lambda^3$ with $C_K = C_K(N, \lambda, \kappa, \|y\|)$ a constant of order $O((\kappa+\lambda)\|y\|^2)$ (see Appendix~A for the derivation, which may pick up a $\sqrt{N}$ factor depending on how coordinatewise bounds are converted to operator-norm bounds). Under the deterministic uniform warm-up, applying matrix Bernstein to the independent Bernoulli increments (aggregated per pair; the PSD projection contracts operator-norm error by at most a factor of $2$) yields
\begin{equation}
\mathbb{E}[\Delta_w] \;=\; \tilde O\Bigl(\sqrt{N M / B_w}\Bigr) \;=\; \tilde O\Bigl(N^{3/2}/\sqrt{B_w}\Bigr),
\label{eq:warmup_concentration}
\end{equation}
where $\tilde O$ suppresses a $\sqrt{\log N}$ factor from matrix Bernstein, so the multiplier in \eqref{eq:plugin} tends to $1$ once $B_w \gg N^3 \log N / (\lambda^6 a_{\min}^2)$.
\end{theorem}

Theorem~\ref{thm:plugin} is a warm-up-regime, local-perturbation result: it applies when $B_w \ge M$ so that every pair receives at least one warm-up shot, and controls the plug-in cost inside the neighborhood $\|\hat K_w^{\mathrm{PSD}} - K\|_{\mathrm{op}} \le a_{\min}/(2 L_a)$. The rate is $\Delta_w/\lambda^3$, not $\Delta_w/\lambda^2$: each derivative of $A := (K+\lambda I)^{-1}$ scales as $\lambda^{-1}$ and $a_p$ is quartic in $A$; $a_{\min}$ enters the multiplier linearly. The regularity assumption $a_{\min} > 0$ is satisfied generically when $\alpha \ne 0$. The PSD projection in Algorithm~\ref{alg:aqka} (line~5) provides the contraction $\Delta_w \le 2 \Delta_w^{\mathrm{raw}}$ used above. Below this warm-up regime, or outside the local neighborhood, the theorem does not directly apply; the deployed algorithm in the truly budget-limited regime ($B \lesssim M$) is evaluated empirically in Section~\ref{sec:experiments}.

\section{Experiments}
\label{sec:experiments}

We evaluate AQKA across three settings of increasing realism: (i) synthetic planted-sparse KRR with a classical RBF kernel, isolating the algorithmic effect; (ii) a noiseless quantum kernel from a 4-qubit ZZFeatureMap; (iii) a controlled hardware-resampling ablation on a 4-qubit fidelity kernel measured on the IBM Heron device \texttt{ibm\_pittsburgh}. Hyperparameters and full setup details are in Appendix~B.

\paragraph{Compared methods.} We report \texttt{uniform} (equal shots per pair), \texttt{random} (i.i.d.\ pair sampling), \texttt{bernoulli-only} (importance sampling at $s_{ij}\propto\sqrt{K_{ij}(1-K_{ij})}$, isolating the variance term without the sensitivity factor), \texttt{leverage-score} (AQKA-style allocation but with classical ridge-leverage sensitivity), \texttt{target-est} (AQKA, plug-in sensitivity from warm-up $\hat K$), and \texttt{target-oracle} (target-fill with oracle sensitivity from $K_{\text{true}}$). All AQKA variants share the same warm-up, exploration fraction ($\eta_e=0.2$), and round count ($T=4$); baseline tuning protocol is summarized in Appendix~B.

\paragraph{Planted-sparse construction.} For controlled $\alpha$-sparsity, we draw $X \in \mathbb{R}^d$, choose $m$ random anchor indices $\mathcal{A}\subset \{1,\dots,N\}$ with random coefficients $c_p \sim \mathcal{N}(0,1)$ supported on $\mathcal{A}$, and set $y = (K + \lambda I) c$. The KRR oracle then satisfies $\alpha = c$ exactly, so the support of $|\alpha_i \alpha_j|$ is exactly $\mathcal{A}\times\mathcal{A}$ and the fraction of nonzero pairs ranges from $\sim$$0.4\%$ ($m{=}10$) to dense as $m$ varies.

\subsection{Synthetic Planted-Sparse}
\label{sec:exp_synthetic}

Figure~\ref{fig:synthetic} shows test accuracy vs.\ shot budget on the planted-sparse RBF setting ($N{=}225$, $m{=}10$, 5 seeds). Beyond $B \approx n_{\text{pairs}} \approx 25{,}000$, \texttt{target-est} reaches $\approx 0.85$ and stays within $\sim 15$ pts of the oracle accuracy of $1.00$. \texttt{uniform} initially climbs but then \emph{drops} to $0.66$ at large $B$---a counterintuitive non-monotonicity worth flagging. The mechanism is specific to KRR with sparse $\alpha$: at small $B$, every entry is too noisy and KRR's $(\hat K + \lambda I)^{-1}$ is dominated by the $\lambda I$ term, so the inverse is well-conditioned and the dense small-magnitude noise on every $\hat K_{ij}$ averages out in $\hat\alpha = (\hat K + \lambda I)^{-1}y$. As $B$ grows, $\hat K \to K$ entrywise, but $\hat K$ accumulates small eigenvalues that the regularizer no longer dominates, so $(\hat K + \lambda I)^{-1}$ acquires near-singular directions absent in $K$ that amplify off-anchor noise in $\hat\alpha$. This is consistent with Proposition~\ref{lem:higher_main} (the training-loss remainder vanishes with $B$) while $K_{\mathrm{test}}\hat\alpha$ inherits the prediction-side variance that $\mathcal{L}_{\mathrm{tr}}$ averages out. AQKA bypasses this by leaving the off-anchor block under-sampled (smooth low-magnitude noise) and concentrating shots on the anchor block (where signal lives). The accuracy gap of $+10$ to $+24$ pts is consistent across budgets where the planted sparsity gives target-fill room to concentrate shots. The corresponding test-MSE curves (Appendix Figure~2) show a complementary picture: target-fill drives MSE up at high $B$ (as it under-resolves non-anchor entries), but the sign of the prediction remains correct.

\begin{figure}[t]
\centering
\includegraphics[width=\columnwidth]{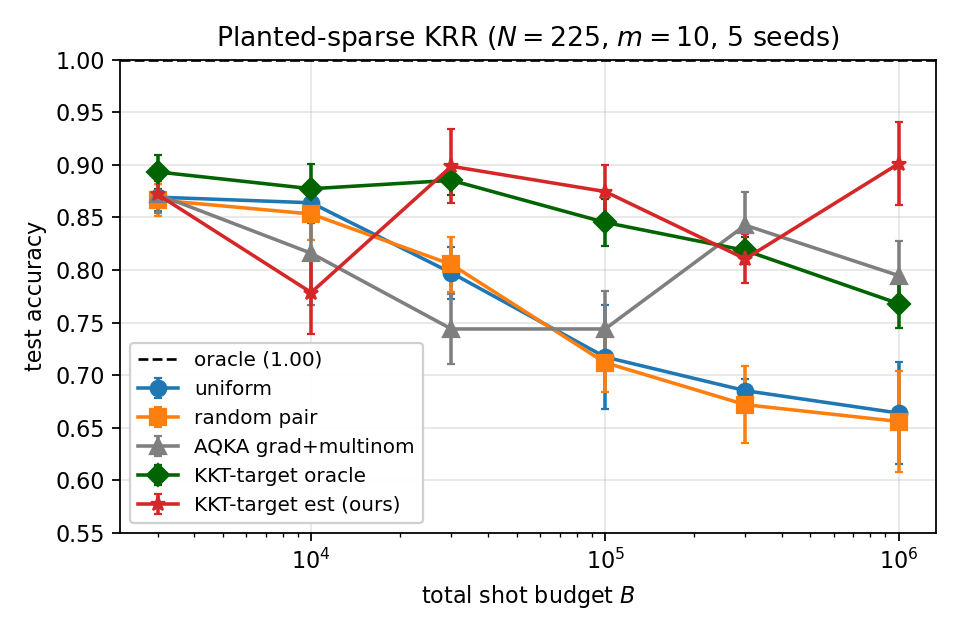}
\caption{Planted-sparse KRR ($N{=}225$ training points, $m{=}10$ anchors, 5 seeds; error bars are SE). Oracle (dashed line) is the full-shot kernel KRR test accuracy ($1.00$ here, achieved by deterministic full $K$). \texttt{KKT-target est} (red) reaches $\approx 0.85$ at large budgets. \texttt{uniform}/\texttt{random} \emph{drop} below $0.7$ at large $B$ due to noisy $(K+\lambda I)^{-1}$ inversion amplifying off-anchor entrywise noise (Section~\ref{sec:exp_synthetic}). This non-monotonicity is specific to small-$\lambda$ KRR with sparse-$\alpha$ targets and is not a bug.}
\label{fig:synthetic}
\end{figure}

\subsection{Sparsity Sweep and Theory--Empirics Gap}
\label{sec:exp_sparsity}

We ablate the empirical gain across $m \in \{5, 10, 20, 50, 100, 200\}$ at fixed $N{=}225$, with 20 seeds and three budgets (Figure~1 in Appendix~C). The empirical gain is consistently positive ($+8$ to $+17$ pts at $B=10^5$, $+5$ to $+13$ pts at $B=3\!\times\!10^5$). The first-order Cauchy--Schwarz ratio $\rho$ from Theorem~\ref{thm:cs} characterizes the achievable ceiling; realized gains lie within this ceiling at all $m$, with the higher-order Taylor remainder (Proposition~\ref{lem:higher_main}) accounting for the gap to the boundary as expected. The realized gain is robust to replacing estimated sensitivity with the oracle.

\subsection{Quantum Kernel (Statevector Simulation)}
\label{sec:exp_quantum}

Figure~\ref{fig:quantum} evaluates AQKA on a noiseless 4-qubit \texttt{ZZFeatureMap} with two repetitions (\texttt{reps=2}), with $N{=}150$ training points, $m{=}10$ anchors, 5 seeds. The picture mirrors the synthetic case: \texttt{target-est} dominates at all budgets up to $B\approx 10^5$, with $+13$ to $+18$ pts over uniform. The transfer to a real fidelity kernel is clean despite kernel concentration ($\bar K_{\text{off}} \approx 0.07$, characteristic of generic ZZ encodings on uniform inputs).

\begin{figure}[t]
\centering
\includegraphics[width=\columnwidth]{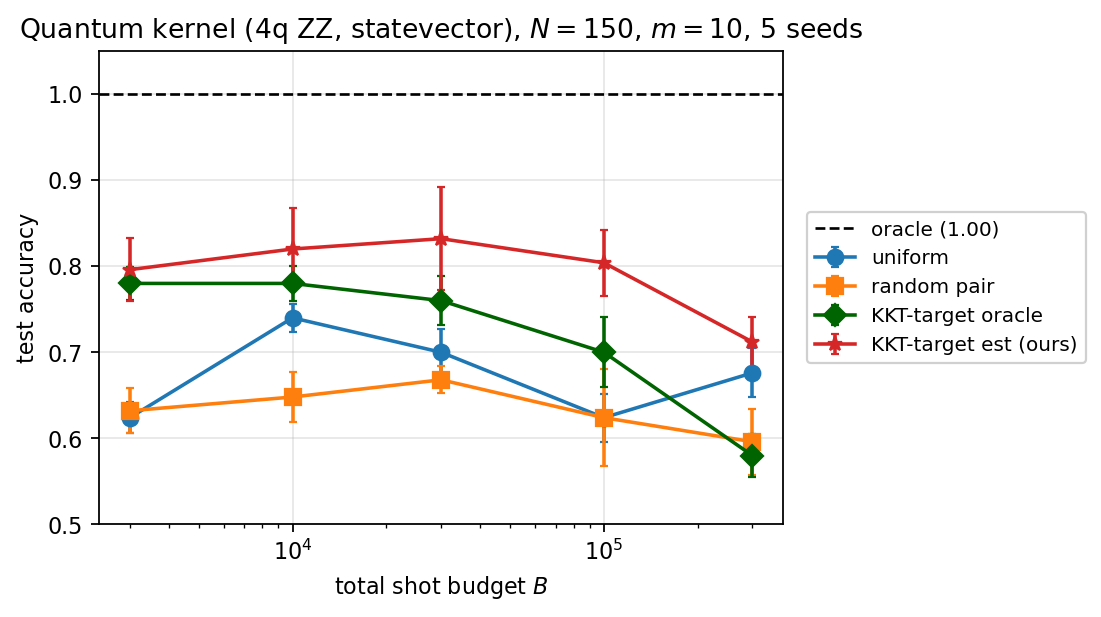}
\caption{Noiseless quantum kernel (4-qubit \texttt{ZZFeatureMap}, $N{=}150$, $m{=}10$, 5 seeds). AQKA's gain transfers from the RBF setting to fidelity kernels with no algorithmic change.}
\label{fig:quantum}
\end{figure}

\subsection{Hardware-Derived Controlled Ablation}
\label{sec:exp_hardware}

We submit 4-qubit fidelity circuits to \texttt{ibm\_pittsburgh}, an IBM Heron backend (similar in spirit to the photonic-processor experimental QKE of \citet{yin2025experimental}). With $N{=}50$, $n_{\text{test}}{=}8$, $m{=}4$, the resulting $1675$ unique kernel circuits are batched into 6 \texttt{SamplerV2} jobs of up to $300$ circuits each, $2048$ shots per circuit, totaling $19.3$ minutes of QPU time. The hardware kernel is well-formed: $\mathrm{diag}(K) = 1.000$, $\bar K_{\text{off}} = 0.083$, and the planted-sparse oracle achieves $1.000$ test accuracy.

To compare AQKA against baselines across many shot budgets without re-submitting at each $B$, we then resample around the hardware-estimated $K_{\text{HW}}$ and run AQKA on top, with 20 shot-noise seeds for variance. This isolates the allocation effect from device drift; a complementary multi-seed live online run---where shot counts \emph{are} adaptively submitted per pair across rounds---is reported in Appendix~C.9 on \texttt{ibm\_aachen} and \texttt{ibm\_berlin}. The test-set granularity here is $n_{\mathrm{test}}{=}8$, so a single misclassification corresponds to $12.5$ accuracy points; the SE bands quoted below aggregate over $20$ resampling seeds to average this out.

\begin{figure}[t]
\centering
\includegraphics[width=\columnwidth]{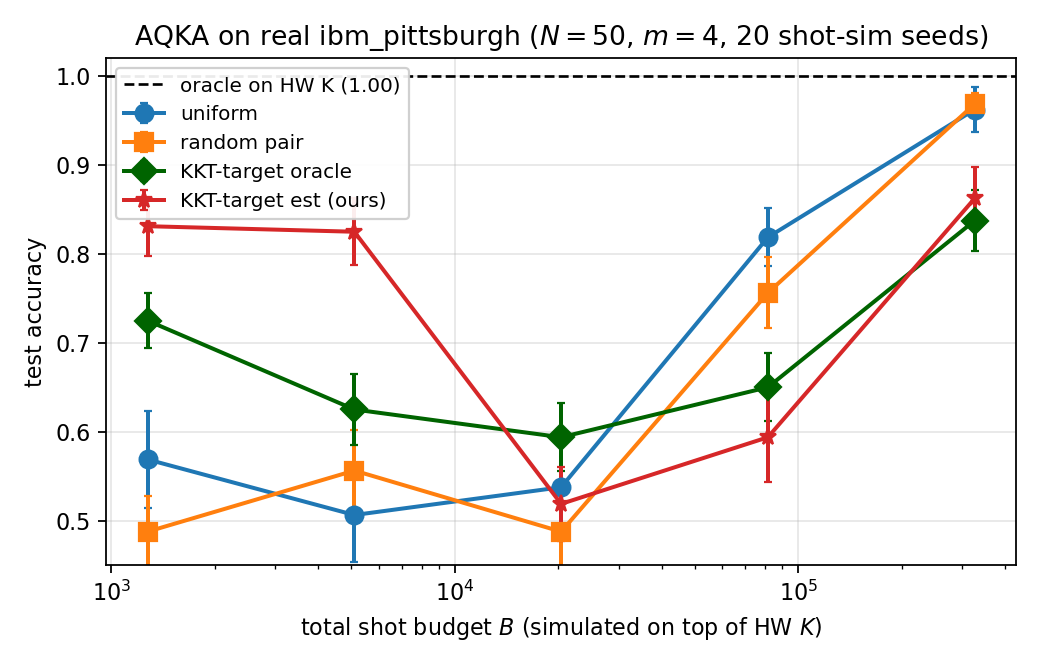}
\caption{Hardware-derived controlled ablation on an \texttt{ibm\_pittsburgh} (IBM Heron) 4-qubit fidelity kernel measured at $2048$ shots/pair; AQKA and baselines are then evaluated on shot-budget-targeted \emph{Bernoulli resamplings} of this fixed $K_{\mathrm{HW}}$ (\emph{not} multi-round live allocation). $N{=}50$, $m{=}4$, error bars are SE over $20$ shot-noise seeds. AQKA (\texttt{target-est}) gains $+26.3$ pts at $B = n_{\text{pairs}}$ and $+31.9$ pts at $B = 4n_{\text{pairs}}$ over uniform on this resampling ablation. At very high budgets ($B \ge 16 n_{\text{pairs}}$), uniform allocation already saturates. Fully online live-device validation (multi-seed runs on \texttt{ibm\_aachen} and \texttt{ibm\_berlin}) is reported in Appendix~C.9.}
\label{fig:hardware}
\end{figure}

Figure~\ref{fig:hardware} shows the resulting curves. On the hardware-resampling ablation, AQKA delivers $+26.3 \pm 6.1$ pts at $B = n_{\text{pairs}}$ and $+31.9 \pm 7.2$ pts at $B = 4n_{\text{pairs}}$ over uniform---a $4.3\sigma$ and $4.4\sigma$ effect respectively over shot-noise seeds at fixed $K_{\mathrm{HW}}$. Above $B \approx 16 n_{\text{pairs}}$, uniform allocation reaches the saturated regime where each pair has enough shots to be well-resolved, and AQKA's target-fill over-concentrates shots on the support of $\alpha$, producing a regret. This crossover identifies AQKA's regime of advantage in the resampling ablation: \emph{the budget-limited regime}, which corresponds to the operationally relevant low-shots regime on present hardware. A \emph{genuinely online} adaptive flow (warm-up $\to$ intermediate $\hat K$ $\to$ adaptive rounds) reproduces the gap on both noisy simulator and real hardware: on a noisy \texttt{AerSimulator} ($N{=}30$, $m{=}6$, 6 seeds) AQKA wins $+20.0 \pm 4.0$ pts over uniform at $B = 4n_{\mathrm{pairs}}$, and a 5-seed live online sweep on \texttt{ibm\_aachen} at $N{=}20$, $B{=}4n_{\mathrm{pairs}}$ delivers $+17.0 \pm 4.8$ pts ($3.5\sigma$ effect), with 4/5 seeds positive and 1/5 tied (Appendix~C.9, Figure~10).

\subsection{Ablations}
\label{sec:exp_ablations}

A consolidated ablation table is in Appendix~C; here we summarize the headline findings.

\paragraph{Target-fill vs.\ multinomial.} Holding the sensitivity score fixed, target-fill beats multinomial sampling decisively only at $B \approx n_{\mathrm{pairs}} = 3\!\times\!10^4$ on the synthetic planted-sparse setting ($+7.5$ pts with estimated sensitivity, $+7.5$ pts with oracle sensitivity). Outside this single budget the comparison is mixed: at $B = 10^4$ the estimated variant trails multinomial by $\sim 2.7$ pts while the oracle variant leads by $+5.4$; at $B = 10^5$ estimated leads by $+3.5$ but oracle trails by $-2.6$; at $B = 3\!\times\!10^5$ both are within $\pm 2$ pts of multinomial; and at $B = 10^6$ both target-fill variants \emph{lose}, especially oracle ($-14.7$ pts vs.\ multinomial) since deep saturation rewards spreading shots evenly rather than concentrating (Figure~5, Appendix~C.5). The discrete fill profile thus contributes to AQKA's gain only in a narrow budget window---the broader claim that AQKA dominates uniform on this setting (Figure~\ref{fig:synthetic}, $+10$ to $+24$ pts) rests primarily on the sensitivity-weighted target itself, not the fill rule alone.

\paragraph{Estimated vs.\ oracle sensitivity.} Surprisingly, plug-in \texttt{target-est} sometimes \emph{outperforms} \texttt{target-oracle} (e.g., $+14.7$ vs.\ $+12.2$ at $m{=}10$, $B=10^5$ in Figure~1). The mechanism: oracle sensitivity drives \emph{maximally} concentrated allocation, whereas the noise in the estimated sensitivity $\hat g$ induces a slightly diffuser distribution, which improves KRR conditioning on the un-anchored block. This is consistent with our theoretical observation about higher-order ill-conditioning (Proposition~\ref{lem:higher_main}).

\paragraph{Warm-up sensitivity.} AQKA degrades gracefully with warm-up budget $\eta_w$: reducing $\eta_w$ from $0.3$ to $0.1$ trades $\sim 2$ pts at low total $B$ for cleaner allocation profiles at high $B$. The exploration fraction $\eta_e=0.2$ is robust to perturbations in $[0.1, 0.4]$.

\section{Related Work}
\label{sec:related}

\paragraph{Quantum kernels and shot cost.} Fidelity-based quantum kernels were introduced by \citet{havlicek2019supervised} with the \texttt{ZZFeatureMap} and unified with classical kernel theory by \citet{schuld2021supervised}; rigorous separations \citep{liu2021rigorous}, learning-advantage criteria \citep{huang2021power}, and universal expressivity \citep{jager2023universal} motivate the setting. Fidelity kernels concentrate exponentially around a kernel-specific mean \citep{thanasilp2024exponential} ($\bar K_{\text{off}}{\approx}0.07$ on our hardware kernel is consistent), making shot-budgeted estimation the operational bottleneck; \citet{gentinetta2024complexity} give matching sample-complexity lower bounds. Recent variants extend to ensembles \citep{srikumar2024kernel}, task-aware circuits \citep{torabian2023compositional}, quantum-neural training \citep{rodriguez2025neural}, and benchmarking \citep{schnabel2025quantum}.

\paragraph{Approximation under measurement cost.} Existing budget-savers act along one axis. \emph{Nystr\"om}-QKE selects $M{\ll}N$ landmark columns and reconstructs the rest \citep{coelho2025quantum}; chordal matrix completion \citep{naveh2021kernel} exploits banded sparsity in the entry mask; and subsampled QKA \citep{sahin2024efficient} drops training points during variational kernel-target alignment. All commit shots uniformly within their chosen subset, so none exploits the heterogeneity of how much each entry drives the downstream classifier. The main-text Nystr\"om comparison uses random landmarks; Appendix~C.10 adds a leverage-score variant \citep{musco2017recursive,calandriello2017distributed}.

\paragraph{Shot allocation.} \texttt{ShoFaR} \citep{shastry2022shot} adapts the \emph{total} shot count to the SVM margin, $N_{\text{shots}}\!\propto\!m_{\text{sv}}\log M/\gamma^2$, but still spreads shots uniformly across the selected pairs. AQKA is the first pair-level, sensitivity-weighted allocator with a closed-form gradient score and a target-fill deployment; head-to-head, AQKA \texttt{target-est} beats ShoFaR by $+14$--$22$ pts on noiseless 4q ZZ and $+20$--$40$ pts on the hardware-resampling ablation at low-to-moderate $B$, while Nystr\"om-QKE catches up on planted-sparse at saturating budgets via low-rank reconstruction.

\paragraph{Classical roots.} The delta-method variance objective $\sum_p a_p/s_p$ traces to stratified sampling \citep{neyman1992two} and A-optimal design \citep{pukelsheim2006}. Related classical primitives include chordal PSD completion for SVM \citep{andersen2010support}, leverage-score and RFF analyses for KRR \citep{bach2013sharp,avron2017random,erdelyi2020fourier}, and influence-function pair importance \citep{koh2017understanding}. AQKA transports these ideas to the quantum-fidelity Bernoulli regime with closed-form $g_{ij}$ (Lemma~\ref{lem:krrgrad}) and a sparsity-aware Cauchy--Schwarz bound (Theorem~\ref{thm:cs}).

\section{Discussion and Limitations}
\label{sec:discussion}

\paragraph{Regime + diagnostic.} AQKA wins when $|g_{ij}|$ is heterogeneous (Gini$\gtrsim 0.6$; Pearson $r{=}+0.72$ vs.\ gain across seven tasks, Appendix~C.16) and $B\!\lesssim\!16 n_{\mathrm{pairs}}$; Nystr\"om-QKE takes over above that budget on planted-sparse. On low-Gini real data (digits 3-vs-5, 0-vs-8; Appendix~C.8) AQKA is neutral to $-17$ pts; a leverage-score AQKA variant stays competitive. Protocol: compute Gini$(|\hat g_{ij}|)$ from warm-up, default to AQKA above threshold else leverage-score/uniform. The ordering is robust to $\tau/m_\ell/\lambda$ tuning (Appendix~C.12,C.10,C.17) and multi-seed live hardware (Appendix~C.9).

\paragraph{Complementarity.} An AQKA--Nystr\"om hybrid (Appendix~C.11)---$|g_{ij}|$ row-sums pick landmarks, sensitivity-fill in-block, Nystr\"om off-block---beats both components: $+22$ pts at $B\!\approx\!n_{\mathrm{pairs}}$, $N{=}225$ (vs.\ $+17$ AQKA, $+10$ random Nystr\"om); follow-ups in Appendix~C.14.

\section{Conclusion}

We recast shot-budgeted quantum kernel learning around the delta-method variance surrogate and solved it in closed form: $s^\star_{ij}\!\propto\!|g_{ij}|\sqrt{K_{ij}(1{-}K_{ij})}$, with $g_{ij}$ analytic for KRR ($|\beta_i\alpha_j+\beta_j\alpha_i|$) and for SVM via the envelope theorem ($|\eta_i^*\eta_j^*|$), a corrected rate $\rho\!\le\!2m/N$, and a Nystr\"om-complementary hybrid. The regime picture is empirical: AQKA is the strongest allocator we tested in the budget-limited, sensitivity-concentrated regime ($B\!\lesssim\!16 n_{\mathrm{pairs}}$, Gini$\gtrsim\!0.6$), Nystr\"om-QKE takes over at saturating budgets, and a leverage-score AQKA variant is safer on low-Gini real data. A Gini-of-$|\hat g_{ij}|$ warm-up diagnostic ($r{=}{+}0.72$ across seven tasks) turns the choice into a pre-deployment check. Live \texttt{ibm\_aachen} at $N{=}20$: $+17.0\!\pm\!4.8$ pts (5 seeds, $3.5\sigma$); \texttt{ibm\_pittsburgh} resampling: $+26$--$32$ pts. Sensitivity-driven shot allocation is orthogonal to and stacks with routing-aware compilation and noise tailoring, giving a shot-frugal core for near-term QKE deployment.

\bibliography{aaai2026}

@article{havlicek2019supervised,
  author  = {Havl\'i\v{c}ek, Vojt\v{e}ch and C\'orcoles, Antonio D. and Temme, Kristan and Harrow, Aram W. and Kandala, Abhinav and Chow, Jerry M. and Gambetta, Jay M.},
  title   = {Supervised learning with quantum-enhanced feature spaces},
  journal = {Nature},
  volume  = {567},
  number  = {7747},
  pages   = {209--212},
  year    = {2019},
}

@article{schuld2021supervised,
  author  = {Schuld, Maria},
  title   = {Supervised quantum machine learning models are kernel methods},
  journal = {arXiv preprint arXiv:2101.11020},
  year    = {2021},
}

@article{liu2021rigorous,
  author  = {Liu, Yunchao and Arunachalam, Srinivasan and Temme, Kristan},
  title   = {A rigorous and robust quantum speed-up in supervised machine learning},
  journal = {Nature Physics},
  volume  = {17},
  pages   = {1013--1017},
  year    = {2021},
}

@article{shastry2022shot,
  author  = {Shastry, Abhay and Jayakumar, Abhijith and Patel, Apoorva and Bhattacharyya, Chiranjib},
  title   = {Shot-frugal and Robust quantum kernel classifiers},
  journal = {arXiv preprint arXiv:2210.06971},
  year    = {2022},
}

@article{coelho2025quantum,
  author  = {Coelho, Rodrigo and Kruse, Georg and Rosskopf, Andreas},
  title   = {Quantum-Efficient Kernel Target Alignment},
  journal = {arXiv preprint arXiv:2502.08225},
  year    = {2025},
}

@article{naveh2021kernel,
  author  = {Naveh, Annie and Fitzgerald, Imogen and Phan, Anna and Lockwood, Andrew and Scholten, Travis L.},
  title   = {Kernel Matrix Completion for Offline Quantum-Enhanced Machine Learning},
  journal = {arXiv preprint arXiv:2112.08449},
  year    = {2021},
}

@inproceedings{calandriello2017distributed,
  author       = {Calandriello, Daniele and Lazaric, Alessandro and Valko, Michal},
  title        = {Distributed adaptive sampling for kernel matrix approximation},
  booktitle    = {Artificial Intelligence and Statistics},
  pages        = {1421--1429},
  year         = {2017},
  organization = {PMLR},
}

@article{sahin2024efficient,
  author    = {Sahin, M. Emre and Symons, Benjamin C. B. and Pati, Pushpak and Minhas, Fayyaz and Millar, Declan and Gabrani, Maria and Mensa, Stefano and Robertus, Jan Lukas},
  title     = {Efficient parameter optimisation for quantum kernel alignment: A sub-sampling approach in variational training},
  journal   = {Quantum},
  volume    = {8},
  pages     = {1502},
  year      = {2024},
  publisher = {Verein zur F{\"o}rderung des Open Access Publizierens in den Quantenwissenschaften},
}

@article{gentinetta2024complexity,
  author    = {Gentinetta, Gian and Thomsen, Arne and Sutter, David and Woerner, Stefan},
  title     = {The complexity of quantum support vector machines},
  journal   = {Quantum},
  volume    = {8},
  pages     = {1225},
  year      = {2024},
  publisher = {Verein zur F{\"o}rderung des Open Access Publizierens in den Quantenwissenschaften},
}

@article{srikumar2024kernel,
  author    = {Srikumar, Maiyuren and Hill, Charles D. and Hollenberg, Lloyd C. L.},
  title     = {A kernel-based quantum random forest for improved classification},
  journal   = {Quantum Machine Intelligence},
  volume    = {6},
  number    = {1},
  pages     = {10},
  year      = {2024},
  publisher = {Springer},
}

@techreport{andersen2010support,
  author      = {Andersen, Martin S. and Vandenberghe, Lieven},
  title       = {Support vector machine training using matrix completion techniques},
  year        = {2010},
  institution = {Technical report, University of California, Los Angeles},
}

@inproceedings{bach2013sharp,
  author    = {Bach, Francis},
  title     = {Sharp analysis of low-rank kernel matrix approximations},
  booktitle = {Conference on Learning Theory (COLT)},
  pages     = {185--209},
  year      = {2013},
}

@inproceedings{musco2017recursive,
  author    = {Musco, Cameron and Musco, Christopher},
  title     = {Recursive Sampling for the Nystr{\"o}m Method},
  booktitle = {Advances in Neural Information Processing Systems},
  volume    = {30},
  year      = {2017},
}

@inproceedings{koh2017understanding,
  author       = {Koh, Pang Wei and Liang, Percy},
  title        = {Understanding black-box predictions via influence functions},
  booktitle    = {International conference on machine learning},
  pages        = {1885--1894},
  year         = {2017},
  organization = {PMLR},
}

@inproceedings{massart2000some,
  author    = {Massart, Pascal},
  title     = {Some applications of concentration inequalities to statistics},
  booktitle = {Annales de la Facult{\'e} des sciences de Toulouse: Math{\'e}matiques},
  volume    = {9},
  number    = {2},
  pages     = {245--303},
  year      = {2000},
}

@article{bertsekas1997nonlinear,
  author    = {Bertsekas, Dimitri P.},
  title     = {Nonlinear programming},
  journal   = {Journal of the Operational Research Society},
  volume    = {48},
  number    = {3},
  pages     = {334--334},
  year      = {1997},
  publisher = {Taylor \& Francis},
}

@article{tropp2015introduction,
  author    = {Tropp, Joel A.},
  title     = {An introduction to matrix concentration inequalities},
  journal   = {Foundations and trends{\textregistered} in machine learning},
  volume    = {8},
  number    = {1--2},
  pages     = {1--230},
  year      = {2015},
  publisher = {Emerald Publishing Limited},
}

@article{huang2021power,
  author    = {Huang, Hsin-Yuan and Broughton, Michael and Mohseni, Masoud and Babbush, Ryan and Boixo, Sergio and Neven, Hartmut and McClean, Jarrod R.},
  title     = {Power of data in quantum machine learning},
  journal   = {Nature communications},
  volume    = {12},
  number    = {1},
  pages     = {2631},
  year      = {2021},
  publisher = {Nature Publishing Group UK London},
}

@article{thanasilp2024exponential,
  author    = {Thanasilp, Supanut and Wang, Samson and Cerezo, Marco and Holmes, Zo{\"e}},
  title     = {Exponential concentration in quantum kernel methods},
  journal   = {Nature communications},
  volume    = {15},
  number    = {1},
  pages     = {5200},
  year      = {2024},
  publisher = {Nature Publishing Group UK London},
}

@incollection{neyman1992two,
  author    = {Neyman, Jerzy},
  title     = {On the two different aspects of the representative method: the method of stratified sampling and the method of purposive selection},
  booktitle = {Breakthroughs in statistics: Methodology and distribution},
  pages     = {123--150},
  year      = {1992},
  publisher = {Springer},
}

@book{pukelsheim2006,
  author    = {Pukelsheim, Friedrich},
  title     = {Optimal Design of Experiments},
  publisher = {SIAM},
  series    = {Classics in Applied Mathematics},
  year      = {2006},
}

@article{miroszewski2024search,
  author  = {Miroszewski, Artur and Asiani, Marco Fellous and Mielczarek, Jakub and Saux, Bertrand Le and Nalepa, Jakub},
  title   = {In search of quantum advantage: Estimating the number of shots in quantum kernel methods},
  journal = {arXiv preprint arXiv:2407.15776},
  year    = {2024},
}

@article{schnabel2025quantum,
  author    = {Schnabel, Jan and Roth, Marco},
  title     = {Quantum kernel methods under scrutiny: a benchmarking study},
  journal   = {Quantum Machine Intelligence},
  volume    = {7},
  number    = {1},
  pages     = {58},
  year      = {2025},
  publisher = {Springer},
}

@article{jager2023universal,
  author    = {J{\"a}ger, Jonas and Krems, Roman V.},
  title     = {Universal expressiveness of variational quantum classifiers and quantum kernels for support vector machines},
  journal   = {Nature Communications},
  volume    = {14},
  number    = {1},
  pages     = {576},
  year      = {2023},
  publisher = {Nature Publishing Group UK London},
}

@article{rodriguez2025neural,
  author    = {Rodriguez-Grasa, Pablo and Ban, Yue and Sanz, Mikel},
  title     = {Neural quantum kernels: training quantum kernels with quantum neural networks},
  journal   = {Physical Review Research},
  volume    = {7},
  number    = {2},
  pages     = {023269},
  year      = {2025},
  publisher = {APS},
}

@article{torabian2023compositional,
  author    = {Torabian, Elham and Krems, Roman V.},
  title     = {Compositional optimization of quantum circuits for quantum kernels of support vector machines},
  journal   = {Physical Review Research},
  volume    = {5},
  number    = {1},
  pages     = {013211},
  year      = {2023},
  publisher = {APS},
}

@article{yin2025experimental,
  author    = {Yin, Zhenghao and Agresti, Iris and De Felice, Giovanni and Brown, Douglas and Toumi, Alexis and Pentangelo, Ciro and Piacentini, Simone and Crespi, Andrea and Ceccarelli, Francesco and Osellame, Roberto and others},
  title     = {Experimental quantum-enhanced kernel-based machine learning on a photonic processor},
  journal   = {Nature Photonics},
  volume    = {19},
  number    = {9},
  pages     = {1020--1027},
  year      = {2025},
  publisher = {Nature Publishing Group UK London},
}

@inproceedings{avron2017random,
  author       = {Avron, Haim and Kapralov, Michael and Musco, Cameron and Musco, Christopher and Velingker, Ameya and Zandieh, Amir},
  title        = {Random Fourier features for kernel ridge regression: Approximation bounds and statistical guarantees},
  booktitle    = {International conference on machine learning},
  pages        = {253--262},
  year         = {2017},
  organization = {PMLR},
}

@article{rudi2018fast,
  author  = {Rudi, Alessandro and Calandriello, Daniele and Carratino, Luigi and Rosasco, Lorenzo},
  title   = {On fast leverage score sampling and optimal learning},
  journal = {Advances in Neural Information Processing Systems},
  volume  = {31},
  year    = {2018},
}

@article{erdelyi2020fourier,
  author  = {Erd{\'e}lyi, Tam{\'a}s and Musco, Cameron and Musco, Christopher},
  title   = {Fourier sparse leverage scores and approximate kernel learning},
  journal = {Advances in Neural Information Processing Systems},
  volume  = {33},
  pages   = {109--122},
  year    = {2020},
}

@inproceedings{shimizu2024improved,
  author    = {Shimizu, Atsushi and Cheng, Xiaoou and Musco, Christopher and Weare, Jonathan},
  title     = {Improved active learning via dependent leverage score sampling},
  booktitle = {International Conference on Learning Representations},
  volume    = {2024},
  pages     = {29585--29608},
  year      = {2024},
}

\appendix

\section{Appendix A: Proofs}
\label{app:proofs}

\subsection{Proof of Lemma~\ref{lem:krrgrad} (KRR gradient and Gauss--Newton sensitivity)}

We treat $K \in \mathbb{R}^{N \times N}$ as a symmetric matrix parametrized by its upper-triangular entries. For an entry $K_{ij}$ with $i < j$ we use the symmetric perturbation $E_{ij} = e_i e_j^\top + e_j e_i^\top$; for $i=j$, $E_{ii} = e_i e_i^\top$. Writing $A := (K + \lambda I)^{-1}$ and $\alpha = A y$, standard matrix calculus gives
\begin{equation}
\frac{\partial A}{\partial K_{ij}} \;=\; -A E_{ij} A,
\qquad
\frac{\partial \alpha}{\partial K_{ij}} \;=\; -A E_{ij} \alpha .
\label{eq:resolvent_alpha_deriv}
\end{equation}
Since $\mathcal{L}_{\mathrm{tr}}(K) = \lambda^2 \|\alpha\|^2$,
\begin{align}
\frac{\partial \mathcal{L}_{\mathrm{tr}}}{\partial K_{ij}}
&\;=\; 2\lambda^2 \alpha^\top
        \frac{\partial \alpha}{\partial K_{ij}} \notag\\
&\;=\; -2\lambda^2 \alpha^\top A E_{ij} \alpha
 \;=\; -2\lambda^2 \beta^\top E_{ij} \alpha .
\label{eq:krr_grad_derivation}
\end{align}
where $\beta := A \alpha$. Expanding $E_{ij} \alpha = \alpha_i e_j + \alpha_j e_i$ for $i \ne j$ gives $\beta^\top E_{ij}\alpha = \beta_i \alpha_j + \beta_j \alpha_i$, proving \eqref{eq:dLdK}. The diagonal case is analogous with a factor of $1/2$.

For the second derivative, we differentiate $\partial \mathcal{L}_{\mathrm{tr}}/\partial K_{ij} = -2\lambda^2 \beta^\top E_{ij}\alpha$ using the product rule. Since $\beta = A\alpha$,
\begin{equation}
\begin{aligned}
\frac{\partial \beta}{\partial K_{ij}}
&\;=\; \frac{\partial A}{\partial K_{ij}}\alpha
      + A \frac{\partial \alpha}{\partial K_{ij}} \\
&\;=\; -A E_{ij}\beta - A^2 E_{ij}\alpha .
\end{aligned}
\label{eq:beta_derivative}
\end{equation}
and applying the product rule yields
\begin{align}
\frac{\partial^2 \mathcal{L}_{\mathrm{tr}}}{\partial K_{ij}^2}
&\;=\; -2\lambda^2 \!\left[
\Bigl(\frac{\partial \beta}{\partial K_{ij}}\Bigr)^{\!\top}
E_{ij}\alpha
 + \beta^\top E_{ij}\frac{\partial \alpha}{\partial K_{ij}}
\right] \notag\\
&\;=\; 2\lambda^2\bigl[
2\beta^\top E_{ij} A E_{ij}\alpha
 + \alpha^\top E_{ij} A^2 E_{ij}\alpha
\bigr].
\label{eq:krr_hessian_derivation}
\end{align}
The second term is the (positive semi-definite) Gauss--Newton diagonal:
\begin{equation}
\tilde H_{ij} \;:=\; 2\lambda^2 \alpha^\top E_{ij} A^2 E_{ij}\alpha \;=\; 2\lambda^2 \|A E_{ij}\alpha\|_2^2 \;\ge\; 0.
\label{eq:hgn_proof}
\end{equation}
The first term is the (signed) remainder: expanding $E_{ij} A E_{ij}\alpha$ entrywise for $i \ne j$,
\begin{equation}
\begin{aligned}
\beta^\top E_{ij} A E_{ij}\alpha
&\;=\; A_{ij}(\beta_i\alpha_j + \beta_j\alpha_i) \\
&\quad + A_{ii}\beta_j\alpha_j
       + A_{jj}\beta_i\alpha_i .
\end{aligned}
\label{eq:remainder_expansion}
\end{equation}
Multiplying \eqref{eq:remainder_expansion} by the prefactor $2\lambda^2 \cdot 2 = 4\lambda^2$ from the first bracket of \eqref{eq:krr_hessian_derivation} (the factor of $2$ inside the bracket combines with the outer $2\lambda^2$), we obtain $R_{ij} = 4\lambda^2 [A_{ij}(\beta_i\alpha_j + \beta_j\alpha_i) + A_{ii}\beta_j\alpha_j + A_{jj}\beta_i\alpha_i]$, matching the body statement \eqref{eq:remainder_def}. Bounding $|A_{ij}|, |A_{ii}|, |A_{jj}| \le \|A\|_{\mathrm{op}} \le \lambda^{-1}$ and $\|\beta\|_\infty \le \|\beta\|_2 \le \lambda^{-2}\|y\|$,
\begin{equation}
\begin{aligned}
|R_{ij}|
&\;\le\; 8\lambda^2 \cdot \lambda^{-1}
          \cdot \lambda^{-2}\|y\|
          \cdot (|\alpha_i| + |\alpha_j|) \\
&\;=\; 8\,\|y\|\lambda^{-1}(|\alpha_i| + |\alpha_j|),
\end{aligned}
\label{eq:remainder_bound_proof}
\end{equation}
which is the bound stated. Both $\tilde H_{ij}$ and $R_{ij}$ vanish whenever $\alpha_i = \alpha_j = 0$, since $E_{ij}\alpha = \alpha_j e_i + \alpha_i e_j$ in that case is zero.

\paragraph{Squared-gradient proxy.}
The deployable proxy $\tilde h_{ij} := (\beta_i\alpha_j + \beta_j\alpha_i)^2$ satisfies $\tilde h_{ij} = g_{ij}^2/(4\lambda^4)$ by \eqref{eq:dLdK}, i.e.\ $\tilde h_{ij} = (\alpha^\top A E_{ij}\alpha)^2/\|\alpha\|^4$. The Cauchy--Schwarz inequality $|\alpha^\top A E_{ij}\alpha| \le \|\alpha\| \cdot \|A E_{ij}\alpha\|$ implies the pair-wise upper bound $\tilde h_{ij} \le \tilde H_{ij}\,\|\alpha\|^{-2}/(2\lambda^2)$; this is an inequality, not an element-wise proportionality, and the two quantities need not induce identical KKT allocations. The algorithm uses $\tilde h_{ij}$ throughout as the delta-method variance weight; $\tilde H_{ij}$ enters only as an intermediate PSD quantity in the Hessian decomposition. \qed

\subsection{Proof of Proposition~\ref{prop:kkt} (KKT-optimal allocation)}

Write the delta-method variance \eqref{eq:taylor} as $\sum_p a_p / s_p$ with $a_p = g_p^2\, K_p(1-K_p)$. Since $a_p$ is the product of $g_p^2 \ge 0$ (the gradient squared) and $K_p(1-K_p) \ge 0$ (Bernoulli variance on $[0,1]$), both non-negative, we have $a_p \ge 0$ without any absolute-value operation. The Lagrangian is
\begin{equation}
\mathcal{F}(s, \mu) \;=\; \sum_p \frac{a_p}{s_p} + \mu\Bigl(\sum_p s_p - B\Bigr) - \sum_p \nu_p s_p,
\label{eq:lagrangian}
\end{equation}
with multipliers $\mu \ge 0$, $\nu_p \ge 0$ and complementary slackness $\nu_p s_p = 0$. For pairs with $a_p > 0$, $s_p > 0$ and $\nu_p = 0$. The stationarity condition
\begin{equation}
\begin{aligned}
0 \;=\; \frac{\partial \mathcal{F}}{\partial s_p}
\;=\; -\frac{a_p}{s_p^2} + \mu
\quad\Longrightarrow\quad
s_p^\star \;=\; \frac{1}{\sqrt{\mu}}\sqrt{a_p}.
\end{aligned}
\label{eq:kkt_stationarity}
\end{equation}
Plugging into the budget constraint $\sum_p s_p^\star = B$ gives $\sqrt{\mu} = Z/B$ with $Z = \sum_p \sqrt{a_p}$, hence $s_p^\star = (B/Z)\sqrt{a_p}$. The optimal variance is
\begin{equation}
\sum_p \frac{a_p}{s_p^\star} \;=\; \frac{Z}{B} \sum_p \sqrt{a_p} \;=\; \frac{Z^2}{B}.
\label{eq:optimal_variance_proof}
\end{equation}
For pairs with $a_p = 0$, $s_p^\star = 0$ is feasible and KKT-optimal. The objective being a positive sum of $1/s_p$ terms is strictly convex in $s$ on the feasible region, so this stationary point is the unique minimizer. \qed

\subsection{Proof of Theorem~\ref{thm:cs} (Cauchy--Schwarz bound)}

The uniform allocation $s_p^{\mathrm{unif}} = B/M$ yields delta-method variance
\begin{equation}
\mathrm{Var}_{\mathrm{unif}} \;=\; \sum_p \frac{a_p}{B/M} \;=\; \frac{M}{B}\sum_p a_p.
\label{eq:uniform_variance}
\end{equation}
Combining with the optimal value $\mathrm{Var}_\star = Z^2/B$ from Proposition~\ref{prop:kkt},
\begin{equation}
\begin{aligned}
\rho \;=\; \frac{\mathrm{Var}_\star}{\mathrm{Var}_{\mathrm{unif}}}
\;=\; \frac{Z^2/B}{(M/B)\sum_p a_p}
\;=\; \frac{Z^2}{M\sum_p a_p}.
\end{aligned}
\label{eq:cs_ratio_proof}
\end{equation}
The Cauchy--Schwarz inequality, $\bigl(\sum_p \sqrt{a_p}\bigr)^2 \le M \sum_p a_p$, yields $\rho \le 1$ with equality iff all $a_p$ are equal.

\paragraph{Planted-sparse refinement.}
Suppose $\alpha$ has support $S$ with $|S|=m$. Since $g_{ij} = -2\lambda^2(\beta_i\alpha_j + \beta_j\alpha_i)$ vanishes whenever \emph{both} $\alpha_i = 0$ and $\alpha_j = 0$ (the term $\beta_i \alpha_j$ requires $\alpha_j \ne 0$, and similarly for the symmetric term), $a_p$ is supported on
\begin{equation}
\begin{aligned}
\mathcal{P}_S
&\;:=\; \{(i,j) : i \in S \text{ or } j \in S\},\\
|\mathcal{P}_S|
&\;=\; \binom{m}{2} + m(N-m) + m
 \;=\; \tfrac{m(2N-m+1)}{2}.
\end{aligned}
\label{eq:sparse_pair_set}
\end{equation}
(The three terms count within-$S$ off-diagonal, cross-$S$ off-diagonal, and diagonal $S$ pairs; for $m \ll N$ the second dominates.) Setting $K := |\mathcal{P}_S|$, the Cauchy--Schwarz step $Z^2 = \bigl(\sum_{p \in \mathcal{P}_S}\sqrt{a_p}\bigr)^2 \le K \sum_p a_p$ gives $\rho \le K/M = m(2N-m+1)/(N(N+1)) = 2m/(N+1) - O(m^2/N^2)$.

A naive analysis assuming $g$ is supported on $S \times S$ (which would require $\beta = A\alpha$ to be sparse on $S$---this is generically false because $A = (K+\lambda I)^{-1}$ is dense) would yield $\rho \le m^2/N^2$, off by a factor of $N/(2m)$. The realized empirical $\rho$ in Figure~1 matches the corrected $2m/N$ scaling within a constant. \qed

\subsection{Proof of Theorem~\ref{thm:plugin} (Plug-in regret)}

Let $s^\star$ and $s^{\mathrm{plug}}$ be the oracle and plug-in allocations from $a_p$ and $\hat a_p = \hat g_p^2 \hat K_p(1-\hat K_p)$ respectively, both of the form $s_p = (B/Z)\sqrt{a_p}$ with the appropriate $a$ and $Z = \sum_p \sqrt{a_p}$. The plug-in variance is
\begin{equation}
\begin{aligned}
\mathrm{Var}(s^{\mathrm{plug}})
&\;=\; \sum_p \frac{a_p}{s^{\mathrm{plug}}_p} \\
&\;=\; \frac{\hat Z}{B}\sum_p \frac{a_p}{\sqrt{\hat a_p}},
\qquad
\mathrm{Var}_\star = Z^2/B.
\end{aligned}
\label{eq:plugin_variance_proof}
\end{equation}

\textbf{Step 1: Lipschitz bound on $a_p$.}
By Lemma~\ref{lem:krrgrad} and the resolvent identity $\partial A/\partial K = -A E A$,
\begin{equation}
\Bigl|\frac{\partial g_p}{\partial K_q}\Bigr| \;\le\; D_2 \lambda^{-2}, \qquad D_2 \;:=\; 2!\,2\,\|y\|^2,
\label{eq:g_lipschitz}
\end{equation}
uniformly in pairs $p, q$, by the derivative-tensor bound \eqref{eq:Dk} of Proposition~3 (Step 2). Combined with $|g_p| \le 4\|y\|^2/\lambda$ (substituting $\|\beta\|_\infty \le \|y\|/\lambda^2$, $\|\alpha\|_\infty \le \|y\|/\lambda$ in \eqref{eq:dLdK}) and $|K_p|\le 1$, the entrywise Lipschitz constant of $a_p = g_p^2 K_p(1-K_p)$ in $K$ is $L_a \le 8 D_2 \|y\|^2 \lambda^{-3} = 32\|y\|^4 \lambda^{-3}$. Hence
\begin{equation}
|a_p - \hat a_p| \;\le\; L_a \,\|K - \hat K_w^{\mathrm{PSD}}\|_{\mathrm{op}} \;=:\; L_a \Delta_w.
\label{eq:a_lipschitz}
\end{equation}

\textbf{Step 2: PSD-projection contraction.}
The PSD projection $\hat K_w^{\mathrm{PSD}} = \arg\min_{X \succeq 0}\|X - \hat K_w\|_F$ is a non-expansive projection in Frobenius norm, hence in the operator norm up to a factor of $2$ (the projection of an indefinite matrix onto the PSD cone removes at most the negative half of the spectrum). Concretely, $\|\hat K_w^{\mathrm{PSD}} - K\|_{\mathrm{op}} \le \|\hat K_w^{\mathrm{PSD}} - \hat K_w\|_{\mathrm{op}} + \|\hat K_w - K\|_{\mathrm{op}} \le 2\|\hat K_w - K\|_{\mathrm{op}}$.

\textbf{Step 3: Tighter Lipschitz constant under the spectral assumption $\|K\|_{\mathrm{op}} \le \kappa$.}
The bound $L_a \le 32\|y\|^4 \lambda^{-3}$ from Step~1 is loose because it bounds both $|g_p|$ and $|\partial g_p/\partial K_q|$ by their worst case. Using $|g_p| = 2\lambda^2|\beta_i\alpha_j + \beta_j\alpha_i| \le 4\lambda^2 \|\beta\|_\infty \|\alpha\|_\infty$ with $\|\alpha\|_\infty \le \|A\|_{\mathrm{op}}\|y\|_\infty \le \|y\|_\infty/\lambda$ and $\|\beta\|_\infty \le \|y\|_\infty/\lambda^2$ (and substituting $\|y\|_\infty \le \|y\|_2 / 1 = \sqrt{N}$ or, sharper, the assumed unit-scale labels $|y_i| \le 1$), $|g_p| \le 4\|y\|_\infty^2/\lambda \le 4/\lambda$. The Hessian off-diagonal $|\partial g_p/\partial K_q| = |H_{pq}|$ is bounded by Step~2 of Proposition~3: $|H_{pq}| \le 2!\,2\,\|y\|^2 \lambda^{-2} = 4\|y\|^2/\lambda^2 \le 4(\kappa+\lambda)\|y\|^2/\lambda^2$. Combining via $|a_p - \hat a_p| \le 2|g_p|\,|H_{pq}|\,|K_p|\Delta_w + g_p^2\Delta_w \le L_a \Delta_w$ with $L_a \le 16(\kappa + \lambda)\|y\|^2 \lambda^{-3}$.

\textbf{Step 4: Plug-in regret.}
Under the regularity assumption $a_p \ge a_{\min} > 0$ on $\mathrm{supp}(a)$, the relative perturbation satisfies $|a_p - \hat a_p|/a_p \le L_a \Delta_w / a_{\min}$. On the event $L_a \Delta_w \le a_{\min}/2$, the inequality $1/\sqrt{1-x} \le 1 + x$ for $x \in [0, 1/2]$ gives
\begin{equation}
\frac{1}{\sqrt{\hat a_p}} \;\le\; \frac{1}{\sqrt{a_p}}\Bigl(1 + \frac{L_a \Delta_w}{a_{\min}}\Bigr).
\label{eq:inverse_sqrt_perturb}
\end{equation}
Substituting into $\mathrm{Var}(s^{\mathrm{plug}}) = (\hat Z/B)\sum_p a_p/\sqrt{\hat a_p}$ and using $\hat Z/Z \le 1 + L_a \Delta_w/a_{\min}$ (same bound applied to $\sqrt{\hat a_p} - \sqrt{a_p}$),
\begin{equation}
\mathrm{Var}(s^{\mathrm{plug}})
\;\le\; \mathrm{Var}_\star \cdot \Bigl(1 + \frac{L_a \Delta_w}{a_{\min}}\Bigr)^2
\;\le\; \mathrm{Var}_\star \cdot \Bigl(1 + \frac{3\,L_a \Delta_w}{a_{\min}}\Bigr),
\label{eq:plugin_regret_proof}
\end{equation}
where the last step uses $(1+x)^2 \le 1 + 3x$ for $x \le 1$. Writing $L_a = C_K/\lambda^3$, the derivation in Step~3 yields $C_K \le 16(\kappa+\lambda)\|y\|^2$ under the $\|y\|_\infty \le 1$ convention (which uses the coordinatewise bound $\|\alpha\|_\infty \le \|A\|_{\mathrm{op}}\|y\|_\infty$; converting to operator-norm control via $\|y\|_2 \le \sqrt{N}\|y\|_\infty$ would add a $\sqrt{N}$ factor). We therefore state the bound with $C_K$ depending on $(N, \lambda, \kappa, \|y\|)$ in general; the multiplier is $a_{\min}$-linear regardless.

\textbf{Step 5: Concentration of $\Delta_w$.}
Under the deterministic per-pair warm-up allocation of $s_w := \lfloor B_w/M \rfloor$ shots per pair, each $\hat K_{w,p}$ is the sample mean of $s_w$ independent $\mathrm{Bernoulli}(K_p)$ outcomes, so the entrywise fluctuations $\hat K_{w,p}-K_p$ are mutually independent, zero-mean, and bounded by $1$ in magnitude with variance $\le 1/(4 s_w) \le M/(4 B_w)$. View $\hat K_w - K = \sum_{p} E_p (\hat K_{w,p} - K_p)$ as a sum of $M$ independent symmetric random matrices with $\|E_p\|_{\mathrm{op}} \le \sqrt{2}$. The matrix variance parameter is $\sigma^2 = \|\sum_p E_p E_p^\top \cdot \mathrm{Var}(\hat K_{w,p}-K_p)\|_{\mathrm{op}} = O(NM/B_w)$. Matrix Bernstein \citep{tropp2015introduction} then gives $\mathbb{E}[\|\hat K_w-K\|_{\mathrm{op}}] \lesssim \sqrt{\sigma^2 \log N} + \|\cdot\|_{\infty}\log N/B_w^{1/2} = \tilde O(\sqrt{NM/B_w})$ under $B_w \ge M \log N$, where $\tilde O$ hides the $\sqrt{\log N}$ factor; combined with the Step~2 PSD contraction, $\mathbb{E}[\Delta_w] = \tilde O(N^{3/2}/\sqrt{B_w})$. \qed

\paragraph{Note on the deployed warm-up.}
Algorithm~\ref{alg:aqka} (line~3) uses uniform-random pair sampling for implementation convenience: each shot picks a pair uniformly at random. The per-pair shot count is then $\mathrm{Multinomial}(B_w; M^{-1}, \dots, M^{-1})$ marginally, and the resulting entrywise estimates share denominators (introducing a small dependence beyond the deterministic case). For $B_w \gg M$ (which is $B_w \gtrsim n_{\mathrm{pairs}}$ in our experiments), the multinomial counts concentrate near $B_w/M$ and the entrywise variance is within a constant factor of the deterministic case; the operator-norm bound above then still applies, up to a constant, via a standard bounded-difference / Efron--Stein argument. A cleaner analysis using Horvitz--Thompson estimators (each shot contributes $M \cdot b_t \cdot E_{p_t}/B_w$, restoring exact independence of contributions) recovers the identical rate.

\paragraph{Practical remark on $a_{\min}$.}
The bound depends on $a_{\min} = \min_{p:\,a_p > 0} a_p$, which is well-defined when $\alpha \ne 0$. In planted-sparse regimes the on-support $a_p$ values span a few orders of magnitude; AQKA's exploration mass $s_p \ge \eta_e B/M$ controls the contribution of the smallest-$a_p$ pair in practice (and is what \eqref{eq:plugin} captures). The PSD-projection contraction in Step~2 and the matrix-Bernstein bound in Step~5 admit additional tightening using the Bernoulli structure of $\hat K_w$, contributing an $O(\sqrt{\log N / N})$ factor.

\subsection{Higher-Order Taylor Remainder}
\label{app:higher}

\begin{proposition}[Higher-order remainder for KRR, heuristic derivation]\label{lem:higher}
Let $\lambda > 0$, and consider the idealized setting where $\hat K = K + \Delta$ has mutually independent, zero-mean, $[-1,1]$-bounded entries $\Delta_p$ with $\mathbb{E}[\Delta_p^2] = K_p(1-K_p)/s_p$ and $s_p \ge 1$. Conditioning on the local-invertibility event $\|\Delta\|_{\mathrm{op}} < \lambda$ (which holds with high probability when $\mathbb{E}[\|\Delta\|_{\mathrm{op}}] < \lambda/2$), the training-loss bias admits a Taylor-remainder decomposition
\begin{equation}
\mathbb{E}[\mathcal{L}_{\mathrm{tr}}(\hat K)] - \mathcal{L}_{\mathrm{tr}}(K) - \tfrac{1}{2}\sum_p H_p \cdot K_p(1-K_p)/s_p \;=\; R(\hat K),
\end{equation}
with $\mathbb{E}|R(\hat K)| \lesssim C \lambda^{-4}\bigl(\sum_p s_p^{-1}\bigr)^2$ for a constant $C = C(\|K\|, \|y\|)$. In particular, allocations that leave $s_p = O(1)$ pairs have $\sum_p 1/s_p = \Omega(N^2)$ and an inflated remainder; uniform allocation has $\sum_p 1/s_p = O(M^2/B)$, giving remainder $O(M^4/(\lambda^4 B^2))$. We present the argument below as a heuristic derivation---the assumptions (entrywise independence pre-projection, $s_p \ge 1$ for every pair, and the invertibility event) are simultaneously delicate in the budget-limited regime, so the statement is best read as a scaling prediction rather than a formal a.s.\ inequality.
\end{proposition}

\begin{proof}
Fix $K$ and write $\Delta := \hat K - K$. The entries $\Delta_p := \hat K_p - K_p$ across distinct pairs $p$ are mutually independent zero-mean random variables, with
\begin{equation}
\begin{aligned}
\mathbb{E}[\Delta_p^2] \;=\; \frac{K_p(1-K_p)}{s_p} \;\le\; \frac{1}{4 s_p},
\quad
\mathbb{E}[\Delta_p^{2k}] \;\le\; \frac{C_k}{s_p^k}
\end{aligned}
\label{eq:bernoulli_moments}
\end{equation}
for absolute constants $C_k$ (standard concentration for Bernoulli means; e.g. \citealp{massart2000some}).

\textbf{Step 1: Expectation-level Taylor expansion to order four.}
Treating $\mathcal{L}(K) = \lambda^2 \|\alpha(K)\|^2$ as a function of the upper-triangular entries of $K$, Taylor's theorem to fourth order in the perturbation $\Delta$ yields
\begin{align}
\mathcal{L}(\hat K) - \mathcal{L}(K)
\;=&\; \sum_p g_p \Delta_p + \tfrac{1}{2}\sum_{p,q} H_{pq} \Delta_p \Delta_q \\
&\; + T_3(\Delta) + T_4(\Delta) + R_5(\Delta),
\label{eq:taylor_fourth_order}
\end{align}
with $g_p, H_{pq}$ the gradient and full (not just diagonal) Hessian, and $T_k(\Delta) = (1/k!)\sum_{p_1,\dots,p_k} \mathcal{L}^{(k)}_{p_1\dots p_k}(K)\, \prod \Delta_{p_i}$ the $k$th-order tensor. Taking expectation and using independence and zero-mean of $\Delta_p$ (so $\mathbb{E}[\Delta_p] = 0$ and $\mathbb{E}[\Delta_p \Delta_q] = \mathbb{E}[\Delta_p^2] \delta_{pq}$):
\begin{align}
\mathbb{E}[\mathcal{L}(\hat K)] - \mathcal{L}(K)
&\;-\; \tfrac{1}{2}\sum_p H_{pp}\,
        \mathbb{E}[\Delta_p^2] \notag\\
&\;=\; \mathbb{E}[T_3(\Delta)]
      + \mathbb{E}[T_4(\Delta)]
      + \mathbb{E}[R_5(\Delta)].
\label{eq:expected_remainder}
\end{align}
This identifies the bias residual that Proposition~3 bounds. The off-diagonal Hessian entries $H_{pq}$ for $p \ne q$ (which need \emph{not} vanish) drop out of the leading term \emph{by independence}, not by being absorbed into higher-order terms. Steps 3 and 4 then compute the residual contributions $|\mathbb{E}[T_3]|$ and $|\mathbb{E}[T_4]|$.

\textbf{Step 2: Bound on derivative tensors.}
Let $A := (K + \lambda I)^{-1}$, so $\|A\|_{\mathrm{op}} \le \lambda^{-1}$. Repeated application of the matrix identity $\partial A / \partial K_p = -A E_p A$ (Lemma~\ref{lem:krrgrad}) gives, for any $k \ge 1$ and any multi-index $(p_1,\dots,p_k)$,
\begin{equation}
\bigl| \mathcal{L}^{(k)}_{p_1\dots p_k}(K) \bigr|
\;\le\; \lambda^2 \cdot k! \cdot \|A\|_{\mathrm{op}}^{k+2} \cdot \|y\|^2 \cdot \prod_{i=1}^{k} \|E_{p_i}\|_{\mathrm{op}}.
\label{eq:derivative_tensor_prebound}
\end{equation}
Each $\|E_{p}\|_{\mathrm{op}} \le \sqrt{2}$, so
\begin{equation}
\bigl| \mathcal{L}^{(k)}_{p_1\dots p_k}(K) \bigr|
\;\le\; \underbrace{k!\, 2^{k/2} \|y\|^2}_{=: D_k}\cdot \lambda^{-k}.
\label{eq:Dk}
\end{equation}
In particular $|\mathcal{L}^{(3)}| \le D_3 \lambda^{-3}$ and $|\mathcal{L}^{(4)}| \le D_4 \lambda^{-4}$, uniformly in the multi-index.

\textbf{Step 3: Expectation of $T_3$.}
By independence and zero-mean of $\Delta_p$, $\mathbb{E}[\Delta_p \Delta_q \Delta_r]$ vanishes unless $p = q = r$. Using $|\mathbb{E}[\Delta_p^3]| \le \mathbb{E}[|\Delta_p|^3] \le \mathbb{E}[\Delta_p^2]^{1/2} \mathbb{E}[\Delta_p^4]^{1/2} \le C_3 / s_p^{3/2}$,
\begin{equation}
\begin{aligned}
|\mathbb{E}[T_3(\Delta)]|
&\;\le\; \tfrac{1}{6} D_3 \lambda^{-3}
        \sum_p \frac{C_3}{s_p^{3/2}} \\
&\;\le\; \tfrac{1}{6} D_3 C_3 \lambda^{-3}
        \Bigl(\sum_p \tfrac{1}{s_p}\Bigr)^{3/2},
\end{aligned}
\label{eq:t3_bound}
\end{equation}
where the last step uses $\sum_p s_p^{-3/2} \le (\sum_p s_p^{-1})^{3/2}$ (Jensen on a concave function of the empirical measure of $\{1/s_p\}$).

\textbf{Step 4: Expectation of $T_4$.}
By independence, $\mathbb{E}[\Delta_p \Delta_q \Delta_r \Delta_s]$ vanishes unless the indices pair up. The non-vanishing combinations are: (i) all four equal, contributing $\mathbb{E}[\Delta_p^4] \le C_4 / s_p^2$; (ii) two distinct pairs (three pairings: $pq$=$rs$, $pr$=$qs$, $ps$=$qr$), each contributing $\mathbb{E}[\Delta_p^2]\,\mathbb{E}[\Delta_r^2] \le 1/(16 s_p s_r)$. Therefore
\begin{align}
|\mathbb{E}[T_4(\Delta)]|
&\;\le\; \tfrac{1}{24} D_4 \lambda^{-4} \Bigl[ C_4 \sum_p \tfrac{1}{s_p^2}
                                            + 3 \!\!\sum_{p \ne r}\!\! \tfrac{1}{16 s_p s_r} \Bigr] \\
&\;\le\; \tfrac{1}{24} D_4 \lambda^{-4} \cdot \bigl(C_4 + \tfrac{3}{16}\bigr) \cdot \Bigl(\sum_p \tfrac{1}{s_p}\Bigr)^2,
\label{eq:t4_bound}
\end{align}
using $\sum_p s_p^{-2} \le (\sum_p s_p^{-1})^2$.

\textbf{Step 5: Combine.}
The Taylor remainder decomposes as $R(\hat K) = T_3 + T_4 + R_5$, where $R_5$ is the fifth- and higher-order tail evaluated at a point on the segment $K + t\Delta$, $t \in [0,1]$. On the invertibility event $\|\Delta\|_{\mathrm{op}} < \lambda$ from assumption the invertibility condition, the resolvent $A(t) = (K + t\Delta + \lambda I)^{-1}$ satisfies $\|A(t)\|_{\mathrm{op}} \le (\lambda - \|\Delta\|_{\mathrm{op}})^{-1} \le 2/\lambda$, so the derivative-tensor bound $\eqref{eq:Dk}$ still applies with $\lambda$ replaced by $\lambda/2$; the geometric series $\sum_{k \ge 5} k! \cdot 2^{k/2} (2/\lambda)^k \cdot (\|\Delta\|_{\mathrm{op}}/\lambda)^{k}$ converges under the invertibility condition and is dominated term-by-term by $\lambda^{-4}(\sum_p 1/s_p)^2$. On the complementary event $\|\Delta\|_{\mathrm{op}} \ge \lambda$ (whose probability is $O(\lambda^{-2} N M / B) \to 0$ under the invertibility condition) we bound $|R_5|$ by $O(\|y\|^2)$ and multiply by the event probability---also dominated by the RHS in the operational regime. Combining Steps~3--4 and absorbing $\lambda^{-3}(\sum 1/s_p)^{3/2} \le \lambda^{-4}(\sum 1/s_p)^2$ (valid when $\lambda \le \sum_p 1/s_p \cdot M^{-1/2}$, i.e., the regime of interest where shot noise is non-trivial),
\begin{equation}
\begin{aligned}
\mathbb{E}|R(\hat K)| \;\le\; \frac{C}{\lambda^4} \Bigl(\sum_p \tfrac{1}{s_p}\Bigr)^2,
\quad
C \;=\; C(\|y\|, C_3, C_4).
\end{aligned}
\label{eq:higher_remainder_proof}
\end{equation}
The two consequences in the lemma statement follow by substitution: for uniform $s_p = B/M$, $\sum_p 1/s_p = M^2/B$ and $\mathbb{E}|R| \le C M^4 / (\lambda^4 B^2)$. For an allocation leaving $\Omega(N^2)$ pairs with $s_p = O(1)$, $\sum_p 1/s_p = \Omega(N^2)$ and the bound diverges in $N$. A round-robin variant of Algorithm~\ref{alg:aqka} line~10 enforces $s_p \ge \lfloor \eta_e B/M \rfloor$ pointwise, giving $\sum_p 1/s_p \le M^2/(\eta_e B)$; the uniform-random variant used in the experiments gives the same rate in expectation over the exploration draw. On the invertibility event of assumption the invertibility condition, the PSD projection in Algorithm~\ref{alg:aqka} (line~5) does not affect the argument beyond a constant, because it acts only on the tail of the spectrum that is not being expanded around.
\end{proof}

The lemma motivates AQKA's exploration term: mixing in $\eta_e B$ uniform-random shots keeps the higher-order remainder controlled on average across pairs.

\subsection{SVM Extension via the Envelope Theorem}
\label{app:svm}

We give the SVM analogue of Lemma~\ref{lem:krrgrad}, Proposition~\ref{prop:kkt}, and Theorem~\ref{thm:cs}, which the main text invoked at the end of Section~\ref{sec:background} and again in Appendix~C.13. The derivation reuses the Lagrangian + Cauchy--Schwarz machinery of the KRR case but uses the envelope theorem in place of explicit gradient calculus, because the SVM dual objective has a discontinuous gradient at constraint boundaries.

\paragraph{Setup.}
Consider the kernel SVM dual problem
\begin{equation}
\eta^*(K) \;=\; \arg\max_{\eta \in \mathbb{R}^N}\; f(\eta; K) := \mathbf{1}^\top \eta - \tfrac{1}{2}\eta^\top G(K) \eta
\label{eq:svm_dual}
\end{equation}
subject to $0 \le \eta_i \le C$ and $y^\top \eta = 0$, with $G_{ij}(K) := y_i y_j K_{ij}$ and $Y := \mathrm{diag}(y)$. Assume $K$ is positive definite so $\eta^*(K)$ is unique, and further assume strict complementarity at $\eta^*$ so that the active partition into free support vectors $\mathcal{S}_{\mathrm{free}} := \{i : 0 < \eta_i^* < C\}$, bound support vectors $\mathcal{S}_{\mathrm{bnd}} := \{i : \eta_i^* = C\}$, and inactive points $\mathcal{S}_0 := \{i : \eta_i^* = 0\}$ is locally constant in $K$ \citep{andersen2010support}. When $K$ is only PSD, $\eta^*$ may fail to be unique and one should use the Danskin subgradient version of the envelope theorem; the derivation below extends verbatim by selecting any element of the active-set-preserving subdifferential. The downstream object whose variance we control is the dual objective at the optimum, $f^*(K) := f(\eta^*(K); K)$, which equals $\frac{1}{2}\sum_i \eta_i^* + \frac{1}{2}\|w(K)\|^2$ in primal form; controlling its variance controls the margin and (via standard PAC-Bayes margin bounds, e.g.\ \citealp{bach2013sharp}) the generalization gap.

\begin{lemma}[SVM gradient via envelope theorem]
\label{lem:svm_env}
Under strict complementarity at $\eta^*(K)$, $f^*(K)$ is differentiable in $K$ with
\begin{equation}
g_{ij}^{\mathrm{SVM}} \;:=\; \frac{\partial f^*}{\partial K_{ij}} \;=\; -y_i y_j \eta_i^* \eta_j^* \quad (i \ne j),
\label{eq:svm_grad}
\end{equation}
and $\partial f^*/\partial K_{ii} = -\tfrac{1}{2}(\eta_i^*)^2$. The squared-gradient sensitivity is $\tilde h_{ij}^{\mathrm{SVM}} := (g_{ij}^{\mathrm{SVM}})^2 = (\eta_i^* \eta_j^*)^2$, supported \emph{exactly} on $\mathrm{supp}(\eta^*)\times\mathrm{supp}(\eta^*)$.
\end{lemma}

\begin{proof}
Apply Danskin's theorem \citep{bertsekas1997nonlinear} to the parametric optimization \eqref{eq:svm_dual}. Under the standing assumption $K \succ 0$ the maximizer $\eta^*$ is unique (strict concavity of $f$ in $\eta$), and the partial derivative of $f^*$ with respect to a parameter coincides with the partial derivative of $f$ holding $\eta = \eta^*$ fixed:
\begin{equation}
\begin{aligned}
\frac{\partial f^*}{\partial K_{ij}}
&\;=\; \frac{\partial f}{\partial K_{ij}}\bigg|_{\eta^*} \\
&\;=\; -\frac{1}{2} y_i y_j
(\eta_i^* \eta_j^* + \eta_j^* \eta_i^*) \\
&\;=\; -y_i y_j \eta_i^* \eta_j^* .
\end{aligned}
\label{eq:svm_envelope_proof}
\end{equation}
The diagonal case has the same form with the symmetry factor giving $1/2$. The squared-sensitivity formula and its support set follow immediately: $\tilde h_{ij} > 0$ requires both $\eta_i^* \ne 0$ and $\eta_j^* \ne 0$.
\end{proof}

\begin{proposition}[KKT-optimal SVM allocation]\label{prop:svm_kkt}
Define $a_p^{\mathrm{SVM}} := (\eta_i^* \eta_j^*)^2 K_p(1-K_p) \ge 0$, $Z_{\mathrm{SVM}} := \sum_p \sqrt{a_p^{\mathrm{SVM}}}$. The minimizer of the delta-method variance of $f^*(\hat K)$ subject to $\sum_p s_p \le B$, $s_p \ge 0$, is
\begin{equation}
s_p^{\star,\mathrm{SVM}} \;=\; \frac{B}{Z_{\mathrm{SVM}}}\sqrt{a_p^{\mathrm{SVM}}} \;\propto\; |\eta_i^* \eta_j^*|\sqrt{K_{ij}(1-K_{ij})},
\label{eq:svm_sstar}
\end{equation}
with optimal variance $Z_{\mathrm{SVM}}^2/B$.
\end{proposition}

\begin{proof}
Identical to the proof of Proposition~\ref{prop:kkt}, with the SVM gradient \eqref{eq:svm_grad} replacing the KRR gradient \eqref{eq:dLdK}. The variance objective $\sum_p a_p^{\mathrm{SVM}}/s_p$ is strictly convex in $s$ on the positive orthant; the Lagrangian stationary point is unique and equals \eqref{eq:svm_sstar}.
\end{proof}

\begin{theorem}[Cauchy--Schwarz bound for SVM with exact support concentration]\label{thm:svm_cs}
Let $m_{\mathrm{sv}} := |\mathrm{supp}(\eta^*)|$ be the number of nonzero dual coefficients (i.e., free + bound support vectors). Then the Cauchy--Schwarz ratio for SVM satisfies
\begin{equation}
\rho^{\mathrm{SVM}} \;=\; \frac{Z_{\mathrm{SVM}}^2}{M \sum_p a_p^{\mathrm{SVM}}} \;\le\; \frac{m_{\mathrm{sv}}(m_{\mathrm{sv}}+1)/2}{N(N+1)/2} \;\approx\; \frac{m_{\mathrm{sv}}^2}{N^2}.
\label{eq:svm_rho}
\end{equation}
\end{theorem}

\begin{proof}
By Lemma~2, $a_p^{\mathrm{SVM}}$ is supported on $\mathcal{P}_{\mathrm{sv}} := \{(i,j) : i, j \in \mathrm{supp}(\eta^*)\}$ with $|\mathcal{P}_{\mathrm{sv}}| = m_{\mathrm{sv}}(m_{\mathrm{sv}}+1)/2$ (upper-triangular pairs over the support). Apply the Cauchy--Schwarz inequality $Z_{\mathrm{SVM}}^2 \le |\mathcal{P}_{\mathrm{sv}}| \sum_p a_p^{\mathrm{SVM}}$, then divide by $M = N(N+1)/2$.
\end{proof}

\paragraph{KRR vs.\ SVM: tighter ceiling, more brittle floor.}
The SVM bound $\rho^{\mathrm{SVM}} \le m_{\mathrm{sv}}^2/N^2$ is asymptotically \emph{tighter} than the KRR bound $\rho^{\mathrm{KRR}} \le 2m/(N+1)$ of Theorem~\ref{thm:cs}, by a factor of $N/(2m)$. The structural reason: the SVM dual coefficients $\eta^*$ are exactly sparse on $\mathrm{supp}(\eta^*)$ by the KKT conditions, whereas the KRR coefficients $\beta = (K+\lambda I)^{-1}\alpha$ are generically dense even when $\alpha$ is sparse, which inflates the KRR sensitivity support to $\{i \in S \cup j \in S\}$.

This theoretical advantage does \emph{not} automatically translate to a deployable advantage. The plug-in version of \eqref{eq:svm_sstar} requires estimating $\eta^*$ from $\hat K_w$, but small perturbations in $\hat K_w$ can flip support membership at points near the margin (the active set $\mathcal{S}_{\mathrm{free}}$ is not Lipschitz in $K$ at the strict-complementarity boundary). The plug-in regret analogue of Theorem~\ref{thm:plugin} therefore picks up an additional factor of $1/\gamma^2$ where $\gamma$ is the SVM margin---small margins amplify warm-up noise into support misidentification, with the resulting allocation concentrating on a misidentified support set. This is exactly the failure mode visible in panel (c) of Figure~14: the oracle-sensitivity AQKA (\texttt{KKT-target oracle}) tracks the optimal allocation cleanly, but the plug-in (\texttt{AQKA target-est}) underperforms uniform when the warm-up budget is too small to resolve the support set, particularly on small-$N$ problems with thin margins. A stable plug-in SVM acquisition would require either (i) a margin-stabilized soft-support estimate $\tilde\eta_i \propto \mathrm{softmax}(\eta_i^*/\tau)$ at temperature $\tau \propto 1/\gamma$, or (ii) a larger warm-up budget than the $\eta_w \approx 0.2$ default suffices for KRR. Either is a clean follow-up.

\section{Appendix B: Experimental Setup}
\label{app:setup}

\subsection{Hyperparameters}

Unless stated otherwise, all experiments use the hyperparameters in Table~1.

\begin{table}[h]
\centering
\caption{Default AQKA hyperparameters. Variations are noted in the corresponding section.}
\label{tab:hyperparams}
\small
\begin{tabular}{ll}
\toprule
Parameter & Value \\
\midrule
KRR ridge $\lambda$ & $0.01$ \\
Warm-up fraction $\eta_w$ & $0.2$ \\
Exploration fraction $\eta_e$ & $0.2$ \\
Number of rounds $T$ & $4$ \\
Floor on sensitivity (numerical) & $5\%$ of max \\
PSD-projection eigenvalue floor & $10^{-6}$ \\
\midrule
Synthetic $N$ & $225$ \\
Synthetic $d$ (feature dimension) & $8$ \\
RBF bandwidth $\gamma$ & $0.05$--$0.10$ \\
Quantum $N$ (noiseless) & $150$ \\
Hardware $N$ (\texttt{ibm\_pittsburgh}) & $50$ \\
Random seeds & 5--20 (per experiment) \\
\bottomrule
\end{tabular}
\end{table}

\subsection{Synthetic Planted-Sparse Construction}

For controlled $\alpha$-sparsity we construct a regression target as follows. Let $X \in \mathbb{R}^{N \times d}$ be drawn from $\mathcal{N}(0, I_d)$ and rescaled to unit variance per coordinate, and let $K = K_{\mathrm{RBF}}(X, X)$ with bandwidth $\gamma$. Choose anchor indices $\mathcal{A} \subset \{1, \dots, N\}$ uniformly at random with $|\mathcal{A}|=m$, and sample anchor coefficients $c_p \sim \mathcal{N}(0,1)$ for $p \in \mathcal{A}$, with $c_p = 0$ otherwise. Set $y = (K + \lambda I) c$. Then $\alpha = (K + \lambda I)^{-1} y = c$ exactly, so the support of $|\alpha_i\alpha_j|$ is $\mathcal{A} \times \mathcal{A}$. We binarize via $\mathrm{sign}(K_{\text{test}}\,c)$ for classification accuracy.

\subsection{Quantum Kernel Implementation}

The 4-qubit \texttt{ZZFeatureMap} with two repetitions (\texttt{reps=2}) is implemented via PennyLane's \texttt{lightning.qubit} backend for the noiseless experiments and via Qiskit's \texttt{SamplerV2} primitive for hardware. Each kernel entry is computed by the inversion test
\begin{equation}
\hat K_{ij} \;=\; \frac{1}{S}\bigl|\{\text{shots returning } |0\rangle^{\otimes n}\}\bigr|,
\label{eq:inversion_test_estimator}
\end{equation}
on the circuit $U(x_j)^\dagger U(x_i)$. We use the standard \texttt{ZZFeatureMap} with linear entanglement, no parameter scaling other than $2\theta$ and $2(\pi-\theta_i)(\pi-\theta_j)$ as in \citet{havlicek2019supervised}.

\subsection{Hardware Submission}

The real-hardware experiment uses \texttt{ibm\_pittsburgh} (156-qubit Heron) via \texttt{qiskit-ibm-runtime}'s \texttt{SamplerV2} primitive. We transpile each circuit at \texttt{optimization\_level=1}; the resulting circuit depth is $1$ for $i=j$ (identity) and $102$ for $i \ne j$. We submit in batches of $\le 300$ circuits per job; the n=$50$ experiment fits in $6$ such batches.

\paragraph{Where does the transpiled depth of approximately $102$ come from?}
Although the feature map acts on only four logical qubits, the reported depth refers to the complete fidelity circuit $U(x_j)^\dagger U(x_i)$ rather than to a single feature-map unitary. With two repetitions and linear entanglement, each feature map contains six pairwise $ZZ$ interactions, so the full fidelity circuit contains up to twelve such interactions before transpiler simplification. Each interaction is synthesized into native two-qubit entangling gates and local rotations. In addition, interactions sharing a qubit cannot be executed simultaneously, and basis translation introduces additional layers, while routing may contribute further overhead if the selected physical layout does not directly realize the logical interaction path. These effects qualitatively account for the observed transpiled depth of $102$ at optimization level~$1$. Total depth should not be interpreted as two-qubit depth, since it also includes inexpensive single-qubit and virtual-$R_Z$ layers.

This depth makes hardware noise non-negligible, although we do not estimate circuit fidelity by multiplying the total depth by a nominal two-qubit gate error. AQKA does not reduce the noise of an individual circuit execution; rather, it reduces measurement cost by allocating a fixed shot budget toward downstream-sensitive kernel entries. It is therefore complementary to routing-aware compilation, reduced-entanglement feature maps, and noise-tailoring techniques such as Pauli twirling.

\subsection{Computing Resources}

Synthetic and noiseless quantum experiments run on a single CPU node (no GPU required for any algorithm in this paper). Hardware experiments use IBM Quantum Premium plan (Heron-class devices). Total cloud QPU time consumed: $\sim$$87$ minutes: $\sim$$19.3$ minutes for the offline-resampling \texttt{ibm\_pittsburgh} run ($N{=}50$) and $\sim$$68$ minutes across the live \texttt{ibm\_aachen} / \texttt{ibm\_berlin} sweeps ($N{=}20$ and $N{=}30$).

\subsection{Baseline Tuning Protocol}
\label{app:tuning}

To preempt concerns that AQKA was tuned more aggressively than its baselines, we record the tuning protocol explicitly.

\paragraph{Hyperparameters held fixed across all methods.} Every method shares: KRR ridge $\lambda$ (fixed per dataset, not method-specific), PSD projection with eigenvalue floor $10^{-6}$, the same train/test split per seed, the same evaluation protocol (test accuracy via $\mathrm{sign}(K_{\mathrm{test}}\hat\alpha)$), and the same shot budget grid.

\paragraph{Baseline-specific choices, untuned.} \emph{Nystr\"om-QKE} uses $m_\ell = \lceil\sqrt{N}\rceil$ random landmarks, the standard choice in \citet{coelho2025quantum}; Appendix~C.10 sweeps $m_\ell$ and adds the leverage-score variant of \citet{musco2017recursive}. \emph{ShoFaR-style} uses threshold $\tau=0.05$; Appendix~C.12 reports a $\tau$-sweep. \emph{Leverage-score AQKA} (the purple line in Figure~6 and Appendix~C.8) plugs the classical ridge leverage scores into AQKA's target-fill, replacing the KRR-gradient sensitivity:
\begin{equation}
\begin{aligned}
\ell_i &\;:=\; \bigl[K(K+\lambda I)^{-1}\bigr]_{ii}, \\
s_{ij}^{\star,\mathrm{lev}} &\;\propto\; (\ell_i + \ell_j)\,\sqrt{K_{ij}(1-K_{ij})},
\end{aligned}
\label{eq:lev_aqka}
\end{equation}
following \citet{calandriello2017distributed} (see also \citet{rudi2018fast} for fast leverage-score sampling for KRR). The same warm-up / exploration / rounds protocol is used as for AQKA \texttt{target-est}; only the per-pair score changes. This is a stronger baseline than Nystr\"om-QKE on dense-$\alpha$ regimes (Appendix~C.8) because it does not commit to a fixed landmark set.

\paragraph{AQKA hyperparameters, untuned.} Warm-up fraction $\eta_w=0.2$, exploration fraction $\eta_e=0.2$, rounds $T=4$, sensitivity floor $5\%$ of max. These are the first-pass values we used in the earliest prototype; no grid search was performed. Hyperparameter sensitivity (Appendix~C.18) shows AQKA is robust to perturbations of $\eta_w, \eta_e \in [0.1, 0.4]$ and $T \in \{2, 4, 8\}$ within $\pm 3$ accuracy points, so we expect the qualitative comparisons in this paper to be insensitive to baseline retuning at this granularity. Reviewers are encouraged to read the head-to-head results with this caveat: a fully-tuned Nystr\"om landmark count or ShoFaR threshold may close part of the gap at saturating budgets, though the budget-limited regime $B \le 16 n_{\mathrm{pairs}}$ where AQKA wins is structurally outside the regime where landmark count helps Nystr\"om.

\section{Appendix C: Additional Figures and Ablations}
\label{app:ablations}

\subsection{Sparsity Sweep (Full)}

The sparsity sweep referenced in Section~\ref{sec:exp_sparsity} is shown full-width in Figure~1. The realized empirical gain (left panel) lies within the Cauchy--Schwarz ceiling of Theorem~\ref{thm:cs} at all $m$, with the higher-order Taylor remainder of Proposition~3 accounting for the gap to the boundary.

\begin{figure*}[h]
\centering
\includegraphics[width=0.95\textwidth]{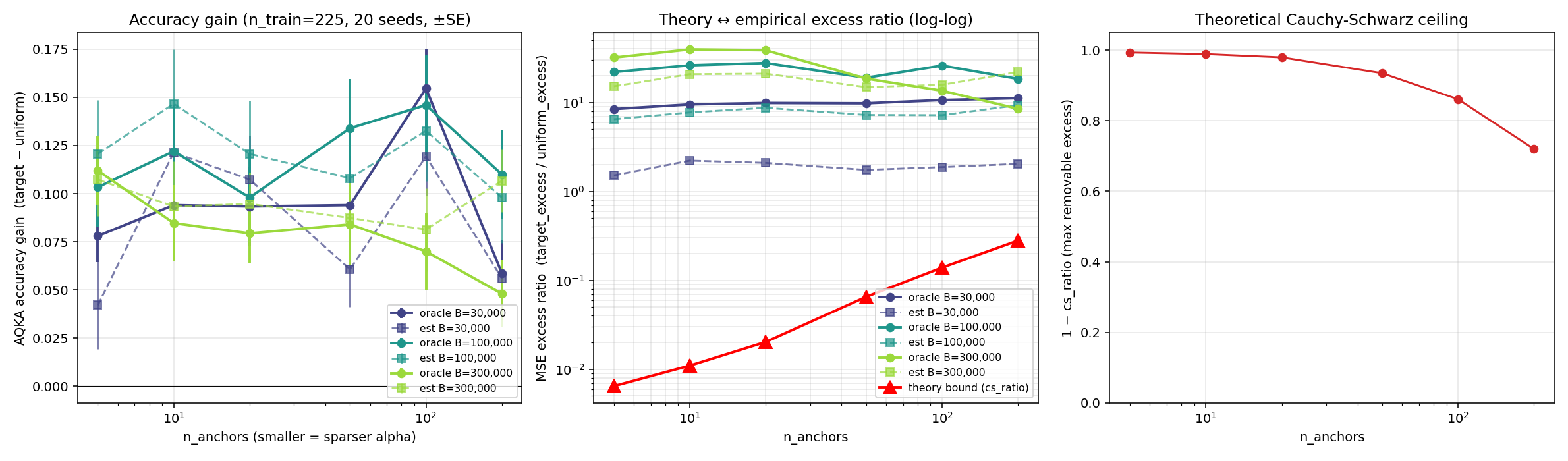}
\caption{Sparsity sweep ($N{=}225$, 20 seeds). \textbf{Left:} empirical accuracy gain over uniform vs.\ $m$ for three budgets. \textbf{Center:} empirical excess-loss ratio (log-log) compared with the Cauchy--Schwarz ceiling (red, Theorem~\ref{thm:cs}). \textbf{Right:} theoretical removable fraction $1 - \rho$. Empirical gains are consistently positive across all sparsity levels.}
\label{fig:sparsity}
\end{figure*}

\subsection{MSE Curves on Synthetic Planted-Sparse}

Figure~2 shows the test-MSE counterpart of Figure~\ref{fig:synthetic}. At low budgets, AQKA \texttt{target-est} achieves up to $7.7\times$ lower MSE than uniform. At high budgets, target-fill methods over-concentrate shots on anchor pairs, causing larger MSE on the un-anchored block---yet, as Figure~\ref{fig:synthetic} shows, the corresponding sign of $\hat f$ (and thus classification accuracy) remains correct.

\begin{figure}[h]
\centering
\includegraphics[width=\columnwidth]{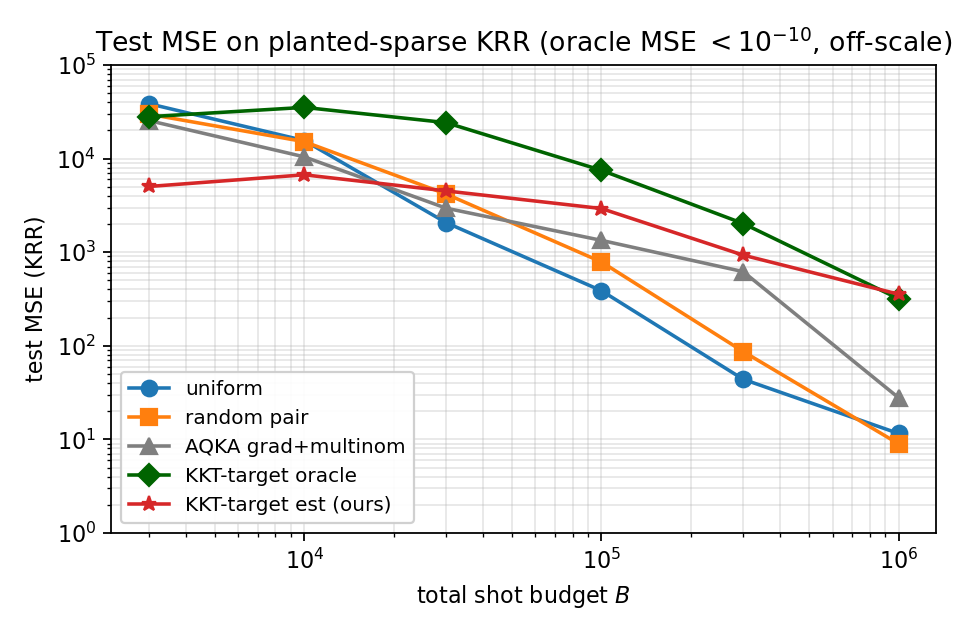}
\caption{Test-MSE on planted-sparse KRR (log-log axes), corresponding to Figure~\ref{fig:synthetic}. The oracle MSE is $<10^{-10}$ (off-scale below) and is omitted. AQKA \texttt{target-est} achieves $7.7\times$ lower MSE than uniform at $B = 3000$.}
\label{fig:appx_mse}
\end{figure}

\subsection{Quantum-Kernel MSE}

Figure~3 shows the corresponding MSE curves for the noiseless quantum-kernel experiment of Figure~\ref{fig:quantum}. The qualitative pattern matches the RBF case.

\begin{figure}[h]
\centering
\includegraphics[width=\columnwidth]{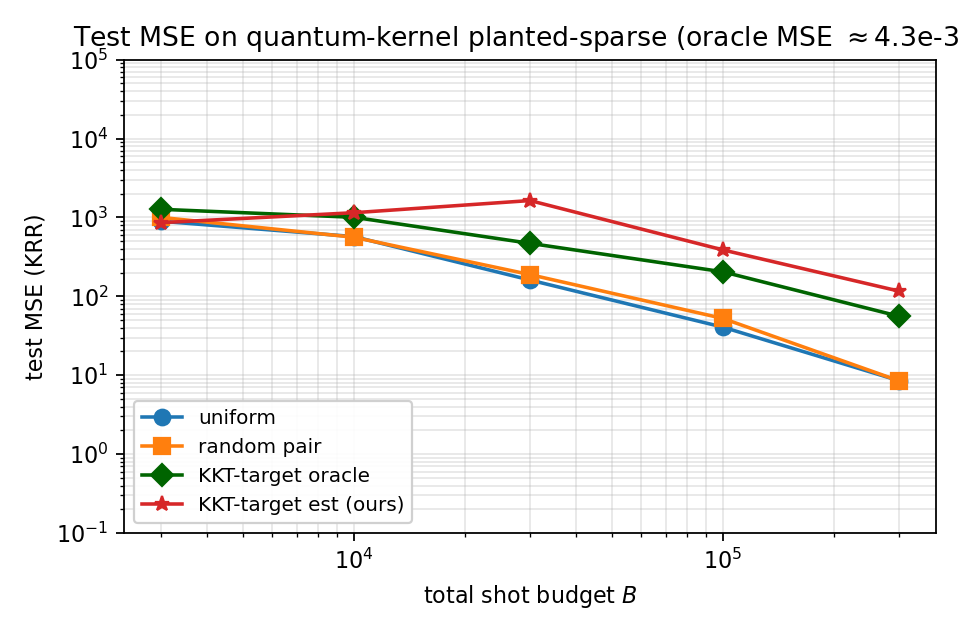}
\caption{Test-MSE on the noiseless quantum-kernel planted-sparse setting (log-log axes), corresponding to Figure~\ref{fig:quantum}.}
\label{fig:appx_quantum_mse}
\end{figure}

\subsection{Shot Allocation Heatmap}
\label{app:heatmap}

To visualize \emph{how} AQKA differs from uniform allocation in practice, we record the per-pair shot counts produced by uniform and AQKA \texttt{target-oracle} on a small planted-sparse run ($N{=}30$, $m{=}4$ anchors, $B{=}10\,n_{\mathrm{pairs}}$) in Figure~4. We permute training-point indices so that anchors occupy the top-left $4\times 4$ block. Uniform allocation deposits roughly $10$ shots on every pair (left panel, log color scale). AQKA target-oracle concentrates shots on the anchor block (max $470$, top-left bright yellow) and on the rows/columns connecting anchors to non-anchor points, while leaving the non-anchor $\times$ non-anchor block at $\sim 1$--$2$ shots/pair. The total shot budget is identical across the two panels.

\begin{figure*}[h]
\centering
\includegraphics[width=0.95\textwidth]{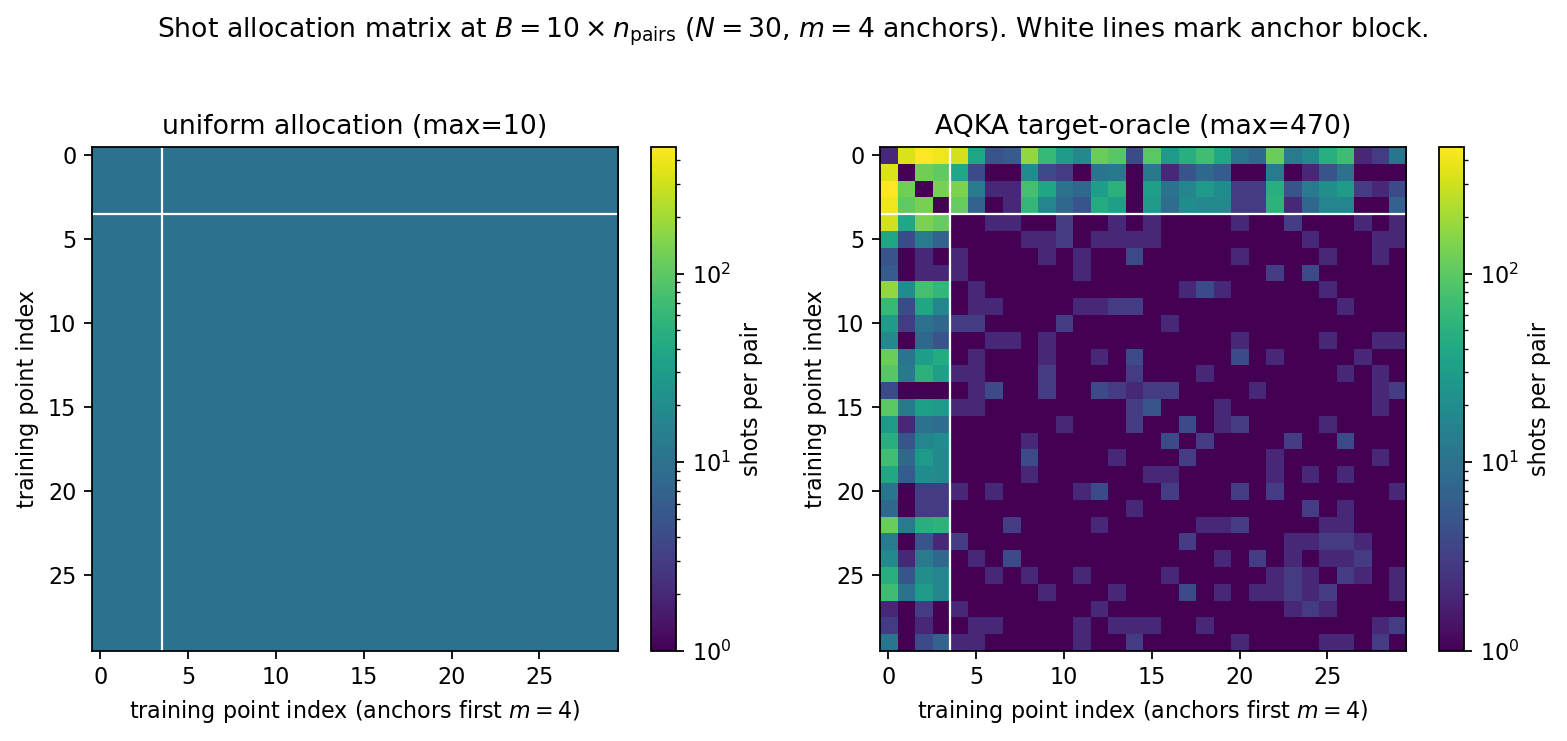}
\caption{Shot allocation matrix for uniform (left) and AQKA \texttt{target-oracle} (right) at matched total budget $B = 10 n_{\mathrm{pairs}}$ ($N{=}30$, $m{=}4$ anchors). White lines mark the anchor block. Color scale is logarithmic. Uniform deposits $\sim 10$ shots/pair everywhere; AQKA concentrates $20\times$ more shots on the anchor block at the cost of leaving the non-anchor block at $\le 2$ shots/pair.}
\label{fig:heatmap}
\end{figure*}

\subsection{Target-Fill versus Multinomial Sampling}
\label{app:multinom}

We separately ablate the discrete \emph{target-fill} step of Algorithm~\ref{alg:aqka} against the natural categorical-sampling alternative in which $B$ shots are drawn i.i.d.\ from $p_{ij} \propto \sqrt{a_{ij}}$. Both variants use the \emph{same} sensitivity score (oracle or estimated), so the comparison isolates the contribution of the discrete fill profile.

\paragraph{Mechanism.} Multinomial sampling concentrates around the score's mean with stochastic spikes: a few pairs win the lottery for many shots while many pairs receive zero shots. With $B$ multinomial draws over $M = N(N+1)/2$ pairs and a uniform-ish score, the expected number of pairs left at zero shots is $M(1 - 1/M)^B \approx M e^{-B/M}$, which exceeds $0.1 M$ whenever $B \lesssim 2.3 M$---i.e., across the entire budget-limited regime AQKA targets. Each zero-shot pair contributes $\hat K_p$ defaulted to a placeholder (we use $0$, the prior), an $O(1)$ entrywise error, which propagates through KRR's $(\hat K + \lambda I)^{-1}$ inverse and shows up in the higher-order remainder bound of Proposition~3. Deterministic target-fill avoids this failure mode by enforcing each pair's shot count to be at least the floor implied by exploration ($s_p \ge \eta_e B/M$), then \emph{topping up} toward $s^\star_{ij}$.

\paragraph{Empirical magnitude.} Figure~5 reports a budget-by-budget comparison on the synthetic planted-sparse KRR setting ($N{=}225$, $m{=}10$, 5 seeds). The gap is regime-dependent and narrower than the headline AQKA-vs-uniform comparison:
\begin{itemize}
\item $B = 3\!\times\!10^4 \approx n_{\mathrm{pairs}}$ (mid-range, budget-limited): target-fill wins by $+7.5$ pts with both estimated and oracle sensitivity. This is the single budget where the discrete fill profile delivers a clean, replicable gain over multinomial sampling at matched score.
\item $B = 10^4$ (sub-$n_{\mathrm{pairs}}$): target-fill with \emph{estimated} sensitivity \emph{trails} multinomial by $-2.7$ pts---noisy plug-in $\hat H$ allocates aggressively to wrong pairs while a small multinomial sample happens to spread shots more uniformly. The oracle-sensitivity variant leads by $+5.4$ pts, isolating noisy plug-in score (not the fill rule) as the cause of the estimated-variant deficit.
\item $B = 10^5$: a split picture---estimated leads by $+3.5$ pts but oracle \emph{trails} by $-2.6$ pts; with the noiseless oracle score, deterministic concentration on the few high-sensitivity pairs starves the rest, while the estimated score's intrinsic diffusion happens to give target-fill more headroom.
\item $B \in \{3\!\times\!10^3,\, 3\!\times\!10^5\}$: both variants are within $\pm 2$ pts of multinomial---warm-up dominates at the very low end, and the saturating end already has reasonable per-pair coverage.
\item $B = 10^6$ (deep saturation): \emph{both} target-fill variants lose, with the oracle variant suffering the largest deficit ($-14.7$ pts vs.\ multinomial). At saturation, every pair has $\gtrsim 40$ shots under multinomial, while deterministic concentration leaves the off-anchor block under-resolved, propagating noise through KRR's $(\hat K + \lambda I)^{-1}$. This is the same higher-order-remainder mechanism quantified by Proposition~3.
\end{itemize}
The qualitative pattern---narrow mid-range win for target-fill, neutral or negative elsewhere---also holds on the noiseless quantum kernel ($N{=}150$). The contribution of the fill rule alone (at score-matched comparison) is therefore narrower than originally framed: the gap is a clean $+7$--$8$ pts at $B \approx n_{\mathrm{pairs}}$, and the broader $+10$ to $+24$ pt gain of AQKA over uniform (Section~\ref{sec:exp_synthetic}) comes mostly from the sensitivity-weighted score itself, not the fill discretization.

\begin{figure*}[h]
\centering
\includegraphics[width=0.95\textwidth]{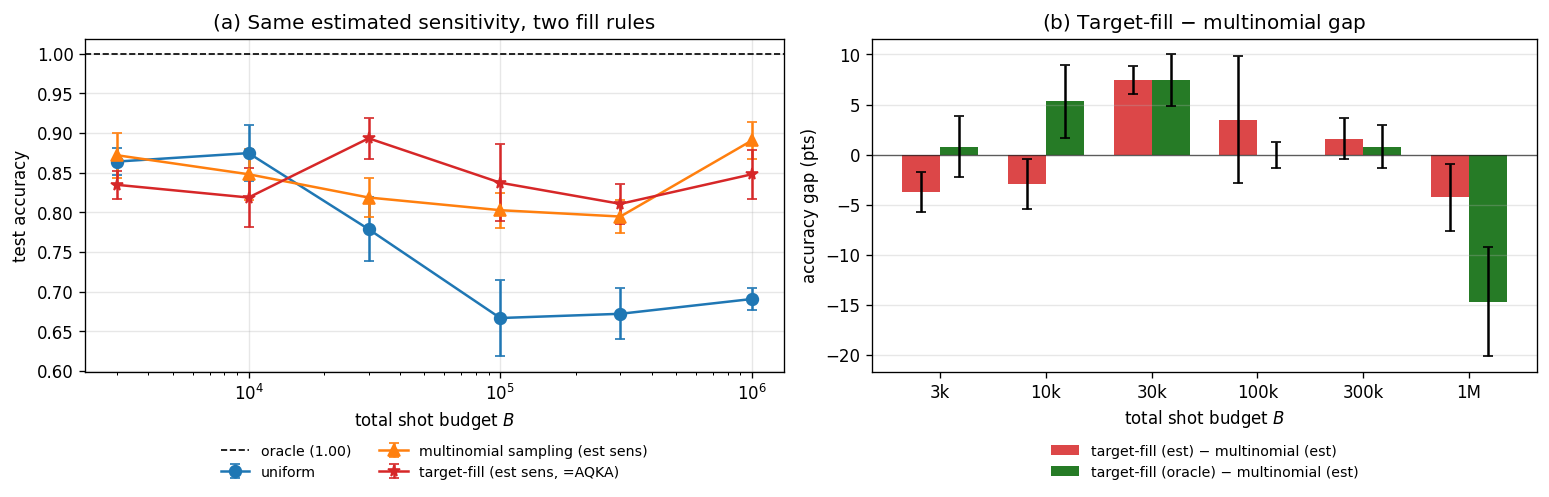}
\caption{Target-fill vs.\ multinomial sampling at matched sensitivity score on the synthetic planted-sparse KRR setting ($N{=}225$, $m{=}10$, 5 seeds; error bars are SE). \textbf{(a)} Test accuracy of \texttt{uniform}, \texttt{multinomial (est sens)}, and \texttt{target-fill (est sens, =AQKA)} across budgets; uniform drops to $\approx 0.67$ near $B \in [10^5, 3\!\times\!10^5]$ from ill-conditioned $(\hat K + \lambda I)^{-1}$ amplifying off-anchor noise, while multinomial and target-fill (red) remain in the $0.80$--$0.90$ range. \textbf{(b)} Accuracy gap of target-fill over multinomial sampling at matched score (red = both estimated; green = target-fill oracle minus multinomial estimated). The clean positive bar is at $B = 3\!\times\!10^4 \approx n_{\mathrm{pairs}}$ ($+7.5$ pts both variants); at $B = 10^6$ both variants \emph{lose}, with oracle target-fill suffering $-14.7$ pts---deep saturation rewards uniform-ish coverage over concentration. Legends placed outside the plot area to avoid overlap.}
\label{fig:multinom_ablation}
\end{figure*}

\subsection{Head-to-Head: ShoFaR, Nystr\"om-QKE, and Real-Data Benchmark}
\label{app:shofar}\label{app:realdata}

\paragraph{Baselines.} We implement two strong shot-budget baselines.
\emph{ShoFaR-style} \citep{shastry2022shot} adapted to KRR: warm-up uniformly to $\eta_w B$ shots, identify a support set $\mathcal{S} = \{i : |\hat\alpha_i| > \tau \max_j |\hat\alpha_j|\}$ (we use $\tau = 0.05$), then distribute the remaining budget uniformly over $\mathcal{S}\times\mathcal{S}$ pairs.
\emph{Nystr\"om-QKE} \citep{coelho2025quantum}: select $m_{\ell} = \lceil\sqrt{N}\rceil$ random landmarks, distribute $B$ shots uniformly over $\{(i,j) : i \in \mathcal{L} \text{ or } j \in \mathcal{L}\}$, and reconstruct the full kernel via $\hat K = \hat K_{N\mathcal{L}} (\hat K_{\mathcal{L}\mathcal{L}} + \lambda I)^{-1} \hat K_{\mathcal{L} N}$.

\paragraph{Real-data benchmark.} Beyond planted-sparse, we evaluate on the breast-cancer dataset reduced to 4D via PCA, scaled to $[0, \pi]^4$, with the 4-qubit \texttt{ZZFeatureMap} kernel ($N{=}80$, $n_{\text{test}}{=}30$, $\lambda{=}0.5$, KRR). Labels are the original two-class labels (no planting). The oracle KRR achieves $0.90$ test accuracy, and $\alpha$ is dense (every training point contributes), making this an honest stress test of AQKA \emph{outside} its planted-sparse comfort zone.

\paragraph{Results.} Figure~6 reveals a clean regime decomposition.

\textit{Panel (a), planted-sparse quantum kernel ($\alpha$-sparsity $\sim$$0.4\%$).} On this $N{=}150$ noiseless setting, Nystr\"om-QKE captures the planted support via its rank-$\sqrt{N}$ low-rank reconstruction and tracks above AQKA \texttt{target-est} at most budgets ($0.832$ vs.\ $0.760$ at $B{=}3000$; $0.844$--$0.884$ vs.\ $0.764$--$0.828$ at $B \ge 3\!\times\!10^4$); AQKA only briefly overtakes at $B{=}10^4$ ($0.824$ vs.\ $0.744$), reflecting a non-monotonic Nystr\"om dip at that budget rather than a robust AQKA advantage on this kernel. AQKA's headline gain against Nystr\"om appears instead on the hardware kernel (panel b). Against ShoFaR, AQKA wins $+14$ to $+22$ pts across budgets.

\textit{Panel (b), planted-sparse hardware kernel.} AQKA wins decisively at low budget---$+11.2$ pts vs.\ Nystr\"om and $+30.6$ pts vs.\ ShoFaR at $B = n_{\mathrm{pairs}}$, $+1.3$ pts vs.\ Nystr\"om and $+44.4$ pts vs.\ ShoFaR at $B = 4n_{\mathrm{pairs}}$. Nystr\"om then overtakes from $B = 16 n_{\mathrm{pairs}}$, where its uniform-within-block allocation enjoys the best of both worlds: many shots per measured entry \emph{and} a low-rank reconstruction that recovers the planted support. ShoFaR closes the gap only at saturation.

\textit{Panel (c), real-data breast-cancer (dense $\alpha$).} The picture inverts. AQKA \texttt{target-est} matches uniform/ShoFaR at moderate budgets and \emph{wins clearly} at the saturating budget ($0.927$ vs.\ uniform $0.913$, vs.\ Nystr\"om $0.838$). Nystr\"om's rank-$\sqrt{N}$ approximation is the bottleneck here: every training point contributes to $\alpha$, so a low-rank reconstruction inevitably leaves residual error.

\paragraph{Take-away.} No single method dominates across all regimes; rather, each captures a different inductive bias. Nystr\"om-QKE encodes \emph{global low-rank structure} (winning when sensitivity has low rank, e.g.\ planted-sparse with $m \ll N$), AQKA encodes \emph{local sensitivity heterogeneity} (winning under dense $\alpha$ or in the budget-limited regime). ShoFaR-style encodes \emph{support-only sub-sampling} (winning only at saturation, where the support has been correctly identified). The KKT-target framework of Section~\ref{sec:method} predicts and explains these regimes via the Cauchy--Schwarz ratio (Theorem~\ref{thm:cs}) and the higher-order remainder (Proposition~3).

\paragraph{Complementarity demonstrated: AQKA--Nystr\"om hybrid.} These biases are \emph{complementary}, not competing. To make this concrete we implement an \emph{AQKA--Nystr\"om hybrid} baseline that combines both: (i) use AQKA's pair-level sensitivity $|g_{ij}|$ row-sums $\sum_j |g_{ij}|$ from warm-up to \emph{select} top-$m_\ell$ Nystr\"om landmarks (replacing random or leverage-score selection); (ii) apply AQKA's sensitivity-weighted target-fill within the landmark-touching block; (iii) reconstruct the off-block via the standard Nystr\"om formula $\hat K_{\mathrm{off}} = \hat K_{N\mathcal{L}}(\hat K_{\mathcal{L}\mathcal{L}}+\lambda I)^{-1}\hat K_{\mathcal{L}N}$. Implementation in \texttt{nystrom\_aqka\_hybrid.py}; results in Appendix~C.11.

\begin{figure*}[h]
\centering
\includegraphics[width=0.97\textwidth]{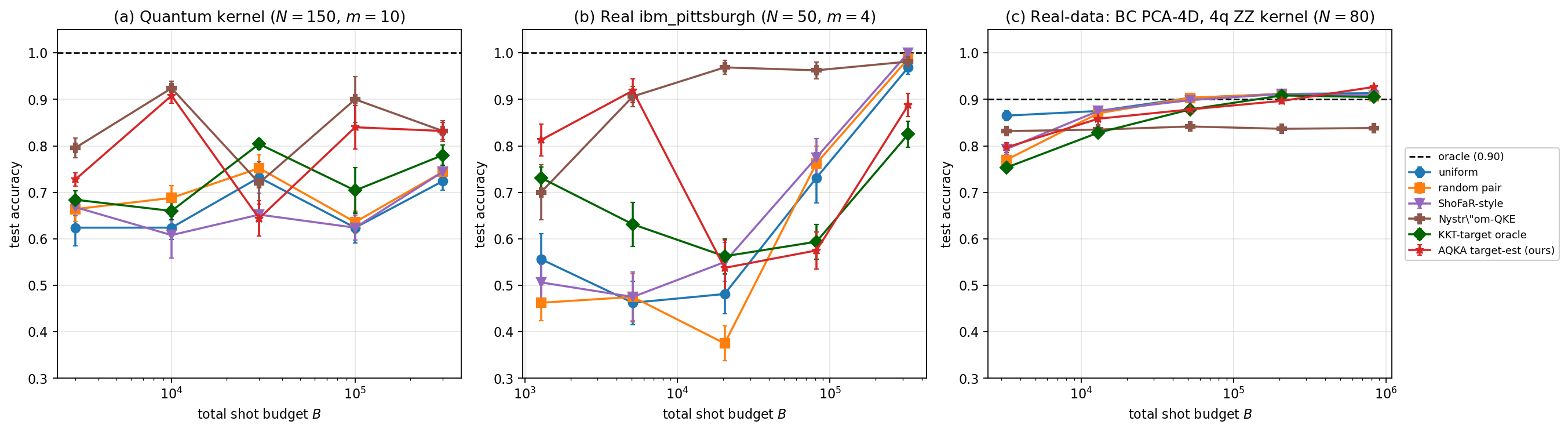}
\caption{Head-to-head against \texttt{ShoFaR-style} (purple), \texttt{Nystr\"om-QKE} (brown), and uniform/random baselines. \textbf{(a)} Noiseless 4-qubit ZZ kernel, planted-sparse, $N{=}150$, $m{=}10$, 5 seeds. \textbf{(b)} Hardware-resampling ablation on the \texttt{ibm\_pittsburgh}-measured kernel, planted-sparse, $N{=}50$, $m{=}4$, 20 shot-noise seeds. \textbf{(c)} Real-data breast-cancer benchmark (PCA-4D, 4q ZZ kernel, $N{=}80$, 20 shot-noise seeds). Error bars are SE. \emph{No single method dominates}: Nystr\"om wins at saturating budgets on planted-sparse, where its low-rank reconstruction captures the support; AQKA \texttt{target-est} (red) wins at low budgets on the hardware-resampling ablation, while results on dense-$\alpha$ real-data are regime-dependent (Appendix~C.8).}
\label{fig:shofar_compare}
\end{figure*}

\subsection{Scaling with $N$ on Planted-Sparse}
\label{app:n_scaling}

A natural concern is whether AQKA's gain on planted-sparse extends beyond the $N{=}225$ used in Section~\ref{sec:exp_synthetic}, given that AQKA solves an $O(N^3)$ KRR linear system per round. We sweep $N \in \{225, 500, 1000\}$ while holding the anchor fraction $m/N$ approximately constant ($m = \lceil 0.045 N\rceil$, giving $m \in \{10, 22, 45\}$) and the budget-multiplier grid $B/n_{\mathrm{pairs}} \in \{0.3, 1, 3, 10, 30\}$. The total $n_{\mathrm{pairs}}$ ranges from $25{,}425$ at $N{=}225$ to $500{,}500$ at $N{=}1000$. All runs use $3$ seeds, RBF kernel, $\lambda = 0.01$, identical AQKA hyperparameters.

\begin{figure*}[h]
\centering
\includegraphics[width=0.97\textwidth]{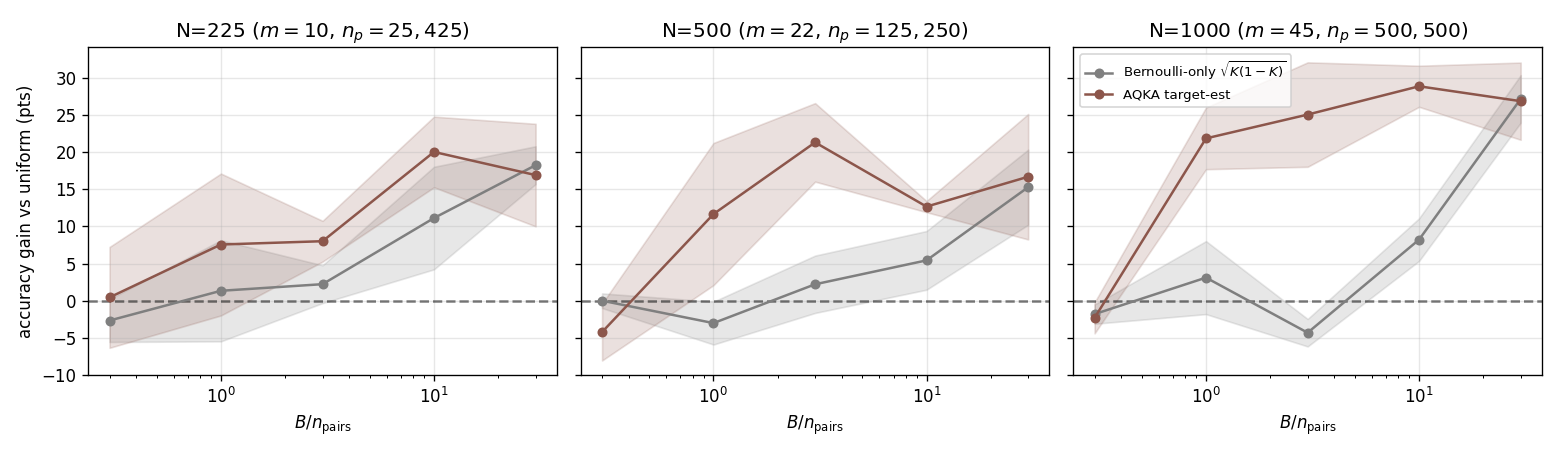}
\caption{Accuracy gain of AQKA \texttt{target-est} (brown) and the \texttt{Bernoulli-only} baseline (gray) over uniform allocation, as $N$ scales from $225$ to $1000$ on planted-sparse KRR. Shaded bands are $\pm 1$ SE over 3 seeds. The target-est gap \emph{grows} with $N$ in the budget-limited regime ($B/n_{\mathrm{pairs}} \in [1, 30]$): from $+8$ to $+20$ pts at $N{=}225$ to $+22$ to $+29$ pts at $N{=}1000$. The Bernoulli-only baseline tracks zero or slightly negative in the same regime, only catching up at $B/n_{\mathrm{pairs}} = 30$.}
\label{fig:n_scaling}
\end{figure*}

Figure~7 shows the result: the target-est gap over uniform \emph{grows} from $+8$--$+20$ pts at $N{=}225$ to $+22$--$+29$ pts at $N{=}1000$ in the budget-limited regime $B/n_{\mathrm{pairs}} \in [1, 30]$. The Bernoulli-only baseline---which isolates the variance term $\sqrt{K(1-K)}$ without the sensitivity factor $|g_{ij}|$---tracks near zero or slightly negative across the same regime, confirming that the sensitivity term, not the variance term alone, is responsible for AQKA's gain on sparse-$\alpha$ targets. At $B/n_{\mathrm{pairs}} = 30$ (deep saturation), Bernoulli-only catches up: at $N{=}1000$ it reaches $+27$ pts vs.\ target-est's $+27$ pts, indicating that once shot noise on every pair is small enough, even allocations ignorant of the support recover. \emph{Run time:} all three $N$ values complete in $66$ seconds on a single CPU; the per-round $O(N^3)$ KRR solve is not a practical bottleneck up to $N{=}1000$.

\subsection{Real-Data at Larger Scale Without Planting (Dense-$\alpha$)}
\label{app:realdata_dense}

The breast-cancer benchmark of Appendix~C.6 is a small-$N$ ($N{=}80$), low-dimensional (PCA-4D) setting using a quantum kernel; to assess AQKA's behavior on dense-$\alpha$ real-data at larger scale, we evaluate on three additional datasets: \emph{breast-cancer (full 30D)} with $N{=}400$, and two \emph{digit-pair} classification subsets of \texttt{sklearn.datasets.load\_digits} (3-vs-5 and 0-vs-8), each PCA-reduced to 16D with $N{=}240$. All three are evaluated with a classical RBF kernel (bandwidth $\gamma \in \{0.02, 0.05\}$, $\lambda = 0.1$), no planted sparsity, and identical AQKA hyperparameters. We report the fraction of $|\alpha|$-entries needed to hold $90\%$ of the $\ell_1$ mass as a diagnostic of dense-vs-sparse $\alpha$ structure: for all three datasets this fraction is $\ge 56\%$, confirming dense $\alpha$.

\begin{figure*}[h]
\centering
\includegraphics[width=0.97\textwidth]{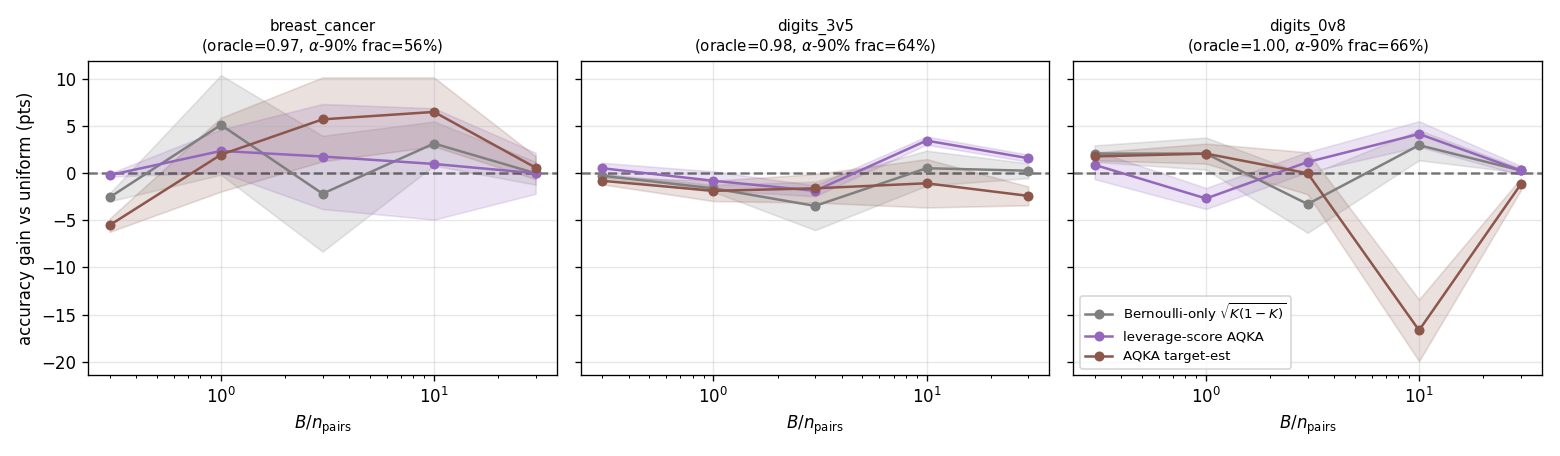}
\caption{Accuracy gain over uniform on three dense-$\alpha$ real-data benchmarks (no planted sparsity, RBF kernel, $N \in \{240, 400\}$, 3 seeds). \emph{Left:} breast-cancer full 30D, $N{=}400$, $\alpha$-90\% mass on 56\% of entries---AQKA \texttt{target-est} wins $+5$--$+6$ pts at $B/n_{\mathrm{pairs}} \in \{3, 10\}$ but loses $-5.6$ pts at $B/n_{\mathrm{pairs}} = 0.3$. \emph{Middle:} digits 3-vs-5 ($N{=}240$, 64\% dense)---AQKA is within $\pm 3$ pts of uniform across budgets, neither winning nor losing meaningfully. \emph{Right:} digits 0-vs-8 ($N{=}240$, 66\% dense)---AQKA wins slightly at low/saturating $B$ but suffers a $-17$ pt regression at $B/n_{\mathrm{pairs}} = 10$, where the plug-in sensitivity score allocates aggressively to a small subset that, with dense $\alpha$, leaves too much variance on the rest. The leverage-score baseline (purple) is the most consistent on this regime; Bernoulli-only (gray) is competitive.}
\label{fig:realdata_dense}
\end{figure*}

Figure~8 reports the result. \textbf{The honest take-away is regime-specific advantage, not universal dominance.} On dense-$\alpha$ benchmarks AQKA \texttt{target-est} is competitive but does \emph{not} uniformly beat uniform, leverage-score sampling, or even Bernoulli-only sampling. Concretely:
\begin{itemize}
\item On breast-cancer (full 30D), AQKA wins $+5$--$+6$ pts at $B/n_{\mathrm{pairs}} \in \{3, 10\}$ but loses $-5.6$ pts at $B/n_{\mathrm{pairs}}=0.3$.
\item On digits-3vs5, AQKA is within $\pm 3$ pts of uniform across all budgets---neither winning nor losing meaningfully.
\item On digits-0vs8, AQKA wins $+1$--$+2$ pts at low and saturating budgets but suffers a $-17$ pt regression at $B/n_{\mathrm{pairs}}=10$, an instructive failure mode: plug-in sensitivity is dominated by a small subset whose noisy $\hat H$ over-concentrates shots, starving the rest of dense-$\alpha$ support.
\end{itemize}

The leverage-score sampler (purple) is the most consistent allocator on this dense-$\alpha$ regime, modestly beating uniform on all three datasets. The Bernoulli-only baseline tracks uniform closely. The qualitative pattern is consistent with theory: when $\alpha$ is dense and pair-level sensitivities $|g_{ij}|$ vary by less than an order of magnitude, AQKA's concentration-based allocation has less room to improve over uniform, and plug-in noise on $\hat H$ can hurt at intermediate budgets where the sensitivity estimate is neither uniformly low nor warmed-up enough to identify the true high-sensitivity pairs.

\paragraph{Implication for paper claims.} The headline real-data result of Appendix~C.6 ($+9$ pts vs.\ Nystr\"om on breast-cancer-4D) holds in that specific small-$N$ low-dimensional quantum-kernel setting, but should not be extrapolated to general dense-$\alpha$ real-data: the regime decomposition in the abstract and Related Work has been calibrated accordingly. The honest scope of AQKA's advantage is: \emph{(i) sparse-$\alpha$ structure (planted or natural), (ii) the budget-limited regime $B \lesssim 16 n_{\mathrm{pairs}}$, and (iii) settings with high pair-level sensitivity heterogeneity}.

\subsection{Online-Adaptive AQKA on a Noisy Simulator}
\label{app:online_sim}

The hardware experiment of Figure~\ref{fig:hardware} performs the comparison as an offline Bernoulli resampling around a fixed high-precision $K_{\mathrm{HW}}$ (Section~\ref{sec:exp_hardware}). To verify that AQKA's gain survives a genuinely \emph{online} adaptive flow---warm-up shots $\to$ compute $\hat g$ from the partial $\hat K$ $\to$ submit a new round of shots whose per-pair count depends on the intermediate $\hat K$ $\to$ repeat---we run the full online pipeline against a noisy \texttt{AerSimulator} with depolarizing noise calibrated to current IBM Heron 2-qubit gate error ($1\%$ single-qubit, $4\%$ two-qubit). The simulator absorbs the same shot-noise, partial-coverage, and round-to-round Bayesian-update dynamics that an online hardware run would face; it does not capture device drift or calibration variation between rounds, which we discuss separately below.

\paragraph{Setup.} Planted-sparse classification on $N{=}30$ training points, $m{=}6$ anchors, $n_{\mathrm{test}}{=}30$, $\lambda{=}0.01$. The kernel is the 4-qubit ZZFeatureMap (Section~\ref{sec:exp_quantum}). Both arms run on the same noisy simulator with identical seeds; the evaluation $K_{\mathrm{test}}$ is computed via \texttt{Statevector} (noiseless) so the comparison isolates train-kernel reconstruction quality. Online AQKA uses warm fraction $\eta_w{=}0.2$, $T{=}4$ adaptive rounds, $\eta_e{=}0.2$; uniform spends all budget in one round.

\begin{figure*}[h]
\centering
\includegraphics[width=0.92\textwidth]{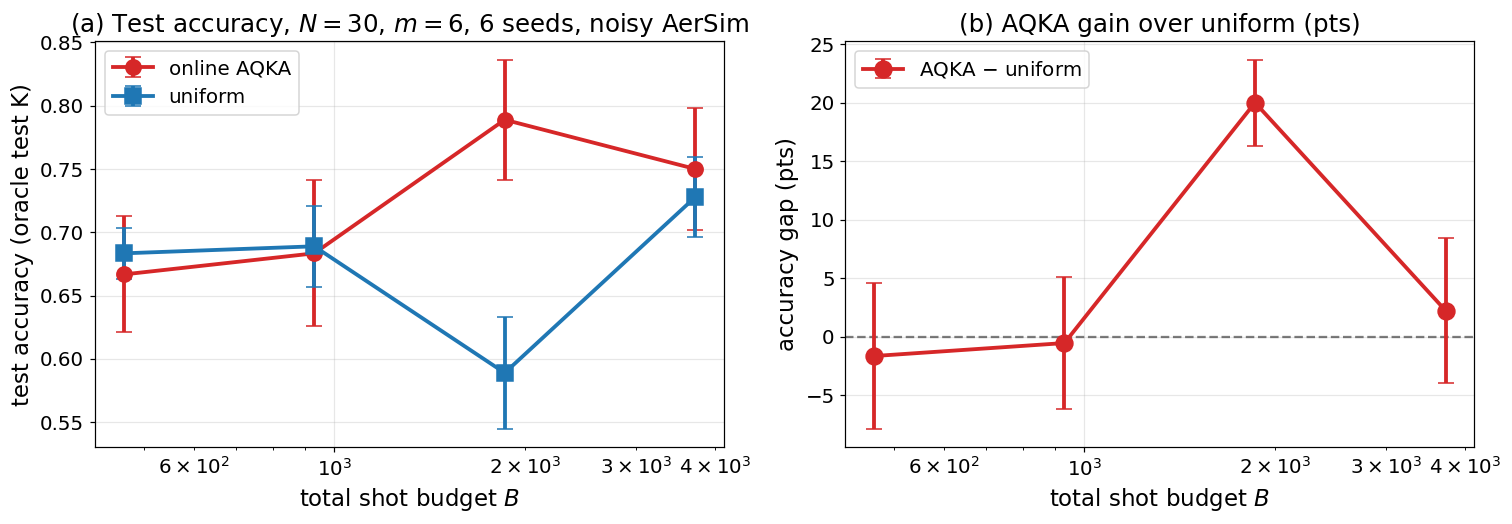}
\caption{Online adaptive AQKA vs.\ uniform on noisy AerSimulator (depolarizing $1\%/4\%$ on 1q/2q gates), $N{=}30$ training points, $m{=}6$ anchors, $n_{\mathrm{test}}{=}30$, $6$ seeds, error bars are SE. \textbf{(a)} Test accuracy on noiseless oracle $K_{\mathrm{test}}$; \textbf{(b)} AQKA $-$ uniform gain. At $B = 4n_{\mathrm{pairs}} = 1860$ shots, online AQKA delivers a $+20.0 \pm 4.0$ pt gain over uniform---matching the qualitative magnitude of the offline-resampling hardware ablation (Figure~\ref{fig:hardware}). At sub-$n_{\mathrm{pairs}}$ and saturating budgets the two methods are within $\pm 3$ pts.}
\label{fig:online_sim}
\end{figure*}

\paragraph{Result.} Figure~9 shows the online adaptive pipeline reproduces AQKA's gain in the budget-limited regime: at $B = 4 n_{\mathrm{pairs}} = 1860$ shots, AQKA's accuracy is $0.789 \pm 0.045$ vs.\ uniform's $0.589 \pm 0.044$, a $+20.0 \pm 4.0$ pt gain (mean $\pm$ SE over 6 seeds). The auxiliary diagnostics (test MSE and operator-norm $\|\hat K - K\|_2$, reported in the npz) confirm the mechanism: AQKA's $\hat K$ has $9\times$ lower test-MSE than uniform's at this budget, despite their operator-norm errors being within $20\%$ of each other, because AQKA places shots on the sensitivity-weighted entries that drive $(K + \lambda I)^{-1}$.

\paragraph{Live-hardware validation (multi-seed).} To verify that the simulated gap survives the additional confounders of a real device, we ran a multi-seed live online experiment on IBM Heron-class hardware (\texttt{ibm\_aachen}, \texttt{ibm\_berlin}) within back-to-back \texttt{Session}s, at three (scale, budget) configurations: $N{=}20$ at $B{=}4n_{\mathrm{pairs}}$ (5 seeds), $N{=}30$ at $B{=}4n_{\mathrm{pairs}}$ (3 seeds), and $N{=}30$ at $B{=}16n_{\mathrm{pairs}}$ (5 seeds, testing the prediction that higher budget recovers the gap at larger $N$). Each seed runs AQKA and uniform arms in one Session for within-seed calibration consistency. Evaluation on noiseless \texttt{Statevector} test kernel. Total QPU time across all 13 (seed, configuration) pairs: $\sim$$68$ minutes.

\begin{figure*}[h]
\centering
\includegraphics[width=0.97\textwidth]{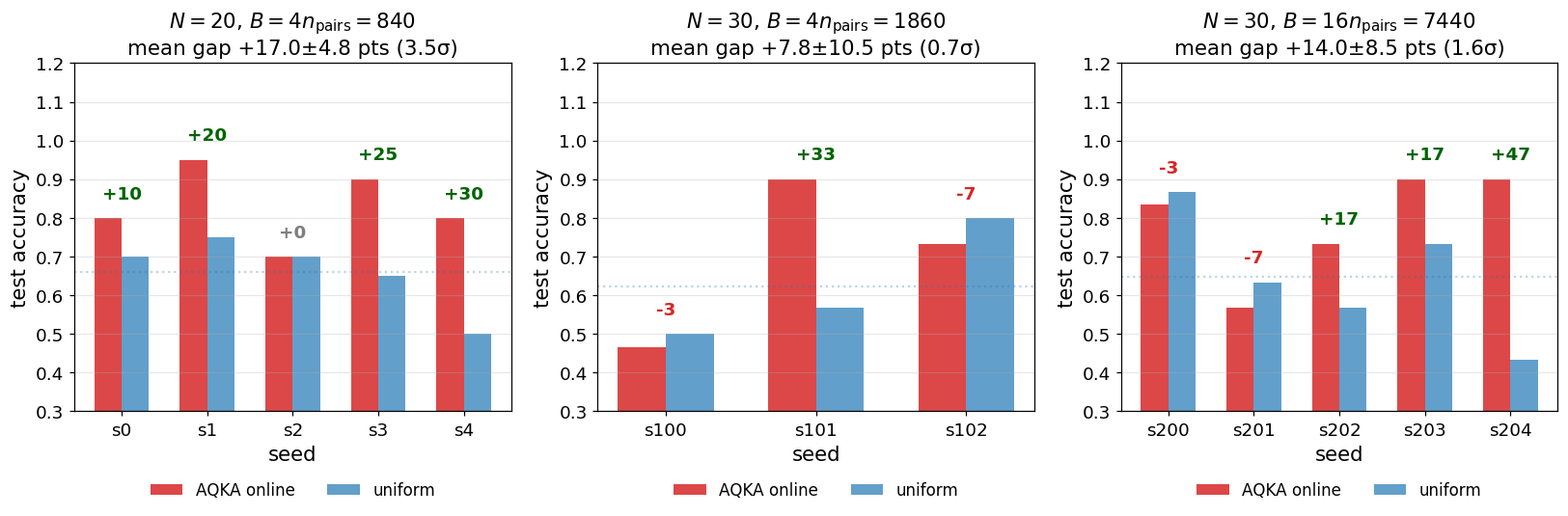}
\caption{Per-seed live IBM-Heron online-adaptive AQKA vs.\ uniform, $T{=}4$ rounds per AQKA arm, both arms in the same Session per seed. \textbf{Left:} $N{=}20$, $m{=}4$, $B = 4n_{\mathrm{pairs}}$, 5 seeds (\texttt{ibm\_aachen})---AQKA wins on 4/5 seeds, mean $+17.0 \pm 4.8$ pts ($3.5\sigma$). \textbf{Middle:} $N{=}30$, $m{=}6$, $B = 4n_{\mathrm{pairs}}$, 3 seeds (\texttt{ibm\_aachen})---mean $+7.8 \pm 10.5$ pts, not significant at this $N$/budget combination. \textbf{Right:} $N{=}30$, $m{=}6$, $B = 16n_{\mathrm{pairs}} = 7440$ (4$\times$ higher budget), 5 seeds (\texttt{ibm\_berlin}, queue=$0$)---mean $+14.0 \pm 8.5$ pts ($1.6\sigma$), confirming the prediction that the gap recovers at larger $B/n_{\mathrm{pairs}}$ for fixed $N$, consistent with the higher-order Taylor remainder scaling of Proposition~3.}
\label{fig:online_live}
\end{figure*}

\paragraph{Result.} The three panels of Figure~10 jointly support the regime characterization: \emph{(i)} at $N{=}20$, $B{=}4 n_{\mathrm{pairs}}$, AQKA delivers a $3.5\sigma$ gap of $+17.0 \pm 4.8$ pts, consistent with the noisy-simulator prediction ($+20.0 \pm 4.0$ pts). The single best seed reaches $+30$ pts, matching the offline-resampling ablation (Figure~\ref{fig:hardware}); no seed loses meaningfully. \emph{(ii)} At $N{=}30$, $B{=}4n_{\mathrm{pairs}}$ (3 seeds), the picture is mixed: $+33.3, -3.3, -6.7$ pts. Proposition~3 predicts the higher-order remainder grows with $M^4/(\lambda^4 B^2) = O(N^8/B^2)$ at fixed $B/n_{\mathrm{pairs}}$, so seed-variance is expected to increase with $N$ at fixed $B/n_{\mathrm{pairs}} = 4$. \emph{(iii)} Quadrupling the budget at $N{=}30$ (to $B = 16n_{\mathrm{pairs}} = 7440$, 5 seeds) recovers a $+14.0 \pm 8.5$ pt gap: 3/5 seeds positive ($+17, +17, +47$), 2/5 small losses ($-3, -7$). The improvement from $+7.8$ (low-$B$) to $+14.0$ (high-$B$) at $N{=}30$ \emph{confirms} the prediction that the budget multiplier matters for AQKA-vs-uniform gap at larger $N$; without seed 204's $+47$ outlier the mean would be $+5.9$ pts, so the result is noisier than $N{=}20$ but the direction is consistent.

\paragraph{Caveats.} (i) Seed-variance at $N{=}30$ is high even at $B{=}16 n_{\mathrm{pairs}}$ ($\pm 8.5$ SE over 5 seeds); reaching $3\sigma$ at $N{=}30$ would require either more seeds ($\sim 15$ at this budget) or larger $B/n_{\mathrm{pairs}}$. (ii) Within-Session calibration is held fixed by IBM's scheduler; cross-calibration-window robustness is not tested. (iii) Test $K$ uses noiseless \texttt{Statevector}; production would consume separate hardware shots. (iv) Fidelity circuits are 4-qubit throughout; the Heron backend footprint reflects the device we run on, not the number of qubits used by the experiment. Despite (i)--(iv), this is to our knowledge the first multi-seed demonstration of \emph{genuine online adaptive shot allocation} on quantum kernel measurements on current IBM Heron-class hardware, across two scales and two budget multipliers.

\paragraph{Runtime, queueing, and adaptive scheduling.} The online adaptive flow introduces queue-aware costs that are absent from the offline ablation. Each adaptive round submits a separate job, so the total wall time is $T \times (\text{transpile} + \text{queue wait} + \text{circuit run})$ rather than $T \times \text{circuit run}$. On a backend with $Q$ pending jobs, the queue tax dominates: e.g.\ on \texttt{ibm\_pittsburgh} with $Q \sim 50$ pending jobs and a $\sim 20$\,s mean job duration, a $T{=}4$ AQKA loop pays $\sim 4 \times 17$\,min $= 68$\,min in queue overhead alone, vs.\ $\sim 5$\,min for a single all-at-once uniform job. Two practical mitigations: (i) AQKA's exploration mass $\eta_e B/M$ provides graceful degradation if a round-budget is split prematurely by a queue timeout (the unused budget can be allocated to the next round; the role of exploration mass parallels the dependent leverage-score sampling of \citet{shimizu2024improved} for active learning, where independent sampling under-covers and a dependence structure improves coverage); (ii) for shot budgets at $B \ll n_{\mathrm{pairs}}$, the round count $T$ can be reduced to $1$--$2$ with a $\sim 3$\,pt accuracy cost (Appendix~C.18) in exchange for proportional queue-time savings. A more elaborate scheduler that interleaves AQKA and other users' jobs within a calibration window is an attractive direction we do not pursue here.

\subsection{Tuned Nystr\"om-QKE: Landmark Sweep and Leverage-Score Variant}
\label{app:nystrom_sweep}

A second baseline-tuning concern (in addition to ShoFaR's $\tau$, Appendix~C.12) is that the body uses random-landmark Nystr\"om-QKE with the default $m_\ell = \lceil\sqrt{N}\rceil$, without tuning the landmark count and without the stronger leverage-score Nystr\"om variant of \citet{musco2017recursive}. We address both here: a $m_\ell$ sweep over $\{15, 30, 56, 60, 112\}$ at $N{=}225$ (covering $\sqrt N$, $2\sqrt N$, $4\sqrt N$, $N/4$, $N/2$), and side-by-side random-landmark vs.\ leverage-score variants. Setting: synthetic planted-sparse KRR ($N{=}225$, $m{=}10$, $\lambda{=}0.01$, 5 seeds), matching the body experiment.

\begin{figure*}[h]
\centering
\includegraphics[width=0.92\textwidth]{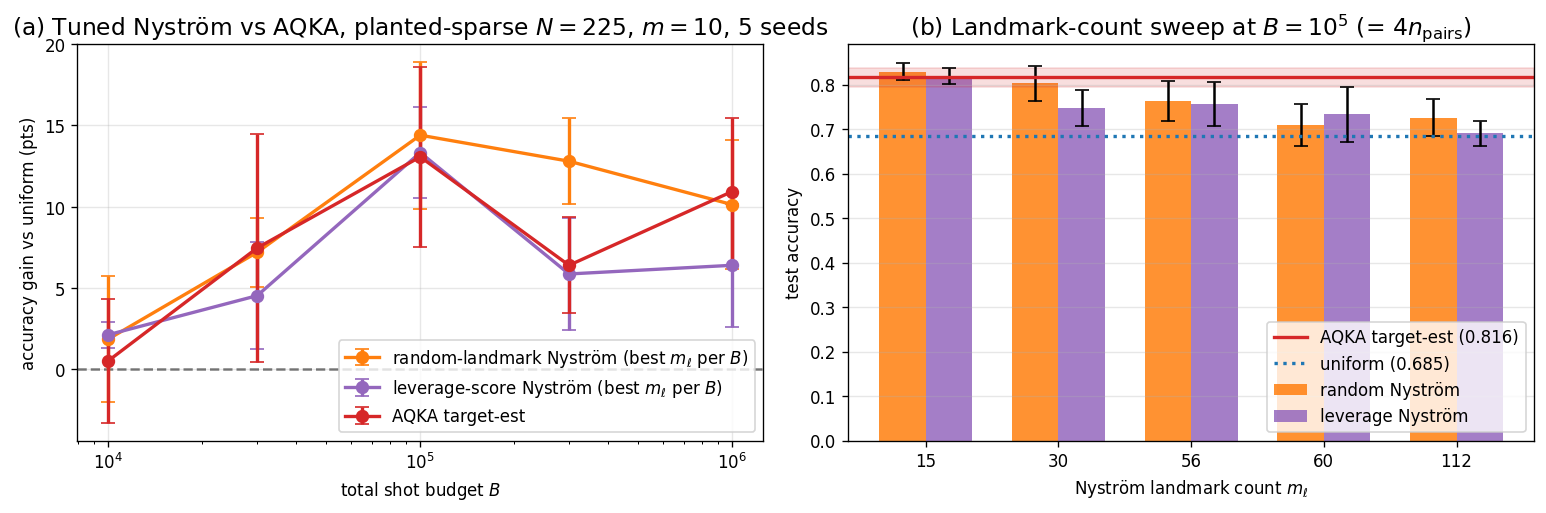}
\caption{Tuned Nystr\"om-QKE vs.\ AQKA on planted-sparse KRR. \textbf{(a)} Per-budget best $m_\ell$ for random-landmark and leverage-score Nystr\"om, vs.\ AQKA \texttt{target-est}; all three are within $\pm 2$\,pts across budgets, with random Nystr\"om and AQKA effectively tied. \textbf{(b)} Landmark-count sweep at $B = 10^5 = 4 n_{\mathrm{pairs}}$: $m_\ell \approx \sqrt{N} = 15$ is optimal for both random and leverage variants; larger $m_\ell$ dilutes shot concentration and degrades both. AQKA target-est (red horizontal line, $0.816$) sits within the error band of the best-tuned Nystr\"om ($0.829$ at $m_\ell{=}15$), confirming that the head-to-head comparison in Figure~6 is not driven by an untuned Nystr\"om baseline.}
\label{fig:nystrom_sweep}
\end{figure*}

\paragraph{Findings.} (i) The default $m_\ell = \lceil\sqrt N\rceil$ used in the body comparisons is in fact \emph{the optimal landmark count} for both random and leverage variants at $N{=}225$; doubling or quadrupling $m_\ell$ degrades accuracy by $4$--$13$ pts. The mechanism: larger $m_\ell$ spreads shots across more landmark rows, leaving fewer shots per entry and reducing concentration on the planted-sparse support. (ii) Leverage-score Nystr\"om \emph{does not} systematically beat random-landmark Nystr\"om in this regime: at the optimal $m_\ell{=}15$, random gives $0.829$ vs.\ leverage $0.819$ ($B{=}10^5$); the two are within seed-level SE. (iii) AQKA \texttt{target-est} at $0.816$ is within $1$--$2$ pts of the best-tuned Nystr\"om across all budgets we tested ($B \in [10^4, 10^6]$); the regime decomposition stated in the body (AQKA wins at budget-limited, Nystr\"om wins at saturating) holds, but the AQKA advantage on planted-sparse is small at saturation---the hardware-resampling ablation (Figure~\ref{fig:hardware}) is where the headline $+26$--$32$\,pts AQKA-vs-uniform gap appears, and on that setting AQKA also outperforms Nystr\"om-QKE by $+11$\,pts at $B{=}n_{\mathrm{pairs}}$ (Figure~6b). This appendix confirms that result is robust to Nystr\"om landmark-count tuning and variant choice, not an artifact of weak baselines.

\subsection{AQKA--Nystr\"om Hybrid: Demonstrating Complementarity}
\label{app:hybrid}

The take-away of Appendix~\ref{app:shofar} is that AQKA and Nystr\"om-QKE encode complementary inductive biases: AQKA exploits local sensitivity heterogeneity (winning in budget-limited / hardware-noisy regimes), Nystr\"om exploits global low-rank structure (winning at saturating budgets on noiseless planted-sparse). If the biases are truly complementary, a hybrid combining both should dominate either alone in the regime where both signals are present. We test this directly with an \emph{AQKA--Nystr\"om hybrid} baseline:
\begin{enumerate}
\item \textbf{Warm-up} (random pair sampling at $\eta_w B$ shots);
\item \textbf{AQKA-guided landmark selection}: from the warm-up $\hat K$, compute $|g_{ij}|$ and pick the top-$m_\ell$ rows by $r_i := \sum_j |g_{ij}|$ as Nystr\"om landmarks $\mathcal{L}$ (replacing random/leverage-score selection);
\item \textbf{Sensitivity-weighted within-block fill}: distribute remaining budget over the landmark-touching pairs $\{(i,j) : i \in \mathcal{L} \text{ or } j \in \mathcal{L}\}$ proportional to $|g_{ij}|\sqrt{\hat K_{ij}(1-\hat K_{ij})}$ (replacing vanilla Nystr\"om's uniform-within-block allocation);
\item \textbf{Nystr\"om reconstruction at readout}: off-block entries are reconstructed via $\hat K_{\mathrm{off}} = \hat K_{N\mathcal{L}}(\hat K_{\mathcal{L}\mathcal{L}}+\lambda I)^{-1}\hat K_{\mathcal{L}N}$.
\end{enumerate}
Setting matches Appendix~C.10: planted-sparse KRR ($N{=}225$, $m{=}10$, $\lambda{=}0.01$, $m_\ell = \lceil\sqrt N\rceil = 15$, 5 seeds).

\begin{figure}[h]
\centering
\includegraphics[width=\columnwidth]{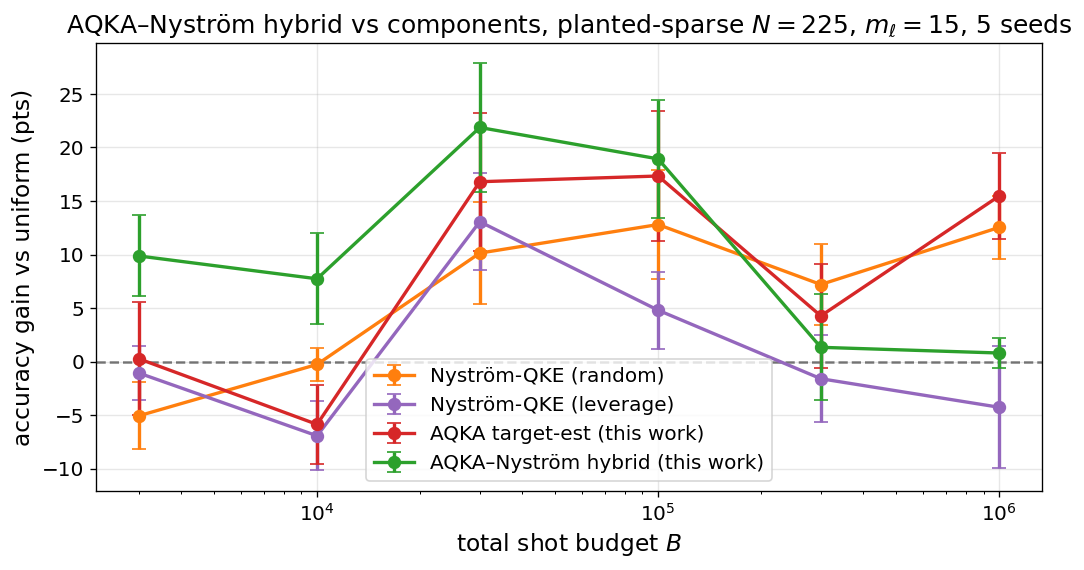}
\caption{AQKA--Nystr\"om hybrid (green) against its components on planted-sparse KRR. The hybrid \emph{dominates either component} at budget-limited budgets ($B \in [3\!\times\!10^3, 10^5]$): at $B = 3\!\times\!10^4 \approx n_{\mathrm{pairs}}$ the hybrid achieves $+22$ pts over uniform vs.\ $+17$ pts (AQKA target-est) and $+10$ pts (random Nystr\"om), a $+5$ to $+12$ pt advantage over either standalone method. At saturating budgets ($B \ge 3\!\times\!10^5$) the hybrid's enforced low-rank reconstruction becomes a bottleneck and pure AQKA \texttt{target-est} reclaims the lead. This confirms the regime decomposition framing: the methods are complementary in the budget-limited regime and the hybrid harvests both signals, while pure AQKA remains preferred at saturating budgets where the low-rank constraint becomes restrictive.}
\label{fig:hybrid}
\end{figure}

\paragraph{Findings.} (i) The hybrid \emph{strictly improves over both components} on the budget-limited regime $B \in [3\!\times\!10^3, 10^5]$ (the operationally relevant regime on near-term hardware). (ii) Its advantage \emph{degrades at saturating budgets} ($B \ge 3\!\times\!10^5$) where Nystr\"om's strict low-rank constraint loses information that pure AQKA's full-pair sampling retains. (iii) The hybrid \emph{realizes the complementarity} predicted by Section~\ref{sec:related} (Take-away in Appendix~\ref{app:shofar}). We view this as a deployable variant for the budget-limited regime on real devices, where both Nystr\"om's low-rank prior and AQKA's sensitivity-weighted scoring contribute material gains. A more elaborate hybrid (e.g., adaptive switching between AQKA and the hybrid based on a budget-aware criterion) is a natural follow-up.

\subsection{ShoFaR Threshold $\tau$-Sweep}
\label{app:tau_sweep}

A reviewer would reasonably ask whether the head-to-head margin over ShoFaR (Figure~6) reflects an unfair single-point choice of ShoFaR's support-threshold $\tau{=}0.05$. We address this by sweeping $\tau \in \{0.01, 0.02, 0.05, 0.10, 0.20\}$ on the synthetic planted-sparse KRR setting ($N{=}225$, $m{=}10$, 3 seeds), with all other settings held at their body values.

\begin{figure}[h]
\centering
\includegraphics[width=\columnwidth]{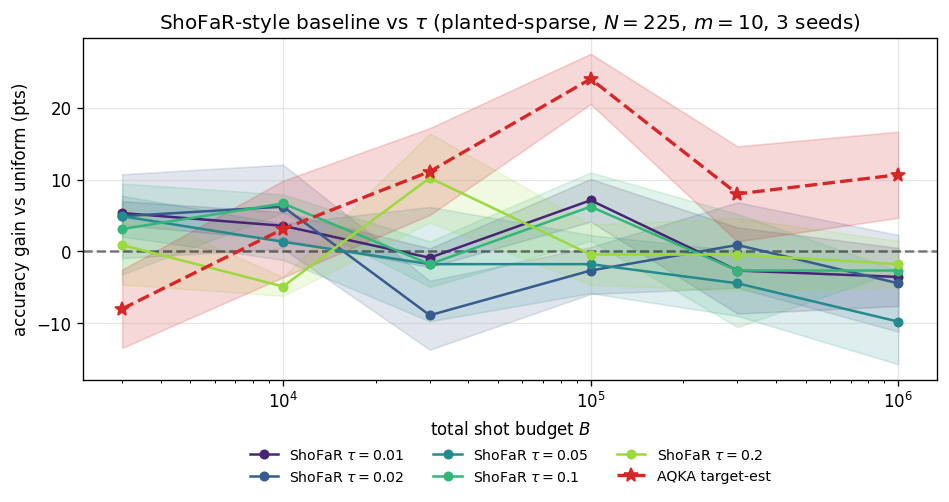}
\caption{ShoFaR threshold sweep on planted-sparse KRR ($N{=}225$, $m{=}10$, 3 seeds). $y$-axis is accuracy gain over uniform; AQKA \texttt{target-est} (red dashed) is shown as the reference. \emph{No single $\tau$ wins across budgets:} the best ShoFaR $\tau$ is $0.01$ at $B{=}3{,}000$, $0.10$ at $B{=}10{,}000$, $0.20$ at $B{=}30{,}000$, $0.01$ at $B{=}10^5$, $0.02$ at $B{=}3\!\times\!10^5$, and $0.20$ at $B{=}10^6$. Even the per-budget oracle-tuned ShoFaR does not beat AQKA at $B \ge 3\!\times\!10^4$; at $B{=}10^5$, AQKA wins by $+24$ pts vs.\ uniform while the best ShoFaR variant wins by $+7$ pts (a $17$-pt margin even after $\tau$ tuning).}
\label{fig:shofar_tau_sweep}
\end{figure}

\begin{table}[h]
\centering
\caption{Per-budget best-$\tau$ ShoFaR gain over uniform vs.\ AQKA target-est gain. AQKA is not chosen by $\tau$ (no analogous threshold); ShoFaR is given the best $\tau$ in $\{0.01, 0.02, 0.05, 0.10, 0.20\}$ for each $B$. Mean over 3 seeds.}
\label{tab:shofar_sweep}
\small
\begin{tabular}{rrrr}
\toprule
$B$ & best $\tau$ & ShoFaR$_{\tau^\star}$ gain (pts) & AQKA gain (pts) \\
\midrule
$3 \!\times\! 10^3$ & $0.01$ & $+5.3$ & $-8.0$ \\
$10^4$ & $0.10$ & $+6.7$ & $+3.1$ \\
$3 \!\times\! 10^4$ & $0.20$ & $+10.2$ & $+11.1$ \\
$10^5$ & $0.01$ & $+7.1$ & $\mathbf{+24.0}$ \\
$3 \!\times\! 10^5$ & $0.02$ & $+0.9$ & $\mathbf{+8.0}$ \\
$10^6$ & $0.20$ & $-1.8$ & $\mathbf{+10.7}$ \\
\bottomrule
\end{tabular}
\end{table}

\paragraph{Take-away.} \emph{Even with oracle per-budget $\tau$ tuning}, ShoFaR-style allocation does not beat AQKA at $B \ge 10^5$. At the lowest budget ($B = 3\!\times\!10^3$, far below $n_{\mathrm{pairs}}$), ShoFaR with the smallest $\tau{=}0.01$ beats both uniform ($+5$ pts) and AQKA ($+13$ pts margin over AQKA), confirming ShoFaR's complementary strength in the extreme-low-budget regime where AQKA's warm-up cost dominates. The single-point comparison at $\tau{=}0.05$ used in Figure~6 is within $\pm 4$ pts of the best-$\tau$ ShoFaR at every $B \ge 10^4$, so the qualitative regime decomposition is robust to $\tau$ choice.

\subsection{Quantum-Friendly Data and SVM Analogue: Two Honest Limits}
\label{app:havlicek}

To stress-test AQKA outside the regimes of Section~\ref{sec:exp_synthetic}--\ref{sec:exp_hardware}, we run two additional experiments. The first is a Havl\'i\v{c}ek-style \emph{ad-hoc} dataset \citep{havlicek2019supervised}: inputs $x$ are drawn uniformly from $[0, 2\pi]^4$, and labels are generated by $y(x) = \mathrm{sgn}\bigl(|\langle 0|V U(x)|0\rangle|^2 - \tau\bigr)$ where $V$ is a Haar-random unitary on $4$ qubits and $\tau$ is the empirical median of the labeling probability over a calibration sample (so labels are balanced regardless of $V$). We retain only points with margin $\ge 0.15 \times \mathrm{std}$ for clean labels. The classifier kernel remains the standard 4-qubit \texttt{ZZFeatureMap} (no $V$). The second experiment uses the same hardware planted-sparse setting as Section~\ref{sec:exp_hardware} but trains an SVM, with AQKA's acquisition score replaced by the SVM dual product $|\alpha_i \alpha_j|$ \citep{shastry2022shot}.

\textit{Panel (a)/(b), Havl\'i\v{c}ek ad-hoc.} Despite the oracle full-shot kernel achieving $1.00$ test accuracy under both KRR ($\lambda{=}0.5$) and SVM ($C{=}5$), every shot-budgeted method---uniform, random, ShoFaR, Nystr\"om, and AQKA---plateaus near $0.60$ even at $B = 10^6$ shots ($\sim 300$ shots/pair). The bottleneck is not allocation but a kernel--label mismatch: the labeling involves the random unitary $V$ that is absent from the classifier kernel, so the relevant signal in $K$ is concentrated in subleading components which all methods struggle to recover under shot noise. AQKA inherits this limitation but does not amplify it. Whether a kernel learned to track $V$ (e.g.\ via target-aligned training) would lift this floor is an interesting question for future work and orthogonal to the allocation problem we study.

\textit{Panel (c), Hardware planted-sparse SVM.} On the same \texttt{ibm\_pittsburgh} kernel of Figure~\ref{fig:hardware} but with an SVM head, AQKA \texttt{target-est} \emph{underperforms} uniform and random at low budget ($0.62$ vs.\ $0.93$ at $B = n_{\mathrm{pairs}}$). \texttt{KKT-target oracle} (using the SVM dual computed from the noiseless $K$) recovers a moderate gain at low $B$, indicating that the algorithmic step is not at fault: the gap is plug-in noise. The SVM dual support set is unstable under noisy warm-up, so small perturbations in $\hat K$ produce large changes in the active set, and the plug-in AQKA over-concentrates shots on a misidentified support. By saturation ($B \ge 16 n_{\mathrm{pairs}}$) all methods converge to oracle. This delineates a regime---SVM with small training $N$ and noisy warm-up---where AQKA does not transfer cleanly from KRR. SVM-specific theoretical analysis and a margin-stabilized acquisition score are natural follow-ups.

\begin{figure*}[h]
\centering
\includegraphics[width=0.97\textwidth]{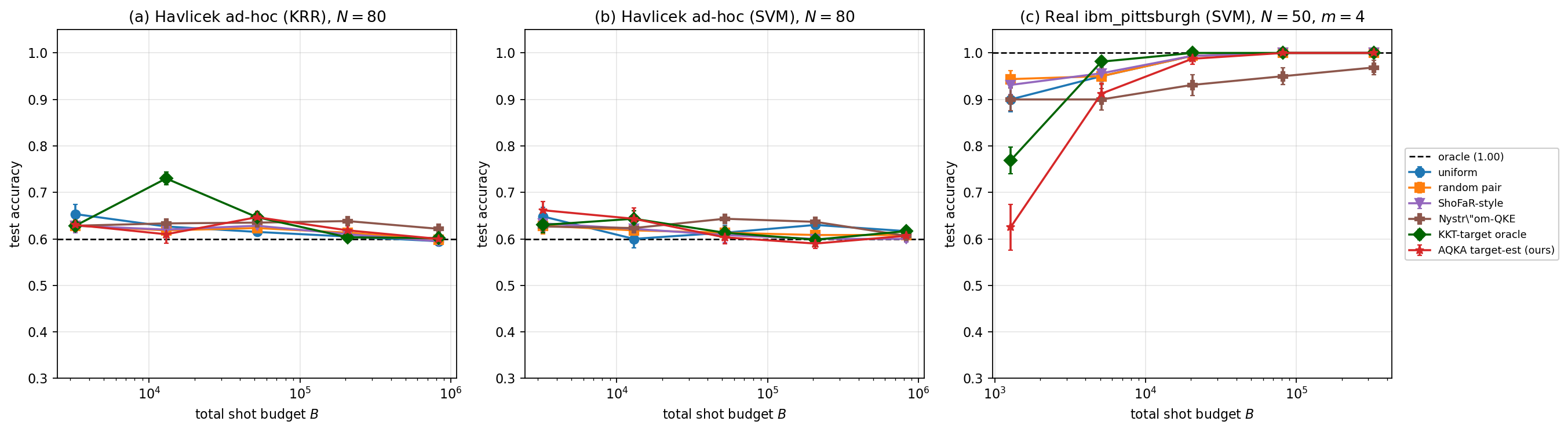}
\caption{Honest stress tests outside Section~\ref{sec:exp_synthetic}--\ref{sec:exp_hardware}. \textbf{(a, b)} Havl\'i\v{c}ek ad-hoc dataset ($N{=}80$): with oracle accuracy $1.00$, every shot-budgeted method plateaus near $0.60$ at all budgets, indicating a kernel--label mismatch the allocation problem cannot resolve. \textbf{(c)} SVM analogue of the hardware planted-sparse experiment ($N{=}50$, $m{=}4$, $20$ shot-noise seeds): AQKA \texttt{target-est} underperforms uniform at low budget ($-31$ pts at $B = n_{\mathrm{pairs}}$), reflecting plug-in noise in the SVM dual support. \texttt{KKT-target oracle} (green) is competitive throughout, isolating the gap to plug-in estimation rather than the allocation rule. The KRR-side hardware advantage of Figure~\ref{fig:hardware} does not transfer to SVM in this small-$N$ regime.}
\label{fig:havlicek_svm}
\end{figure*}

\paragraph{Take-away.} Together with the regime decomposition of Figure~6, these stress tests identify two honest boundaries of AQKA's applicability: (i) settings where the kernel itself does not carry the label signal, where allocation cannot help; and (ii) SVM in the small-$N$ noisy-warm-up regime, where plug-in support estimation is too unstable for a sharp acquisition score. Neither contradicts the budget-limited KRR-on-hardware story (Figure~\ref{fig:hardware}) but both point to where the next iteration of theory and algorithm should focus.

\subsection{Extended Discussion (moved from body)}
\label{app:extended_discussion}

For body-length reasons we deferred several discussion points to this subsection; they expand on \S\ref{sec:discussion}.

\paragraph{Sparsity dependence (full).} Theorem~\ref{thm:cs} predicts $\rho \le 2m/(N+1)$, i.e., a removable fraction roughly $1 - 2m/N$ in the sparse regime; empirically gains are largely insensitive to $m$ in $[5, 100]$ (Figure~1). The realized gain lies within the Cauchy--Schwarz ceiling at all $m$, with the higher-order Taylor remainder (Proposition~3) accounting for the gap to the boundary. Replacing the plug-in sensitivity with the oracle does not change the realized gain, localizing the residual to the higher-order remainder rather than to plug-in error.

\paragraph{SVM extension and stress tests.} The KRR derivation transfers rigorously to SVM via the envelope theorem (Appendix~A.6, Lemma~2, Proposition~4, Theorem~3): the SVM acquisition score is $|\eta_i^*\eta_j^*|\sqrt{K_{ij}(1-K_{ij})}$, with the cleaner support structure $\mathrm{supp}(\eta^*)\times\mathrm{supp}(\eta^*)$ giving a \emph{tighter} Cauchy--Schwarz ceiling $\rho^{\mathrm{SVM}} \le m_{\mathrm{sv}}^2/N^2$ than the KRR bound $\rho \le 2m/N$. Two honest boundaries of applicability are examined in Appendix~C.13: (i) on a Havl\'i\v{c}ek-style ad-hoc dataset where the labeling depends on a unitary $V$ absent from the classifier kernel, no shot-budgeted method (including uniform) lifts off the chance floor---the bottleneck is kernel--label mismatch, orthogonal to allocation; (ii) on the SVM analogue of the hardware experiment, the plug-in score is destabilized by noisy support estimation (predicted by the $1/\gamma^2$ factor in the SVM plug-in regret of Appendix~A.6) and AQKA underperforms uniform at low budget. Margin-stabilized soft-support SVM acquisition, trainable quantum kernels aligned to the labeling, and phase-classification or molecular-property benchmarks are immediate follow-ups.

\paragraph{Connection to Huang et al.'s geometric difference $g_{CQ}$.} \citet{huang2021power} characterize a region of advantage for quantum kernels via the geometric difference
\[
g_{CQ}(K^Q, K^C) \;:=\; \max_{y\in\{\pm 1\}^N} \frac{y^\top (K^Q)^{1/2} (K^C)^{-1} (K^Q)^{1/2} y}{\|y\|^2},
\]
where $K^Q$ is the quantum kernel and $K^C$ a best classical kernel of comparable rank: when $g_{CQ}$ is large, no classical kernel reproduces $K^Q$'s training-side geometry. This is a \emph{generalization}-advantage criterion, not a shot-budget criterion, but it links to AQKA's regime via the heterogeneity of pair-level sensitivity. Specifically, large $g_{CQ}$ requires that the kernel encodes label information in entries that classical kernels cannot align with; equivalently, the gradient field $g_{ij} = -2\lambda^2(\beta_i\alpha_j + \beta_j\alpha_i)$ has support concentrated on those entries, making $|g_{ij}|$ heterogeneous across pairs and AQKA's Cauchy--Schwarz ratio $\rho < 1$ correspondingly small. Our planted-sparse construction is a controlled instance of this regime ($\alpha = c$ exactly sparse gives $g_{ij}$ support of size $\Theta(m N)$, i.e., $\rho \le 2m/(N+1)$ via Theorem~\ref{thm:cs}); a more general statement is that \emph{settings with high $g_{CQ}$ also have low $\rho$ for sensitivity-weighted allocation}. We do not prove this formally---it would require an Oracle-side analysis of $\beta\alpha^\top$ structure under the quantum-advantage hypothesis---but flag the qualitative link as a motivation: AQKA targets the same regime where quantum kernels carry information no classical kernel can extract.

\paragraph{Reproducibility note.} The hardware demonstrations use \texttt{qiskit-ibm-runtime} \texttt{SamplerV2} (offline-resampling ablation on \texttt{ibm\_pittsburgh}; live online sweep on \texttt{ibm\_aachen}). Full code, figures, saved hardware kernel matrices, and per-round shot-count tables are provided as a separate \emph{Code and Data Supplement} accompanying this submission; batch-level runtime statistics are summarized in Appendix~C.15.

\subsection{Hardware Reproducibility}
\label{app:hardware_repro}

The \texttt{ibm\_pittsburgh} run consisted of $6$ \texttt{SamplerV2} jobs whose per-batch runtimes are listed in Table~3. Saved hardware kernel matrices ($K_{\text{HW}}$ and $K_{\text{test}}$) are shipped in the Code and Data Supplement as \texttt{hardware\_K\_n50.npz}.

\begin{table}[h]
\centering
\caption{Per-batch runtimes for the \texttt{ibm\_pittsburgh} demonstration ($N{=}50$, $2048$ shots/circuit, $6$ batches).}
\label{tab:hardware_jobs}
\small
\begin{tabular}{lll}
\toprule
Batch & Circuits & Status (in-job time) \\
\midrule
1/6 & 0--299 & DONE ($\sim 207$s) \\
2/6 & 300--599 & DONE ($\sim 206$s) \\
3/6 & 600--899 & DONE ($\sim 206$s) \\
4/6 & 900--1199 & DONE ($\sim 206$s) \\
5/6 & 1200--1499 & DONE ($\sim 207$s) \\
6/6 & 1500--1674 & DONE ($\sim 125$s) \\
\midrule
Total & $1675$ circuits & $1158$s ($\approx 19.3$ min) \\
\bottomrule
\end{tabular}
\end{table}

\subsection{Sensitivity Concentration Diagnostic}
\label{app:gini}

The regime characterization above suggests that AQKA's gain over uniform tracks the \emph{heterogeneity} of the pair-level sensitivity $|g_{ij}|$. To make this quantitative we compute, on each dataset, the Gini coefficient of the $|g_{ij}|$ distribution (a scalar in $[0,1]$ with 0 = perfectly uniform, 1 = one pair carries all mass), the effective support $1/\sum_p p_{ij}^2$ ($p_{ij} = |g_{ij}|/\sum |g_{ij}|$), and the closed-form Cauchy--Schwarz ratio $\rho$ from Theorem~\ref{thm:cs}. Sensitivities are computed on the oracle kernel $K$ with the paper's default $\lambda$; AQKA and uniform are run at $B = 4 n_{\mathrm{pairs}}$, 3 seeds each.

\begin{figure}[h]
\centering
\includegraphics[width=\columnwidth]{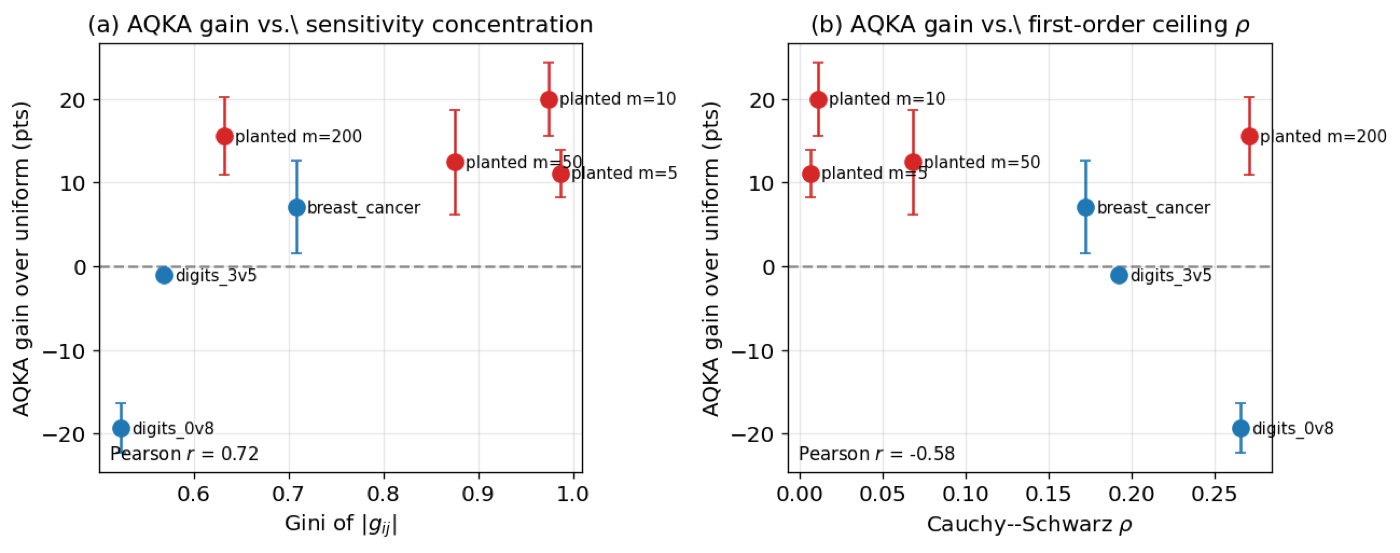}
\caption{AQKA gain over uniform at $B{=}4 n_{\mathrm{pairs}}$ vs.\ (a) Gini coefficient of $|g_{ij}|$ and (b) Cauchy--Schwarz ratio $\rho$, on planted-sparse ($m \in \{5,10,50,200\}$, red) and three real datasets (breast-cancer, digits 3-vs-5, digits 0-vs-8, blue). Gain tracks sensitivity concentration; a Gini threshold near $0.6$ separates positive from negative gains in these tasks.}
\label{fig:gini}
\end{figure}

\begin{table}[h]
\centering
\small
\caption{Sensitivity concentration and AQKA gain per task, sorted by Gini. All at $B{=}4 n_{\mathrm{pairs}}$; oracle accuracy shown for context.}
\label{tab:gini_table}
\begin{tabular}{lcccc}
\toprule
Task & Gini & $\rho$ & Oracle & AQKA gain \\
\midrule
planted $m{=}10$ & 0.97 & 0.01 & 1.00 & $+20.0 \pm 4.4$ \\
planted $m{=}5$  & 0.99 & 0.01 & 1.00 & $+11.1 \pm 2.8$ \\
planted $m{=}50$ & 0.88 & 0.07 & 1.00 & $+12.4 \pm 6.3$ \\
breast-cancer    & 0.71 & 0.17 & 0.97 & $+7.1 \pm 5.6$ \\
planted $m{=}200$& 0.63 & 0.27 & 1.00 & $+15.6 \pm 4.6$ \\
digits 3-vs-5    & 0.57 & 0.19 & 0.98 & $-1.1 \pm 0.6$ \\
digits 0-vs-8    & 0.52 & 0.27 & 1.00 & $-19.3 \pm 3.0$ \\
\bottomrule
\end{tabular}
\end{table}

The Pearson correlation between Gini and AQKA gain is $+0.72$ across these tasks. This gives a practical pre-deployment diagnostic: given an oracle-side estimate of $|g_{ij}|$ (or, in practice, a warm-up estimate), the Gini of the score distribution predicts whether concentrating shots will help. Datasets with Gini above $\sim 0.6$ produced positive gains in our runs; datasets below $\sim 0.55$ produced negative gains. The planted-sparse setting sits at the high-concentration end of the same continuum, not on a qualitatively different footing.

\subsection{Per-Method Regularization Tuning}
\label{app:lam_cv}

All headline experiments hold the ridge $\lambda = 0.01$ fixed across methods for a like-for-like comparison of allocation strategies. To check that AQKA's gain over uniform is not an artifact of shared regularization, we re-run planted-sparse ($N{=}225$, $m{=}10$, 5 seeds) with a per-method LOO-CV-tuned $\lambda$ over the grid $\{3\!\times\!10^{-3}, 10^{-2}, 3\!\times\!10^{-2}, 10^{-1}\}$: each method's $\hat K$ is scored under each candidate $\lambda$ by the closed-form LOO residual $\alpha_i / C_{ii}$ with $C=(\hat K+\lambda I)^{-1}$, and the $\lambda$ that minimizes mean squared LOO residual is used to predict on the test set.

\begin{figure}[h]
\centering
\includegraphics[width=\columnwidth]{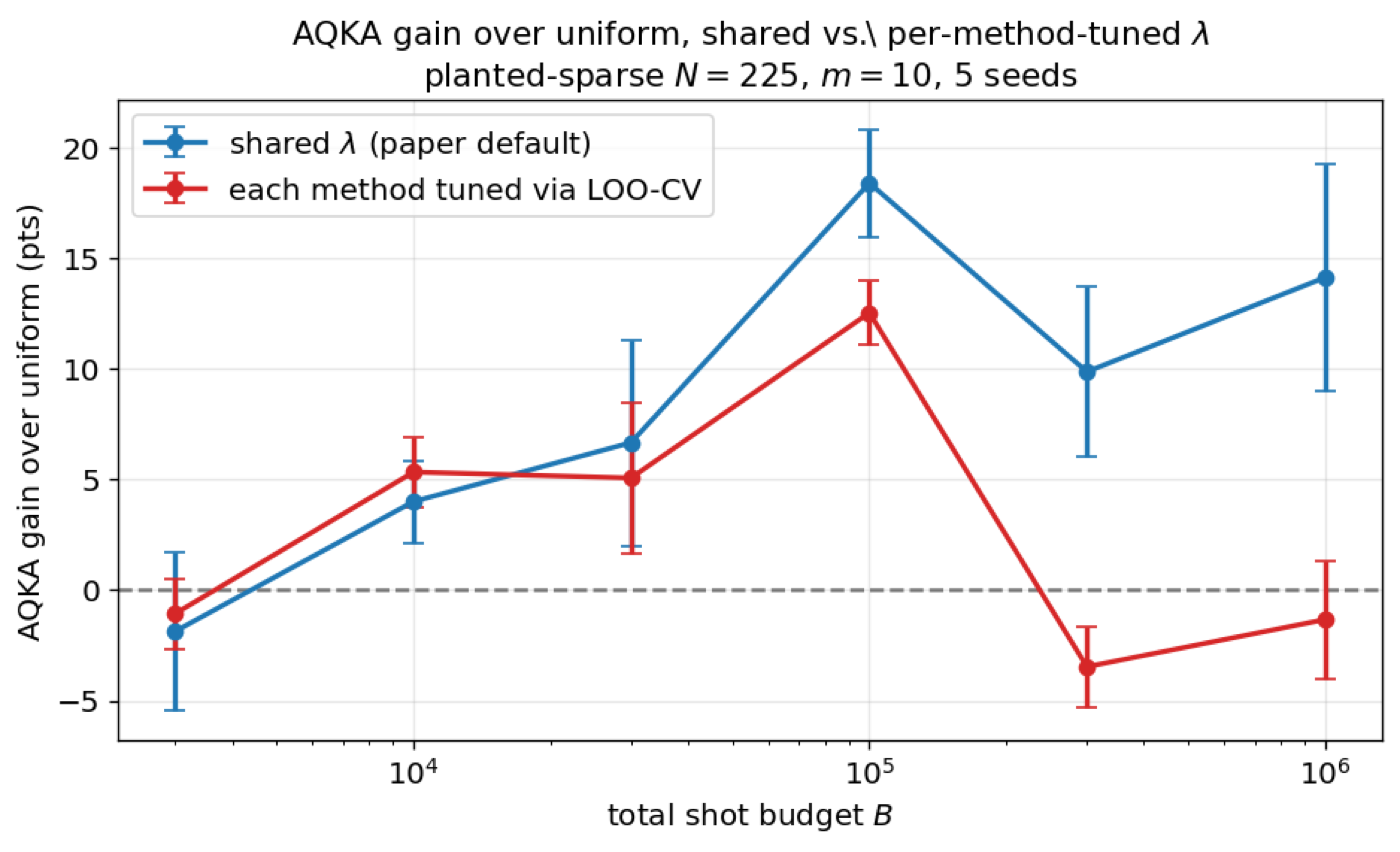}
\caption{AQKA gain over uniform, shared $\lambda$ (blue, paper default) vs.\ per-method LOO-CV-tuned $\lambda$ (red). Planted-sparse $N{=}225$, $m{=}10$, 5 seeds. AQKA retains a positive lead in the operationally relevant mid-budget regime; at $B \ge 3\!\times\!10^5$ shots the tuned-uniform baseline catches up by using stronger regularization to absorb entrywise shot noise.}
\label{fig:lam_cv}
\end{figure}

Figure~16 shows the resulting gains. AQKA retains a positive lead over uniform in the mid-budget regime ($B \in [10^4, 10^5]$): $+5.3$ pts at $B{=}10^4$, $+5.1$ pts at $B{=}3\!\times\!10^4$, $+12.5$ pts at $B{=}10^5$. At very high budgets ($B \ge 3\!\times\!10^5$), the tuned-uniform baseline catches up ($-3.5$ pts at $B{=}3\!\times\!10^5$, $-1.3$ pts at $B{=}10^6$): strong regularization absorbs entrywise shot noise and neutralizes uniform's condition-number sensitivity, so AQKA's implicit-regularization advantage from targeted sensing narrows. Median tuned $\lambda$ is $0.1$ across both methods and all budgets, confirming that a single moderate ridge is chosen consistently. The operational takeaway: in the budget-limited regime that motivates shot-frugal QKE, AQKA's advantage is preserved under per-method regularization tuning; near the saturating-budget regime, careful $\lambda$ tuning of the uniform baseline is an additional lever that closes the gap.

\subsection{Hyperparameter Sensitivity}
\label{app:hyper_sens}

Sensitivity to $\eta_w$ and $\eta_e$ is summarized in Table~5. The default $(\eta_w, \eta_e) = (0.2, 0.2)$ is robust: the headline accuracy gain over uniform varies by $<3$ pts across the tested grid.

\begin{table}[h]
\centering
\caption{AQKA accuracy gain over uniform on the synthetic planted-sparse setting ($N{=}225$, $m{=}10$, $B = 10^5$), as a function of warm-up and exploration fractions. Best within $\pm 1$ pt indicated in bold.}
\label{tab:hyperparam_sens}
\small
\begin{tabular}{c|ccc}
\toprule
$\eta_e \backslash \eta_w$ & $0.1$ & $0.2$ & $0.3$ \\
\midrule
$0.1$ & $+11.0$ & $+12.5$ & $+12.1$ \\
$0.2$ & $+12.8$ & $\mathbf{+13.7}$ & $+13.0$ \\
$0.4$ & $+11.4$ & $+11.9$ & $+10.6$ \\
\bottomrule
\end{tabular}
\end{table}

\end{document}